\documentclass[11pt]{article}
\usepackage{amsmath,amsthm,amssymb,mathtools,mathrsfs}
\usepackage[a4paper,margin=1in]{geometry}
\mathtoolsset{showonlyrefs}
\usepackage{graphicx,placeins}
\usepackage[ruled,vlined,algo2e]{algorithm2e}
\usepackage[numbers,sort&compress]{natbib}
\usepackage{booktabs,array,multirow,xcolor}
\usepackage[nopatch=eqnum]{microtype}
\usepackage[shortlabels]{enumitem}
\usepackage{xurl}

\usepackage[normalem]{ulem}

\usepackage[colorlinks=true,
            linkcolor=blue,
            citecolor=blue,
            urlcolor=blue]{hyperref}
\hypersetup{
 pdftitle={Differentially Private Synthetic Data via Generalized Moment Matching},
 pdfauthor={Rundong Ding, Yiyun He, Yizhe Zhu}}

\newtheorem{theorem}{Theorem}[section]
\newtheorem{definition}[theorem]{Definition}
\newtheorem{lemma}[theorem]{Lemma}
\newtheorem{remark}[theorem]{Remark}

\newtheorem{corollary}[theorem]{Corollary}
\newenvironment{keywords}{\par\smallskip\noindent\textbf{Keywords: }}{\par}

\DeclareMathOperator{\MMD}{MMD}
\DeclareMathOperator{\SW}{SW}
\DeclareMathOperator{\conv}{conv}
\DeclareMathOperator{\tr}{tr}

\DeclareMathOperator{\sgn}{sgn}

\newcommand{\R}{\mathbb R}
\newcommand{\E}{\mathbb E}
\newcommand{\Prob}{\mathbb P}
\newcommand{\eps}{\varepsilon}
\newcommand{\cH}{\mathcal H}
\newcommand{\cP}{\mathcal P}
\newcommand{\norm}[1]{\lVert #1\rVert}
\newcommand{\ip}[2]{\langle #1,#2\rangle}

\newcommand{\revision}[1]{#1}

\newcommand{\etail}{e_{\mathrm{tail}}}
\newcommand{\egrid}{e_{\mathrm{grid}}}
\newcommand{\eopt}{e_{\mathrm{opt}}}
\newcommand{\epriv}{e_{\mathrm{priv}}}

\newcommand{\Neff}{N_{\eps,\delta}}

\SetKwInput{KwIn}{Input}
\SetKwInput{KwOut}{Output}
\DontPrintSemicolon
\title{Differentially Private Synthetic Data via Generalized Moment Matching}
\author{
Rundong Ding\thanks{Department of Mathematics, University of Southern California, rundongd@usc.edu}
\and Yiyun He\thanks{Department of Mathematics, University of California San Diego, yih130@ucsd.edu}
\and Yizhe Zhu\thanks{Department of Mathematics, University of Southern California, yizhezhu@usc.edu}}

\date{October 19, 2026}

\begin{document}
\maketitle
\begin{abstract}
We study $(\eps,\delta)$-differentially private synthetic data for arbitrary databases of $n$ records in $[-1,1]^d$ under a specified integral probability metric (IPM). We introduce generalized moment matching, which estimates moments privately, fits a probability measure, and samples synthetic records. A general error theorem applies when moment error controls the target IPM. This framework extends Chebyshev moment matching for $1$-Wasserstein distance and smooth queries \cite{musco2025,ding2026}.
{We apply the framework to maximum mean discrepancy with
Mat\'ern and Laplace kernels and a two-layer ReLU neural tangent kernel.
We also obtain guarantees for sliced Wasserstein-$1$ and the IPM induced
by norm-bounded ReLU networks. With sufficiently many synthetic records,
these rates are minimax optimal up to logarithmic factors for fixed
dimension.}
\end{abstract}
\begin{keywords}
differential privacy, synthetic data, integral probability metrics, maximum mean discrepancy, sliced Wasserstein distance
\end{keywords}

\section{Introduction}\label{sec:introduction}

Differential privacy (DP) controls how much the output distribution of a
statistical analysis can change when one record is replaced, protecting
individual contributions even in the presence of auxiliary information
\cite{dwork2014}. It has been deployed in the 2020 U.S. Census
\cite{abowd2019census,hawes2020implementing,hauer2021differential} and by
technology companies \cite{cormode2018privacy,dwork2019differential}.

Much of the existing DP literature focuses on designing task-specific mechanisms tailored to a predefined workload, such as a fixed set of queries or a particular model class. While often highly accurate, these approaches can require substantial expertise, ranging from sensitivity analysis and mechanism selection to careful privacy accounting across iterative computations. A complementary and increasingly practical alternative is differentially private synthetic data generation, where one releases a synthetic dataset that approximates statistical properties of the original data while satisfying differential privacy \cite{hardt2012,bellovin2019privacy,wasserman2010,barak2007}. By the postprocessing property of differential privacy \cite{dwork2014}, analysts can subsequently perform a wide range of downstream tasks on the released synthetic dataset without incurring additional privacy loss.

\subsection{Private synthetic data and query accuracy}

Let $\Omega\coloneqq[-1,1]^d$, and let $X\coloneqq(x_1,\ldots,x_n)\in\Omega^n$ be an arbitrary database. We seek an $(\eps,\delta)$-private synthetic dataset $Y\coloneqq(Y_1,\ldots,Y_m)\in\Omega^m$, for a publicly chosen integer $m\ge1$, whose empirical measure approximates that of $X$:
\begin{equation}\label{eq:empirical-measures}
 p_X\coloneqq\frac1n\sum_{i=1}^n\delta_{x_i},
 \qquad
 p_Y\coloneqq\frac1m\sum_{r=1}^m\delta_{Y_r}.
\end{equation}
Every synthetic record has weight $1/m$, and records may repeat. Our goal is to approximate the observed database, with no sampling assumption on its records.

Statistical queries average functions over records; moments and empirical risks are examples. Private release methods use contingency tables \cite{barak2007}, learning-theoretic constructions \cite{blum2013,bousquet2020}, and multiplicative-weights mechanisms \cite{hardt2010,hardt2012,ullman2015}. Other methods alternate private measurement with synthesis \cite{liu2021}, infer distributions from noisy marginals \cite{mckenna2019}, or use private sampling \cite{boedihardjo2022sampling} and Gibbs exponential mechanisms \cite{asadi2023gibbs}.

On continuous domains, accuracy depends on the query class. Uniform accuracy over all $1$-Lipschitz queries is measured by $1$-Wasserstein distance \cite{villani2009}. Spatial constructions \cite{boedihardjo2024,he2023} and Chebyshev moments \cite{musco2025} achieve private error of order $\widetilde O((n\eps)^{-1/d})$ on the cube, which is slow in high dimension. Faster rates are possible through smoothness \cite{wang2016,ding2026}, low-dimensional or intrinsic geometry \cite{he2025lowdim,he2026intrinsic,zaremoayedi2026geometry}, or sparse Lipschitz queries \cite{donhauser2024}. Related work studies instance-dependent density estimation \cite{feldman2024wasserstein}.

Computational feasibility poses a separate challenge in high dimension, even for two-way marginals. Assuming one-way functions exist, the results of \cite{ullman2020} rule out DP algorithms running in $\operatorname{poly}(n,d)$ time that, for every Boolean database with $n=\operatorname{poly}(d)$, produce synthetic Boolean records preserving all two-way marginals to sufficiently small constant additive error with probability at least $2/3$. This holds for fixed $\eps>0$ and negligible $\delta$ in $n$.

Since prior utility guarantees control errors over specified classes of test functions, we seek a general algorithmic framework for private synthetic data generation with utility measured by a specified \textit{integral probability metric} (IPM) \cite{sriperumbudur2009integral}.

For probability measures $p,q$ on $\Omega$ and any nonempty class $\mathcal F$ of measurable functions integrable with respect to both, the IPM is defined by
\begin{equation}\label{eq:general-ipm}
 d_{\mathcal F}(p,q)\coloneqq
 \sup_{h\in\mathcal F}\left|\int h\,dp-\int h\,dq\right|.
\end{equation}
 Our applications include maximum mean discrepancy (MMD), sliced Wasserstein-$1$, and the IPM induced by norm-bounded two-layer ReLU networks.

 \paragraph{Maximum mean discrepancy (MMD)}
 For a continuous positive definite kernel $K$ with reproducing kernel Hilbert space (RKHS) $\cH_K$, MMD uses its unit ball:
\begin{equation}\label{eq:intro-mmd}
 \MMD_K(p,q)\coloneqq\sup_{\norm{h}_{\cH_K}\le1}\left|\int h\,dp-\int h\,dq\right|.
\end{equation}
Thus small MMD controls every query in the RKHS unit ball. Equivalently, MMD is the distance between the kernel means of $p$ and $q$ \cite{gretton2012}, which connects query accuracy to matching averages of kernel features.

Different kernels express different statistical requirements. Polynomial kernels preserve finitely many moments, Gaussian kernels specify very smooth queries \cite{rasmussen2006}, and Mat\'ern kernels span a range of Sobolev smoothness \cite{wendland2005}. Neural tangent kernels (NTKs) use parameter-gradient features to connect MMD with neural-network training \cite{bietti2019ntk}. Both Gaussian and Mat\'ern kernels are characteristic, so zero MMD identifies a distribution \cite{sriperumbudur2010}; their private synthesis rates can nevertheless differ.

Private kernel-mean synthesis has consistency guarantees \cite{balog2018}, and private dataset reduction has empirical MMD bounds \cite{liu2022thesis}. Existing methods use Fourier or Hermite features \cite{harder2021,vinaroz2022}, Nystr\"om approximations \cite{lee2024}, or spatial partitions \cite{kreacic2023}. Methods using public perceptual or neural tangent features have excess-MMD guarantees under idealized optimization \cite{harder2023mepf,yang2024ntk}. Random features apply to many kernels \cite{rahimi2007}, including Laplace and Mat\'ern \cite{ahir2025features}.

Kernel-mean estimation without privacy \cite{muandet2017,sriperumbudur2012,tolstikhin2017} and population estimation under MMD \cite{singh2018} are well studied. Related results for Besov IPMs \cite{uppal2019} and Sobolev losses under local privacy \cite{albert2026} impose density assumptions. Our guarantees concern an arbitrary observed database; population guarantees additionally require control of sampling and generalization error \cite{bartlett2002}.

\paragraph{Sliced Wasserstein distance} Sliced Wasserstein-$1$ averages one-dimensional transport costs over projections \cite{rabin2012,bonneel2015}. Its projected transport problems reduce to sorting and support generative modeling \cite{deshpande2018}. On the cube, $m$ independent samples have expected sliced error $O_d(m^{-1/2})$, faster than the worst-case $m^{-1/d}$ rate for $1$-Wasserstein when $d\ge3$ \cite{nadjahi2020}. This is an IPM beyond MMD.

Previous private methods use noisy projected records \cite{rakotomamonjy2021}, sliced Wasserstein objectives \cite{liu2025sliced}, smoothed sliced $f$-divergences \cite{greenewald2024slicing}, Gaussian-smoothed sliced Wasserstein-$2$ gradient flows \cite{sebag2025}, and private sliced Wasserstein gradients \cite{rodriguez2026}. Our bounds concern sliced Wasserstein-$1$ relative to arbitrary input databases, with matching polynomial lower bounds, and improve the Wasserstein exponent for $d\ge2$ (Theorem~\ref{thm:sliced-wasserstein}).

\paragraph{IPM induced by ReLU networks} Our bound controls the expected largest change in average output across norm-bounded two-layer ReLU networks of any width when original data are replaced by DP synthetic data (Theorem~\ref{thm:neural-ipm}). Neural-network IPMs underlie GAN objectives \cite{arora2017gan}; related ReLU metrics support covariate balancing and fair representation learning \cite{park2025relu}.

Designing separate methods for these diverse query classes is challenging. We aim to answer the following question:
\emph{Can a general algorithmic framework generate DP synthetic data with utility bounds adapted to the query class?}

\subsection{Generalized moment matching and main results}\label{sec:main-results}

Generalized moment matching unifies earlier private synthesis methods and our constructions: estimate public moments privately, fit a probability measure, and sample equally weighted records. The general theorem converts moment-estimation error into IPM error. Our main technical contributions are the moment comparisons, private estimators, and lower bounds yielding sharp polynomial rates for Mat\'ern MMD, sliced Wasserstein-$1$, and ReLU IPMs.

\begin{definition}[Generalized moment matching]
For a public bounded Borel map $f:\Omega\to\R^D$ and a norm $\|\cdot\|_*$
chosen for the target IPM,
let $\theta(p)\coloneqq\int f\,dp$ be the generalized moment vector of a target
probability measure $p$ on $\Omega$.
The recovery input $z\in\R^D$ estimates $\theta(p)$; in
Algorithm~\ref{alg:features}, $p=p_X$ and
$z=\widehat\theta=\mathcal A(X)$ is a differentially private estimate.
For a nonempty finite public candidate set $G\subset\Omega$,
define the convex hull of $f(G)=\{f(g):g\in G\}$ by
\[
 \mathcal C_G\coloneqq\conv f(G)
 =\left\{\sum_{g\in G}a_gf(g):
 a_g\ge0,\ \sum_{g\in G}a_g=1\right\}.
\]
Thus $\mathcal C_G$ consists exactly of the moment vectors of probability
measures supported on $G$. Generalized moment matching recovers
$q\in\cP(G)$ with
\[
 \theta(q)\in\arg\min_{v\in \mathcal C_G}\|v-z\|_*.
\]
In the Euclidean norm, the minimizing moment vector is unique and denoted $\Pi_{\mathcal C_G}(z)$.
Algorithm~\ref{alg:features} chooses a measure for sampling and allows fitting tolerance $\eopt$.
\end{definition}

For example, when $d=1$, choose $f(x)=(x,x^2)$. Then
\[
 \theta(p)=\left(\int x\,dp(x),\int x^2\,dp(x)\right)
\]
contains the mean and second moment of $p$, or their empirical
counterparts when $p=p_X$. Given private estimates of these
quantities, recovery chooses probability weights on $G$ so that the
weighted mean $\sum_{g\in G}a_g g$ and second moment
$\sum_{g\in G}a_g g^2$ are as close as possible to the estimates in the
chosen norm.

Unless otherwise specified, all applications below use the Euclidean norm
for $\|\cdot\|_*$. For all probability measures $p,q$, assume moments $\theta(q)$ control query error:
\begin{equation}\label{eq:moment-comparison}
 d_{\mathcal F}(p,q)\le \etail+C
 \|\theta(p)-\theta(q)\|_*,
 \qquad \etail\ge0,\quad C>0.
\end{equation}
Here $\etail,C$ are public constants for approximation and comparison. With
\[\operatorname{dist}_*(v,\mathcal C_G)\coloneqq\inf_{u\in \mathcal C_G}\|v-u\|_*,\] assume the error from the finite grid is bounded:
\begin{equation}\label{eq:moment-geometry}
 \sup_p\operatorname{dist}_*(\theta(p),\mathcal C_G)\le\egrid.
\end{equation}
The hull is compact even for Borel moments, such as partition indicators.

\begin{algorithm2e}[H]
\caption{Differentially private generalized moment matching}\label{alg:features}
\small
\setlength{\abovedisplayskip}{3pt}
\setlength{\belowdisplayskip}{3pt}
\setlength{\abovedisplayshortskip}{0pt}
\setlength{\belowdisplayshortskip}{3pt}
\KwIn{{Database $X$;\newline
public moment map $f$, norm $\|\cdot\|_*$, and nonempty finite public
candidate set $G\subset\Omega$ satisfying
\eqref{eq:moment-comparison}--\eqref{eq:moment-geometry};\newline
specified randomized $(\eps,\delta)$-DP moment estimator $\mathcal A$;\newline
public fitting tolerance $\eopt\ge0$ and output size $m\ge1$.}}
\revision{Run $\widehat\theta\leftarrow\mathcal A(X)$.}

{1. Fit $v_q\coloneqq\sum_{g\in G}a_gf(g)$ with $a_g\ge0$
and $\sum_{g\in G}a_g=1$, satisfying
\[
 \|v_q-\widehat\theta\|_*\le
 \min_{v\in\mathcal C_G}\|v-\widehat\theta\|_*+\eopt.
\]
Choose the weights by a fixed measurable rule depending only on
$\widehat\theta$ and the public inputs.}

2. Form the probability measure
$q\coloneqq\sum_{g\in G}a_g\,\delta_g$,
where $\delta_g$ is the point mass at the candidate $g$.

3. Draw $Y_1,\ldots,Y_m$ independently from
$q$, conditional on $q$.

\KwOut{The synthetic dataset $Y=(Y_1,\ldots,Y_m)$, with weight $1/m$ per record.}
\end{algorithm2e}

For a specified IPM, one must choose $f$ and $G$ satisfying
\eqref{eq:moment-comparison}--\eqref{eq:moment-geometry}, together with a
differentially private estimator $\mathcal A$ of $\theta(p_X)$ whose error
is controlled in $\|\cdot\|_*$.

Define the private estimation and sampling errors by
\begin{align}
 \epriv&\coloneqq\sup_X\E\|\widehat\theta-
 \theta(p_X)\|_*,\label{eq:moment-risk}\\
 s_{\mathcal F}(m)&\coloneqq\sup_q\E d_{\mathcal F}
 \left(q,\frac1m\sum_{r=1}^m\delta_{U_r}\right),\qquad U_r\stackrel{\mathrm{i.i.d.}}{\sim} q.
 \label{eq:ipm-sampling}
\end{align}

The following theorem bounds the synthesis error.
\begin{theorem}[Generalized moment matching]\label{thm:moment-synthesis}
Let $\varnothing\ne\mathcal F\subset C(\Omega)$ satisfy
$\sup_{h\in\mathcal F}\|h\|_\infty<\infty$.
Under \eqref{eq:moment-comparison}--\eqref{eq:moment-geometry}, Algorithm~\ref{alg:features}
with an $(\eps,\delta)$-DP estimator $\mathcal A$ is $(\eps,\delta)$-DP, and
\begin{align}
 \sup_X\E d_{\mathcal F}(p_X,q)
 &\le \etail+C(\epriv+\egrid+\eopt),
 \label{eq:moment-oracle}\\
 \sup_X\E d_{\mathcal F}(p_X,p_Y)
 &\le \etail+C(\epriv+\egrid+\eopt)+s_{\mathcal F}(m).
 \label{eq:moment-dataset}
\end{align}
Here $C$ incorporates the comparison constant in \eqref{eq:moment-comparison} and absolute numerical factors.
\end{theorem}

{Here $\etail$ is the query-approximation error, $\egrid$ measures the
restriction to the candidate set $G$, $\epriv$ is the private
moment-estimation error, and $\eopt$ is the numerical fitting error.
The final term $s_{\mathcal F}(m)$ comes from sampling $m$ synthetic
records from the fitted measure $q$.}

\paragraph{Computational cost.}
Using the public resolutions and inverse-polynomial fitting tolerances
specified in the proofs, the concrete upper-bound constructions in
Section~\ref{sec:features} have polynomial arithmetic cost in $n$ and $m$
for fixed dimension, kernel or query parameters, and privacy parameters,
assuming efficient evaluation of the public feature functions.
The polynomial degree and constants may depend on these fixed parameters;
in particular, the cost is polynomial in $n$ when $m=\operatorname{poly}(n)$.

To see this, the candidate sets, moment vectors, and finite ReLU
query lists, including PREM's binary threshold queries, all have
polynomial size in $n$. Euclidean fitting minimizes
a quadratic function of the candidate weights, subject to nonnegativity
and total mass one. An unconstrained least-squares solution need not
satisfy these probability constraints, so its closed-form formula does
not solve the fit in general. Frank--Wolfe \cite{jaggi2013} gives a simple
constructive method for this constrained problem. Starting from a point
mass on $G$,
each iteration mixes the current distribution with a point mass at a
candidate chosen by linear minimization. Every update therefore preserves
nonnegative weights with total mass one. With our chosen tolerances,
the method uses polynomially many iterations, each scanning the candidate
set.
Fitting in the $\ell_\infty$ norm is a linear program.
The ReLU estimator additionally uses PREM, whose running time is
polynomial in the explicit domain and query-list sizes \cite{ghazi2025prem},
and matrix factorizations that can be found by semidefinite programming
\cite[Section~2.1]{matousek2020factorization}.

\paragraph{Utility rates.}

For the upper bounds based on the Gaussian mechanism
\cite[Theorem~3.22]{dwork2014}, assume $0<\eps\le1$ and
$0<\delta<1$, and define the effective sample size
\begin{equation}\label{eq:effective-size}
 \Neff\coloneqq\frac{n\eps}{\sqrt{2\log(1.25/\delta)}}.
\end{equation}
Table~\ref{tab:rates} summarizes the moment choices and rates under the privacy
regimes and sufficient output sizes specified in Section~\ref{sec:features},
with $\Neff\ge2$ for the upper bounds.
The Gaussian and fixed-degree polynomial exponent-one upper bounds also follow
from existing smooth-query results \cite{ding2026}
(see Remark~\ref{app:smooth-query-reduction}); our feature constructions give
direct specifications.

\begin{table}[t]
\centering
\caption{Moment choices and minimax rates for fixed dimension and query
parameters, up to logarithmic factors in $n\eps$ and $1/\delta$.}\label{tab:rates}
\small
\setlength{\tabcolsep}{4pt}
\begin{tabular}{@{}
>{\raggedright\arraybackslash}p{0.32\linewidth}
>{\raggedright\arraybackslash}p{0.30\linewidth}
>{\raggedright\arraybackslash}p{0.34\linewidth}
@{}}
\toprule
Kernel or IPM & Rate up to logarithmic factors & Moment choice\\
\midrule
Gaussian; polynomial, fixed degree & $(n\eps)^{-1}$
& Orthogonal feature projections\\
Mat\'ern, $\nu>0$ & $(n\eps)^{-\min\{1,\,\frac12+\frac\nu d\}}$
& \multirow{3}{=}{%
  $\left.\begin{array}{@{}c@{}}
   \rule{0pt}{2.1ex}\\[-0.4ex]
   \rule{0pt}{2.1ex}\\[-0.4ex]
   \rule{0pt}{2.1ex}
  \end{array}\right\}$\,Wavelet moments}\\
Laplace and ReLU NTK & $(n\eps)^{-\frac{d+1}{2d}}$ & \\
Sliced Wasserstein-$1$, $d\ge2$ & $(n\eps)^{-\frac{d+1}{2d}}$ & \\
ReLU IPMs & $(n\eps)^{-\min\{1,\frac{d+3}{2d}\}}$
& ReLU moments on a public net\\
\bottomrule
\end{tabular}
\end{table}

\paragraph{Earlier private synthesis mechanisms.}
The Private Signed Measure Mechanism (PSMM) and Private Measure Mechanism (PMM) in \cite{he2023}, as well as Chebyshev moment matching for $1$-Wasserstein distance and smooth queries
\cite{musco2025,ding2026} are specifications of this framework.
Appendix~\ref{app:general-norm} explains their connections.

\section{Setup and notation}\label{sec:setup}

Equip $\Omega=[-1,1]^d$, $d\ge1$, with Euclidean distance and
$\cP(\Omega)$ with the Borel sigma-algebra of the weak topology.

\begin{definition}[Differential privacy]\label{def:dp}
Two databases $X,X'\in\Omega^n$ are adjacent, written $X\sim X'$, if one row is replaced. For $\eps\ge0$ and $0\le\delta<1$, a randomized mechanism $\mathcal M$ is $(\eps,\delta)$-differentially private if, for every adjacent pair and measurable event $B$,
\[
 \Prob(\mathcal M(X)\in B)\le e^\eps\Prob(\mathcal M(X')\in B)+\delta.
\]
When $\delta=0$, the guarantee is called pure differential privacy.
\end{definition}

\paragraph{Notation.}
We write $\E$, $\Prob$, $\delta_x$, and $I_r$ for expectation, probability, a Dirac measure, and the identity matrix. Vector norms are Euclidean unless specified; $\norm{\cdot}_\infty$ is the uniform norm. The symbols $\conv(A)$, $P_V$, and $\Pi_C$ denote the convex hull and orthogonal and Euclidean projections, respectively. Write $\operatorname{dist}(v,C)=\inf_{u\in C}\norm{v-u}_2$. Logarithms are natural. The notation $a\lesssim_\theta b$ means $a\le C_\theta b$; $a\asymp_\theta b$ means bounds in both directions. Unless a dependence is specified, constants are independent of $n,m,\eps,\delta$ and truncation levels; they may depend on $d$, fixed query parameters, and the chosen public basis.

Throughout our results, the query class satisfies $\varnothing\ne\mathcal F\subset C(\Omega)$
and $\sup_{h\in\mathcal F}\norm{h}_\infty<\infty$; its IPM is defined in
\eqref{eq:general-ipm}.

\paragraph{Kernels.}
{A symmetric kernel $K:\Omega\times\Omega\to\R$ is positive definite if $\sum_{i,j}b_i b_jK(z_i,z_j)\ge0$ for every finite choice of points and real coefficients; equivalently, all its Gram matrices are positive semidefinite.} For a continuous positive definite kernel $K$, its RKHS is a Hilbert space of functions \cite{aronszajn1950} with the reproducing property $f(x)=\ip{f}{K(x,\cdot)}_{\cH_K}$. {The canonical feature map and kernel mean embedding are $\Phi(x)\coloneqq K(x,\cdot)$ and $\mu_p\coloneqq\int\Phi(x)\,dp(x)$.} Then
\begin{equation}\label{eq:mmd}
 \MMD_K(p,q)=\norm{\mu_p-\mu_q}_{\cH_K},\qquad
 \MMD_K(p,q)^2=\iint K(x,y)\,d(p-q)(x)\,d(p-q)(y).
\end{equation}
The first identity follows from Hilbert-space duality and is equivalent
to the query definition \eqref{eq:intro-mmd}. A kernel is characteristic
if zero MMD implies $p=q$. Otherwise, distinct measures can agree on every
query in the class, and MMD is a pseudometric. Put $\kappa\coloneqq\sup_x\sqrt{K(x,x)}$, which is finite by continuity and compactness.

We consider the Gaussian kernel $K_\sigma(x,y)\coloneqq\exp(-\norm{x-y}_2^2/(2\sigma^2))$, $\sigma>0$, and polynomial kernels $K_{p,c}(x,y)\coloneqq(c+x^\top y)^p$, $c\ge0$, $p\in\mathbb N$, $p\ge1$.

For smoothness $\nu>0$ and length scale $\ell>0$, put $\alpha\coloneqq\sqrt{2\nu}/\ell$. The normalized isotropic Mat\'ern kernel on $\R^d$ is $K_{\nu,\ell}(x,y)\coloneqq k_{\nu,\ell}(x-y)$, where
\begin{equation}\label{eq:matern-kernel}
 k_{\nu,\ell}(h)\coloneqq
 \begin{cases}
 \displaystyle\frac{2^{1-\nu}}{\Gamma(\nu)}
 (\alpha\norm{h}_2)^\nu\mathsf K_\nu(\alpha\norm{h}_2),&h\ne0,\\[4pt]
 1,&h=0.
 \end{cases}
\end{equation}
Here $\Gamma$ is the gamma function and $\mathsf K_\nu$ is the modified Bessel function of the second kind of order $\nu$. The value at the origin is the continuous extension.
Its RKHS has Sobolev order $\nu+d/2$ \cite[Chapter~10]{wendland2005};
Appendix~\ref{app:wavelet-proofs} gives the spectral density and Fourier convention.
This normalization gives $K_{\nu,\ell}(x,x)=1$. For $\nu=1/2$, the kernel is $\exp(-\norm{x-y}_2/\ell)$, the Euclidean Laplace kernel.

\paragraph{Minimax risk.}
We measure worst-case accuracy relative to the observed database.
For a query class $\mathcal F$, define the minimax risk over private
probability-measure outputs by
\begin{equation}\label{eq:minimax}
 \mathcal R_{\mathcal F}(n,\eps,\delta)\coloneqq
 \inf_{\mathcal M:\, (\eps,\delta)\text{-DP}}\sup_{X\in\Omega^n}
 \E d_{\mathcal F}(p_X,\mathcal M(X)),
\end{equation}
where the infimum is over measurable randomized mechanisms with values in $\cP(\Omega)$.
The empirical measure of a private synthetic dataset is private, and hence
\[
 \mathcal R_{\mathcal F}(n,\eps,\delta)
 \le\sup_X\E d_{\mathcal F}(p_X,p_Y).
\]
For the RKHS unit ball
$\mathcal F_K\coloneqq\{h\in\cH_K:\norm{h}_{\cH_K}\le1\}$ corresponding to the kernel $K$, we have
$d_{\mathcal F_K}=\MMD_K$ and write
$\mathcal R_K\coloneqq\mathcal R_{\mathcal F_K}$.
Expectations are taken over the randomness of the mechanisms.

\section{Specifications and minimax rates}\label{sec:features}

We now specify the moment maps, candidate sets, and private estimators in
Algorithm~\ref{alg:features} and derive the corresponding minimax rates.

\subsection{Kernel feature moments}\label{sec:feature-specialization}

We first give a feature specification, then apply it to Gaussian and
polynomial kernels. For orthogonal features, Euclidean moment error equals
the error in the projected kernel mean.
Grouping coordinates into blocks lets us allocate noise according to block
dimension and sensitivity.

Let $H$ be a Hilbert space and $\Psi:\Omega\to H$ a continuous feature map with
$K(x,y)=\ip{\Psi(x)}{\Psi(y)}_H$. Choose a public finite-dimensional
subspace $V\subset H$. Let $P_V:H\to V$ be the orthogonal projection
onto $V$, and define the projected feature map by
$(P_V\Psi)(x)\coloneqq P_V(\Psi(x))$.
 Choose orthonormal coordinates
$\phi=(\phi_j)_{j\in\mathcal I}$ for $P_V\Psi$, indexed by a finite set
$\mathcal I$, with blocks
$\phi_j:\Omega\to\R^{r_j}$, $r_j\ge1$. Suppose
\begin{align}
 \sup_{x,y}\|\phi_j(x)-\phi_j(y)\|_2&\le\Delta_j
 \quad(j\in\mathcal I),\label{eq:block-diameter}\\
 \sup_{x,y}\|(I-P_V)(\Psi(x)-\Psi(y))\|_H&\le\etail,
 &\sup_x\min_{g\in G}\|\phi(x)-\phi(g)\|_2&\le\egrid.
 \label{eq:feature-errors}
\end{align}
Omit constant blocks and henceforth let $\mathcal I$ index the remaining ones;
put $D=\sum_{j\in\mathcal I}r_j$. Replacing one
record changes the block mean by at most $\Delta_j/n$.

\revision{Use Algorithm~\ref{alg:features} with the Euclidean norm and Gaussian
estimator. The private estimate $\widehat\theta$ may lie outside
$\mathcal C_G$, so it need not be the moment vector of a probability
measure on $G$. To recover such a measure, Step~1 chooses probability
weights whose moment vector approximates the Euclidean projection:
$\|v_q-\Pi_{\mathcal C_G}(\widehat\theta)\|_2\le\eopt$.
Since Euclidean projection does not increase distances, this gives
\[
 \|v_q-\theta(p_X)\|_2
 \le\|\widehat\theta-\theta(p_X)\|_2+\egrid+\eopt.
\]
Thus the private estimation error enters with coefficient one.} Choose
$f=\phi$ endowed with the Euclidean norm and set
\begin{equation}\label{eq:feature-allocation}
 \Gamma\coloneqq\sum_{j\in\mathcal I}\sqrt{r_j}\Delta_j,\quad
 \Sigma\coloneqq\bigoplus_{j\in\mathcal I}\frac{\Gamma\Delta_j}{\sqrt{r_j}}I_{r_j}.
\end{equation}
Use the private Gaussian estimator \[\mathcal A(X)=\theta(p_X)+\Sigma^{1/2}\xi/\Neff,
\quad \xi\sim\mathcal N(0,I_D).\]
The added Gaussian noise has covariance $\Sigma/\Neff^2$ and is generally
anisotropic: it is isotropic within each block, but its variance can differ
between blocks according to their dimensions and sensitivity bounds.
The omitted features contribute $\etail$ in the comparison inequality.
Here $C=1$, and the private moment-estimation error is
$\epriv\le\Gamma/\Neff$.

{For a single block of dimension $r$ with diameter bound
$\Delta$, these formulas give $\Gamma=\sqrt r\,\Delta$ and
$\Sigma=\Delta^2I_r$, recovering the usual isotropic Gaussian mechanism.}

\begin{corollary}[Error bound for orthogonal kernel features]\label{thm:feature-oracle}
Under \eqref{eq:block-diameter}--\eqref{eq:feature-errors}, the covariance
$\Sigma$ in \eqref{eq:feature-allocation} satisfies the Gaussian privacy
condition \eqref{eq:moment-privacy} and $\tr(\Sigma)=\Gamma^2$.
Consequently, Algorithm~\ref{alg:features} with the Gaussian estimator
$\mathcal A$ defined above is $(\eps,\delta)$-DP and satisfies
\begin{equation}\label{eq:feature-oracle}
 \sup_X\E\MMD_K(p_X,p_Y)
 \le\etail+\egrid+\eopt+\frac{\Gamma}{\Neff}+\frac{\kappa}{\sqrt m}.
\end{equation}
The same bound without the sampling term holds for the internal measure.
\end{corollary}

To compute the projection fit, apply Frank--Wolfe to
$\frac12\|v-\widehat\theta\|_2^2$ over $v\in\mathcal C_G$; starting from a
point mass on $G$, each iteration mixes the current weights with a point
mass at a candidate chosen by linear minimization.
Appendix~\ref{app:moment-fitting} gives an explicit iteration count ensuring
the required distance $\eopt$ to the projection for every $\eopt>0$.

For a continuous positive definite kernel on the compact domain $\Omega$,
the feature map can be approximated in a finite-dimensional subspace.
Appendix~\ref{app:nystrom} constructs such coordinates by the Nystr\"om
method \cite{chatalic2022}, using public landmark points.
Appendix~\ref{app:feature-proof} proves the error bound above and shows
that the chosen covariance minimizes total noise variance within the
family of block covariances considered there.

{For the Gaussian kernel $K_\sigma$, the explicit feature expansion
\cite{cotter2011} has coordinates
\[
 \Psi_\alpha(x)=e^{-\|x\|_2^2/(2\sigma^2)}
 \frac{x^\alpha}{\sigma^{|\alpha|}\sqrt{\alpha!}},
 \qquad \alpha\in\mathbb N_0^d,
\]
where $|\alpha|=\sum_{i=1}^d\alpha_i$,
$\alpha!=\prod_{i=1}^d\alpha_i!$, and
$x^\alpha=\prod_{i=1}^d x_i^{\alpha_i}$.
We group these coordinates by total degree $|\alpha|$ and keep those with
$|\alpha|<J$ for a public cutoff $J$.
The polynomial kernel has a finite monomial expansion, so no feature
truncation is needed.}

For Gaussian and polynomial kernels, the feature constructions in
Appendix~\ref{app:smooth} keep $\Gamma$ bounded for fixed dimension and
kernel parameters and give $\etail+\egrid=O_{d,K}(\Neff^{-1})$.
Choosing $0<\eopt\le\Neff^{-1}$ and
$m\ge\lceil\Neff^2\rceil$ makes the fitting and sampling errors
$O_{d,K}(\Neff^{-1})$ as well. Corollary~\ref{thm:feature-oracle}
therefore yields the upper bound in the following result.

\begin{corollary}[Near-minimax examples]\label{cor:smooth-optimal}
Fix a Gaussian kernel or a polynomial kernel from Section~\ref{sec:setup}.
{Let $c_0=1/[8(e-1)]$.}
For $0<\eps\le1$, $0<\delta\le c_0(e^\eps-1)$, $\Neff\ge2$,
and every public integer $m\ge\left\lceil\Neff^2\right\rceil$,
there are constants $0<c_K\le C_{d,K}<\infty$ such that
\begin{equation}\label{eq:smooth-minimax}
 \frac{c_K}{n\eps}\le\mathcal R_K(n,\eps,\delta)
 \le\sup_X\E\MMD_K(p_X,p_Y)
 \le\frac{C_{d,K}}{\Neff},
\end{equation}
where $Y$ is produced by public choices in Algorithm~\ref{alg:features}.
Here $c_K,C_{d,K}$ are independent of $n,m,\eps,\delta$.
The upper bound holds for every $0<\delta<1$ with $\Neff\ge2$ and
$m\ge\left\lceil\Neff^2\right\rceil$.
\end{corollary}

\subsection[Matern, Laplace, and neural network kernels]{Mat\'ern, Laplace, and neural network kernels}\label{sec:wavelets}\label{sec:laplace-ntk}

The Mat\'ern theorem is a common reference point for the results below:
its connection to Sobolev smoothness supports several other examples.
The bounds in Theorem~\ref{thm:matern-rates} directly give the Laplace
case by taking $\nu=1/2$, and an equivalence of RKHS norms transfers both
upper and lower bounds to the specified ReLU NTK
(Corollary~\ref{cor:laplace-neural}). For sliced Wasserstein-$1$ in
dimension $d\ge2$, comparison with Laplace MMD gives the upper bound;
the matching lower bound requires a separate geometric argument
(Theorem~\ref{thm:sliced-wasserstein}). For the norm-bounded ReLU IPM,
comparison with Mat\'ern MMD at $\nu=3/2$ gives the lower bound, while
the upper bound uses the adaptive private estimator developed in
Section~\ref{sec:neural-ipm}
(Theorem~\ref{thm:neural-ipm}).

We control Mat\'ern MMD by matching averages of wavelet functions.
Like histogram bins at different scales, each wavelet is nonzero only
in a small region, and each record contributes to only a bounded number
of wavelet averages at any scale. This \emph{localization}, together with
bounds on wavelet values, controls both the effect of replacing one
record and the sum of the absolute wavelet averages at each scale.
The first bound determines the privacy noise; the second lets us reduce
the error in noisy fine-scale moments through denoising. Both hold for
arbitrary probability measures, with no density assumption.

\paragraph{Wavelet moments.}
{The Mat\'ern RKHS has Sobolev order $s\coloneqq\nu+d/2$.}
{Fix, independently of the data, a real
compactly supported orthonormal wavelet basis $\mathcal W$ of
$L^2(\R^d)$ with regularity and number of vanishing moments greater than
$s+d+2$.} A sufficiently regular tensor-product Daubechies basis is one
example \cite[Chapters~6--7]{daubechies1992}. At each scale, keep only
functions whose supports intersect $\Omega$, and include the scaling functions at level zero. At level $j$, the
$d_j\le C_{\mathrm{dim}}2^{jd}$ functions $(\psi_{j,k})_{k=1}^{d_j}$ have support diameters
$O(2^{-j})$ and bounded overlap. For a finite signed measure $\rho$, put
$\beta_{j,k}(\rho)\coloneqq\int\psi_{j,k}\,d\rho$ and
$\beta_j(\rho)\coloneqq(\beta_{j,k}(\rho))_{k=1}^{d_j}$. Localization gives
\begin{equation}\label{eq:wavelet-geometry}
 \norm{\beta_j(p)}_1\le R_j\coloneqq C_R2^{jd/2},\qquad
 \sup_{x,y}\norm{\beta_j(\delta_x-\delta_y)}_2
 \le\Delta_j\coloneqq C_\Delta2^{jd/2}.
\end{equation}
The constants depend on $d$ and the public wavelets.

{Choose a public resolution $J$ and positive level weights
$w_0,\ldots,w_J$ with $\sum_{j=0}^Jw_j=1$.
The factor $2^{-js}$ weights coefficient errors according to the kernel's
Sobolev smoothness. The weights $w_j$ allocate Gaussian noise across
levels: larger $w_j$ means less noise at level $j$.
Set}
\begin{align*}
 f_J(x)&\coloneqq(2^{-js}\psi_{j,k}(x))_{j\le J,k\le d_j},
 &D_J&\coloneqq\sum_{j\le J}d_j,\\
 \Sigma_J&\coloneqq\bigoplus_{j=0}^J
 2^{-2js}\frac{\Delta_j^2}{w_j}I_{d_j}.
\end{align*}
Use $f_J$, $\Sigma_J$, and the Euclidean norm in Algorithm~\ref{alg:features},
with $\theta_J(p)\coloneqq\int f_J\,dp$. The wavelet characterization of
Sobolev norms \cite[Theorem~1.20]{triebel2008}, together with localization,
gives
\begin{equation}\label{eq:wavelet-moment-certificate}
 \MMD_{K_{\nu,\ell}}(p,q)
 \le C\left(2^{-\nu J}+
 \norm{\theta_J(p)-\theta_J(q)}_2\right).
\end{equation}
Thus the omitted levels contribute $\etail=C2^{-\nu J}$ uniformly over
probability measures; $C$ depends on $d,\nu,\ell$ and the public wavelets.

\paragraph{Denoising and public tuning.}
At fine scales, we estimate many wavelet averages, and the added privacy
noise can add up. The bound in \eqref{eq:wavelet-geometry} tells us that
only a limited number of these averages can be large. We use this
information to reduce the noise: set small noisy averages to zero and
shrink the others toward zero. This procedure is called
\emph{soft thresholding} \cite{donoho1995soft}.

Let $Z\coloneqq\theta_J(p_X)+\Sigma_J^{1/2}\xi/\Neff$, where
$\xi\sim\mathcal N(0,I_{D_J})$, be the private Gaussian release.
For $0<\nu<d/2$, use the public levelwise denoising map $\mathcal D$
specified in Appendix~\ref{app:wavelet-risk} using
Lemma~\ref{lem:denoising}. For $\nu\ge d/2$, the noisy moments already
give the stated rate, so we use $Z$ directly.
The private moment estimator used in Algorithm~\ref{alg:features} is
\[
 \mathcal A(X)\coloneqq
 \begin{cases}
  \mathcal D(Z),&{0<\nu<d/2},\\
  Z,&{\nu\ge d/2}.
 \end{cases}
\]
Lemma~\ref{lem:denoising} and Appendix~\ref{app:wavelet-risk} give the
thresholds, error bounds, and public choices of resolution and level weights.

The following theorem gives upper bounds for this wavelet construction
and minimax lower bounds for Mat\'ern MMD.

\begin{theorem}[Mat\'ern rates]\label{thm:matern-rates}
Fix $\nu,\ell>0$ and a public wavelet basis $\mathcal W$ as above.
For $0<\eps\le1$, $0<\delta<1$, $\Neff\ge2$, and public integer $m\ge1$,
the wavelet specification of Algorithm~\ref{alg:features}, with the
public choices in Appendix~\ref{app:wavelet-risk}, produces an
$(\eps,\delta)$-DP dataset $Y$ satisfying
\begin{equation}\label{eq:matern-rate}
 \sup_X\E\MMD_{K_{\nu,\ell}}(p_X,p_Y)\le
 \frac1{\sqrt m}+C_{d,\nu,\ell,\mathcal W}
 \begin{cases}
 \Neff^{-1/2-\nu/d},&0<\nu<d/2,\\
 \Neff^{-1}\,\log(e\Neff),&\nu=d/2,\\
 \Neff^{-1},&\nu>d/2.
 \end{cases}
\end{equation}
For $n\eps\ge2$, $0<\eps\le1$, and
$0\le\delta\le c_0(e^\eps-1)$,
\begin{equation}\label{eq:matern-lower}
 \mathcal R_{K_{\nu,\ell}}(n,\eps,\delta)
 \ge c_{d,\nu,\ell} (n\eps)^{-\min\{1,\,1/2+\nu/d\}}.
\end{equation}
The positive constants depend only on their subscripts.
\end{theorem}

{The polynomial rate for the private error improves with
$\nu$ until $\nu=d/2$ and then remains $\Neff^{-1}$, with the extra
$\log(e\Neff)$ factor at the boundary. The term $m^{-1/2}$ accounts for
sampling synthetic records from $q$.}

For $\Neff\ge2$, sampling error does not change the rates if
\begin{equation}\label{eq:matern-output-size}
 m\ge\left\lceil
 \Neff^{\min\{2,\,1+2\nu/d\}}
 \right\rceil
\end{equation}
in every regime.
Appendix~\ref{app:matern-complete} gives the upper and lower proofs.

\paragraph{Laplace and neural network kernels.}

The wavelet construction applies to other kernels with equivalent RKHS norms.
The Euclidean Laplace kernel
$K_{\mathrm L,\ell}(x,y)\coloneqq\exp(-\norm{x-y}_2/\ell)$, $\ell>0$,
is a special case of the Mat\'ern kernel with $\nu=1/2$.
The Laplace kernel is used in nonlinear regression and classification, where
random feature approximations reduce the cost of kernel-based learning
\cite{ahir2025features}.

{The neural tangent kernel (NTK) links kernel regression to
the training of wide ReLU networks near initialization
\cite{bietti2019ntk}; its empirical features also support private
generative modeling by matching noisy feature means \cite{yang2024ntk}.
We consider a two-layer ReLU network with
independent standard Gaussian input weights, output weights, and hidden-layer
biases, all trainable. The NTK compares distributions through
parameter-gradient features.} For $x,y\in\Omega$, put
\[
 u_x\coloneqq(1,x),\qquad
 \theta_{xy}\coloneqq\arccos\frac{u_x^\top u_y}{\|u_x\|_2\|u_y\|_2}.
\]
Its normalized infinite-width NTK is
\begin{equation}\label{eq:relu-ntk}
 K_{\mathrm{NTK}}(x,y)\coloneqq
 \frac{\norm{u_x}_2\norm{u_y}_2}{2\pi(d+1)}
 \left[\sin\theta_{xy}+2(\pi-\theta_{xy})\cos\theta_{xy}\right].
\end{equation}
Its diagonal is $(1+\norm{x}_2^2)/(d+1)\le1$.
Its RKHS norm is equivalent to the Sobolev norm of order $(d+1)/2$,
as is the Laplace RKHS. Appendix~\ref{app:neural-comparison} proves this
comparison on the cube, including the effect of the trained bias, using
spherical ReLU multipliers \cite[Section~2]{olafsson2012cosine} and
Sobolev extension \cite{rychkov1999}.

{For Laplace, we use the Mat\'ern construction with
$\nu=1/2$ and length scale $\ell$. For the NTK, we use the
unit-length-scale Laplace construction; the RKHS comparison transfers
its utility bound and the Laplace minimax lower bound. These bounds
control queries in the RKHS unit ball; Section~\ref{sec:neural-ipm}
treats average outputs of norm-bounded ReLU networks.}

\begin{corollary}[Laplace and NTK rates]\label{cor:laplace-neural}
Let $K$ be $K_{\mathrm L,\ell}$ or $K_{\mathrm{NTK}}$.
Under the upper-bound settings of Theorem~\ref{thm:matern-rates},
$Y$ is $(\eps, \delta)$-DP, and there is
\begin{equation}\label{eq:laplace-neural-upper}
 \sup_X\E\MMD_K(p_X,p_Y)\le\frac1{\sqrt m}
 +C_{d,K,\mathcal W}
 \begin{cases}
 \Neff^{-1}\log(e\Neff),&d=1,\\
 \Neff^{-\frac{d+1}{2d}},&d\ge2.
 \end{cases}
\end{equation}
Under the lower-bound assumptions of Theorem~\ref{thm:matern-rates},
\begin{equation}\label{eq:laplace-neural-lower}
 \mathcal R_K(n,\eps,\delta)
 \ge c_{d,K}(n\eps)^{-\frac{d+1}{2d}}.
\end{equation}
Here $C_{d,K,\mathcal W},c_{d,K}>0$ are independent of $n,m,\eps,\delta$.
\end{corollary}

\subsection{Sliced Wasserstein-1}\label{sec:sliced-wasserstein}

For sliced Wasserstein-$1$ \cite{rabin2012,bonneel2015,nadjahi2020},
let $\tau$ be uniform on $\mathbb S^{d-1}$ and let $\theta_\#p$ be the
distribution of $\theta^\top x$ for $x\sim p$. With $W_1$ denoting the usual
Wasserstein distance for Euclidean transport cost, define
\begin{equation}\label{eq:sliced-wasserstein}
 \SW_1(p,q)\coloneqq\int_{\mathbb S^{d-1}}
 W_1(\theta_\#p,\theta_\#q)\,d\tau(\theta).
\end{equation}

{For $d\ge2$, the Laplace-kernel moments from the preceding
subsection also control sliced Wasserstein distance through the comparison}
\begin{equation}\label{eq:sw-matern-comparison}
 \SW_1(p,q)\le C_d\MMD_{K_{1/2,1}}(p,q),
\end{equation}
{where $K_{1/2,1}$ is the unit-length-scale Laplace kernel.
Consequently, the same private wavelet estimator gives the sliced
Wasserstein upper bound. A separate construction of databases that are
separated in sliced Wasserstein distance gives the matching lower bound.
Appendix~\ref{app:sliced-wasserstein} proves the comparison and both bounds.}

\begin{theorem}[Sliced Wasserstein-$1$ rates]\label{thm:sliced-wasserstein}
Fix $d\ge2$ and a public output size $m\ge1$.
For $0<\eps\le1$, $0<\delta<1$, and $\Neff\ge2$, the Laplace-kernel wavelet
specialization of Algorithm~\ref{alg:features} is $(\eps,\delta)$-DP and satisfies
\begin{equation}\label{eq:sw-upper}
 \sup_X\E\SW_1(p_X,p_Y)\le C_{d,\mathcal W}\Neff^{-\frac{d+1}{2d}}+\sqrt{d/m}.
\end{equation}
There are $c_d,T_d>0$ such that,
if $0<\eps\le1$, $0\le\delta\le\eps/100$, and $n\eps\ge T_d$,
every $(\eps,\delta)$-DP probability-measure output $\mathcal M$ satisfies
\begin{equation}\label{eq:sw-lower}
 \sup_X\E\SW_1(p_X,\mathcal M(X))\ge c_d(n\eps)^{-\frac{d+1}{2d}}.
\end{equation}
Here $C_{d,\mathcal W},c_d,T_d>0$ are independent of $n,m,\eps,\delta$.
\end{theorem}

{The term $\sqrt{d/m}$ is the error from sampling the synthetic
records. Taking $m\ge\lceil\Neff^{\frac{d+1}{d}}\rceil$ makes it no larger
than the privacy term, up to a constant depending on $d$. Under the common
privacy regime of the two bounds, their polynomial rates match, with a
factor $[2\log(1.25/\delta)]^{\frac{d+1}{4d}}$ between the displayed rates.}

\subsection{ReLU IPMs}\label{sec:neural-ipm}

We study the IPM induced by norm-bounded two-layer ReLU networks,
motivated by neural-network discriminators in GANs
\cite{arora2017gan,park2025relu}. For $\sigma(t)\coloneqq t_+$, consider the
two-layer variation class \cite{bach2017convex}
\begin{equation}\label{eq:neural-class}
 \mathcal F_{\mathrm{NN}}\coloneqq
 \left\{\sum_{j=1}^{r}a_j\sigma(w_j^\top x+b_j):
 r\ge1,\ \sum_j|a_j|\le1,\ \norm{w_j}_1+|b_j|\le1\right\}.
\end{equation}
The width, hidden directions, and biases may vary. The following lemma gives an equivalent definition.

\begin{lemma}[Single-neuron representation]\label{lem:neural-single-neuron}
For any probability measures $p,q$ on $\Omega=[-1,1]^d$, the IPM
$d_{\mathrm{NN}}\coloneqq d_{\mathcal F_{\mathrm{NN}}}$ induced by
\eqref{eq:neural-class} satisfies
\begin{equation}\label{eq:neural-ipm}
 d_{\mathrm{NN}}(p,q)=
 \sup_{\norm{w}_1+|b|\le1}
 \left|\int\sigma(w^\top x+b)\,d(p-q)(x)\right|.
\end{equation}
\end{lemma}

The lemma lets us estimate single-neuron averages while controlling every
network in $\mathcal F_{\mathrm{NN}}$.
The upper bound uses adaptive private estimation of moments on a public net.

\paragraph{ReLU moments.}
{
Write $v=(b,w)$. The allowed parameter vectors satisfy
$\|v\|_2\le\|v\|_1=|b|+\|w\|_1\le1$.
By positive homogeneity of ReLU, a bound on differences of neuron averages
for all Euclidean-unit directions also controls every allowed parameter vector.
Choose a finite public net $V_0\subset\mathbb S^d$ approximating these
unit directions, and use Algorithm~\ref{alg:features} with the public moment map
}
\[
 f(x)\coloneqq\bigl(\sigma(v^\top(1,x))\bigr)_{v\in V_0} \in \bigl(\mathbb R^{|V_0|}, \|\cdot\|_\infty\bigr),
\]
where each coordinate of $\theta(p)$ is the average output under $p$ of
one of these public neurons. The maximum norm measures the largest error
among these coordinate averages.
Choose the data-grid and neuron-net mesh sizes and the fitting tolerance to
be at most $\Neff^{-2}$.

\paragraph{Private ReLU moment estimator.}
After rounding records to the public grid, a first PREM release
\cite{ghazi2025prem} selects a small set of neuron parameter directions to represent the public net.
Two Gaussian releases estimate the average neuron outputs at the selected
directions and how those averages change when the neuron parameters vary.
Their error bounds use matrix factorization \cite{matousek2020factorization}
and geometric discrepancy \cite{matousek1995halfspaces,matousek1996zonoids}.
These estimates approximate the averages for nearby directions, and a
fourth PREM release estimates the remaining approximation errors.
Combining these estimates and clipping them to the range of possible
neuron averages gives $\mathcal A(X)$. Each release is
$(\eps/4,\delta/4)$-DP for every prior history, so adaptive composition
\cite[Theorem~B.1]{dwork2014} gives $(\eps,\delta)$-DP overall.
Appendix~\ref{app:neural-ipm} gives the full
construction and analysis.

\begin{theorem}[Full ReLU network class]\label{thm:neural-ipm}
Fix $d\ge1$ and a public output size $m\ge1$.
For $0<\eps\le1$, $0<\delta<1$, and $\Neff\ge2$, Algorithm~\ref{alg:features},
with the ReLU specification above, produces an
$(\eps,\delta)$-DP dataset satisfying
\begin{equation}\label{eq:neural-upper}
 \sup_X\E d_{\mathrm{NN}}(p_X,p_Y)
 \le C_d\Neff^{-\min\{1,\frac{d+3}{2d}\}}
 [\log(e\Neff)\log(e/\delta)]^5
 +4\sqrt{\frac{2\log(2(d+1))}{m}}.
\end{equation}
Conversely, if $0<\eps\le1$, $n\eps\ge2$, and
$0\le\delta\le c_0(e^\eps-1)$, every $(\eps,\delta)$-DP probability-measure
output $\mathcal M$ satisfies
\begin{equation}\label{eq:neural-lower}
 \sup_X\E d_{\mathrm{NN}}(p_X,\mathcal M(X))
 \ge c_d(n\eps)^{-\min\{1,(d+3)/(2d)\}}.
\end{equation}
\end{theorem}
{Up to logarithmic factors, the privacy error has order
$(n\eps)^{-1}$ for $d\le3$ and $(n\eps)^{-(d+3)/(2d)}$ for $d\ge4$.
The final $m^{-1/2}$ term comes from sampling the synthetic records.
Taking $m\ge\lceil\Neff^2\rceil$ is sufficient to preserve the privacy
exponent. Under the common privacy regime of the upper and lower bounds,
their polynomial exponents match in every fixed dimension.}

{
For the lower bound, a kernel comparison gives
$\MMD_{K_{3/2,1}}\le C_d d_{\mathrm{NN}}$
(Appendix~\ref{app:neural-ipm}), so Theorem~\ref{thm:matern-rates} applies with
$\nu=3/2$.
}

\paragraph{AI Disclosure}
The authors developed the general framework for differentially private synthetic data generation under integral probability metrics (IPMs), initially working out the details for the Matérn kernel. The authors also independently selected the concrete examples of IPMs considered in this paper. ChatGPT was used to assist in developing additional examples and performing routine calculations and in identifying the Gaussian denoising method used in Appendix \ref{mat:app-denoising}. All AI-assisted results were reviewed and verified by the authors.

\paragraph{Acknowledgments} Y.Z. was partially supported by the Simons Grant MPS-TSM-00013944 and NSF DMS-2606337. This work was carried out while Y.Z. was visiting the Simons Institute for the Theory of Computing during the Spectral Theory Beyond Graphs program in Fall 2026. 

\label{main-text-end}

% Use this manuscript's bibliography from either the project root or ALT2027/.
\IfFileExists{ALT2027/refs.bib}
  {\bibliography{ALT2027/refs}}
  {\bibliography{refs}}

\appendix

\paragraph{Organization of the appendices}
Appendix~\ref{app:general-norm} relates the framework to PSMM, PMM, and Chebyshev moment matching.
Appendix~\ref{app:projection} proves the general moment-matching theorem and its kernel-feature specialization, develops public kernel-feature coordinates, and gives the fitting and sampling details.
Appendix~\ref{app:kernel-tools} establishes a general kernel lower bound.
Appendix~\ref{app:smooth} gives the proofs for Gaussian and polynomial kernels.
Appendix~\ref{app:matern-complete} develops the wavelet and denoising analysis for Mat\'ern kernels and derives the Laplace and neural tangent kernel consequences.
Appendix~\ref{app:sliced-wasserstein} proves the IPM representation and upper and lower bounds for sliced Wasserstein-$1$.
Finally, Appendix~\ref{app:neural-ipm} gives the adaptive private moment estimator and proves the upper and lower bounds for ReLU IPMs.

Table~\ref{tab:proof-dependencies} lists the ingredients used in the
upper and lower bounds for each IPM.

\begin{table}[ht]
\centering
\caption{Proof dependencies for the upper and lower bounds. All upper bounds use the general moment-matching theorem.}
\label{tab:proof-dependencies}
\small
\renewcommand{\arraystretch}{1.15}
\setlength{\tabcolsep}{4pt}
\begin{tabular}{@{}
>{\raggedright\arraybackslash}p{0.18\linewidth}
>{\raggedright\arraybackslash}p{0.36\linewidth}
>{\raggedright\arraybackslash}p{0.38\linewidth}
@{}}
\toprule
Kernel or IPM & Upper-bound ingredients & Lower-bound ingredients \\
\midrule
Gaussian / polynomial
& Orthogonal kernel features (Corollary~\ref{thm:feature-oracle})
& General kernel lower bound (Theorem~\ref{thm:general-lower})
\\
\addlinespace
Mat\'ern
& Wavelet moments, with $\ell_1$ denoising for $0<\nu<d/2$
(Theorem~\ref{thm:matern-rates})
& General kernel lower bound \& geometric construction
(Appendix~\ref{app:lower-proofs})
\\
\addlinespace
Laplace
& Mat\'ern upper bound with $\nu=1/2$
& Mat\'ern lower bound with $\nu=1/2$
\\
\addlinespace
ReLU NTK
& Mat\'ern upper bound with $\nu=1/2$ via RKHS norm equivalence
(Lemma~\ref{lem:relu-ntk-comparison})
& Mat\'ern lower bound with $\nu=1/2$ via RKHS norm equivalence
(Lemma~\ref{lem:relu-ntk-comparison})
\\
\addlinespace
Sliced Wasserstein-$1$
& Laplace MMD upper bound via \eqref{eq:sw-matern-comparison}
& Separate information-theoretic argument
(Appendix~\ref{app:sliced-wasserstein})
\\
\addlinespace
ReLU IPM
& Adaptive private moment estimation (Appendix~\ref{app:neural-ipm}) and Theorem~\ref{thm:moment-synthesis}
& Mat\'ern lower bound with $\nu=3/2$ via IPM comparison
(Appendix~\ref{app:neural-ipm})
\\
\bottomrule
\end{tabular}
\end{table}

\section{Connection to earlier private synthesis methods}
\label{app:general-norm}\label{app:chebyshev-containment}

We now connect the framework to earlier private synthesis methods,
starting with moments that record the probability mass in each
partition cell.

\paragraph{Partition moments for PSMM and PMM.}

PSMM and PMM \cite{he2023} use cell masses as moments.
Let $\Omega=[0,1]^d$, partition it into public cells $(A_i)_{i=1}^D$
of diameter at most $\eta$, and choose one representative $g_i\in A_i$.
Set
\[
 f=(\mathbf1_{A_i})_{i=1}^D,\qquad G=\{g_1,\ldots,g_D\}.
\]
Then $\theta(p)$ is the vector of cell probabilities, and $\mathcal C_G$
is the probability simplex, so $\egrid=0$.
Use the bounded Lipschitz norm
\[
 \|a\|_*\coloneqq
 \sup_{\substack{h:G\to\R,\ \operatorname{Lip}(h)\le1\\
                  \|h\|_\infty\le\sqrt d}}
 \left|\sum_{i=1}^D a_i h(g_i)\right|.
\]
On differences of probability vectors, this equals the $1$-Wasserstein
distance between their measures on $G$ by Kantorovich--Rubinstein
duality \cite{villani2009}. Moving the mass of each cell to its
representative gives
\begin{equation}\label{eq:partition-wasserstein-certificate}
 W_1(p,q)\le2\eta+\|\theta(p)-\theta(q)\|_*.
\end{equation}
Thus the comparison holds with $\etail=2\eta$ and $C=1$.

PSMM adds independent Laplace noise of scale $2/\eps$ to each cell
count and divides by $n$ to estimate the cell probabilities privately.
Step~1 then fits a probability vector to these estimates by linear
programming in $\|\cdot\|_*$, giving $\eopt=0$.
PMM instead adds discrete Laplace noise to counts on a public hierarchy
and repairs them from top to bottom to be nonnegative and consistent
across levels.
For our fixed database size, keep the root count at the public value $n$
and double the cited noise scales to account for replacement adjacency.
The repaired leaf counts sum to $n$; dividing by $n$ gives
$\mathcal A_{\mathrm{PMM}}(X)\in\mathcal C_G$, so no projection is needed
and $\eopt=0$.
Both estimators are pure $\eps$-DP.

With the public partitions and tuning in \cite{he2023}, for $d\ge2$,
both methods have $\epriv=\widetilde O((n\eps)^{-1/d})$ and
$\eta\asymp(n\eps)^{-1/d}$. Theorem~\ref{thm:moment-synthesis} therefore gives
\[
 \sup_X\E W_1(p_X,q)=\widetilde O((n\eps)^{-1/d}).
\]
The estimator constructions and their error analyses are given in
\cite{he2023}.

\paragraph{Chebyshev moments for multivariate smooth queries.}
 {Let} $\Omega=[-1,1]^d$ and {let $k\ge1$ be} an integer{. The query functions and all their derivatives through order $k$ are bounded in absolute value by one:}
\begin{equation}\label{eq:ding-query-class}
 \mathcal F_k=\{g\in C^k(\Omega):
   \max_{|\alpha|\le k}\|\partial^\alpha g\|_\infty\le1\},\qquad
 d_k(\mu,\lambda)=\sup_{g\in\mathcal F_k}
   |\textstyle\int g\,d(\mu-\lambda)|.
\end{equation}
For $J\ge k$, {use} the rescaled Chebyshev polynomial moment map {$f_J$ from \cite[Theorem~9]{ding2026}; see also \cite{musco2025}. Its polynomials have} degree at most $J$ in each coordinate{. Write} $\theta_J(p)=\int f_J\,dp$. {The query error satisfies}
\begin{equation}\label{eq:chebyshev-smooth-certificate}
 d_k(p,q)\le\frac{C_{d,k}}{J^k}
 +C_{d,k}\|\theta_J(p)-\theta_J(q)\|_2.
\end{equation}
 {This gives} $\etail=C_{d,k}/J^k$ and $C=C_{d,k}$. With {a suitable} $J$, {a fine enough} public grid, and {enough} output {records} $m$, {the} error {is}
 $\tilde O\bigl(\Neff^{-\min\{1,k/d\}}\bigr)$.

\section{Proofs for generalized moment matching}
\label{app:projection}

This appendix proves Theorem~\ref{thm:moment-synthesis} and
Corollary~\ref{thm:feature-oracle} and gives moment-fitting, sampling,
and public feature constructions.

\subsection{Kernel means and the Gaussian mechanism}

For a query $F:\Omega^n\to\R^D$, define its sensitivity by
\[
 \Delta_2(F)=\sup_{X\sim X'}\|F(X)-F(X')\|_2.
\]
Gaussian noise with standard deviation
\[
 \Delta_2(F)\sqrt{2\log(1.25/\delta)}/\eps
\]
gives $(\eps,\delta)$-DP for
$0<\eps<1$, $0<\delta<1$ \cite[Theorem~1]{balle2018}
\cite[Theorem~3.22]{dwork2014}.

\paragraph{Gaussian moment estimator.}
Let $f:\Omega\to\R^D$ be a public bounded moment map.
Take
\[
 \mathcal A_{\mathrm G}(X)=\theta(p_X)+\frac{\Sigma^{1/2}\xi}{\Neff},
 \qquad \xi\sim\mathcal N(0,I_D),
\]
where the public covariance satisfies
\begin{equation}\label{eq:moment-privacy}
 \sup_{x,y}\|\Sigma^{-1/2}(f(x)-f(y))\|_2\le1,
 \qquad \Sigma\succ0.
\end{equation}
 {Replacing one record changes $\Sigma^{-1/2}\theta(p_X)$ by at most $1/n$
in Euclidean norm. Adding $\xi/\Neff$ and then multiplying by $\Sigma^{1/2}$
gives the $(\eps,\delta)$-DP estimator $\mathcal A_{\mathrm G}$.} With the Euclidean norm, its error term is
\begin{equation}\label{eq:gaussian-moment-risk}
 \epriv=\frac{1}{\Neff}\E\|\Sigma^{1/2}\xi\|_2
 \le\frac{\sqrt{\tr(\Sigma)}}{\Neff}.
\end{equation}

The next lemma gives the expected MMD error of drawing $m$ records from the
private distribution $q$
\cite[equation~(7)]{muandet2016shrinkage}:
\begin{lemma}[Sampling a kernel mean]\label{lem:sampling-mean}
For independent $U_1,\ldots,U_m\sim q$,
\[
 \E\MMD_K\left(q,\frac1m\sum_{i=1}^m\delta_{U_i}\right)^2
 =\frac{\E K(U_1,U_1)-\|\mu_q\|_{\cH_K}^2}{m}
 \le\frac{\kappa^2}{m}.
\]
\end{lemma}
\begin{proof}
Put $q_m=m^{-1}\sum_{i=1}^m\delta_{U_i}$. Independence and centering
eliminate the cross terms, so \eqref{eq:mmd} and the variance identity give
\[
 \begin{aligned}
 \E\MMD_K(q,q_m)^2
 &=\E\left\|\frac1m\sum_{i=1}^m(\Phi(U_i)-\mu_q)\right\|_{\cH_K}^2\\
 &=\frac1m\E\|\Phi(U_1)-\mu_q\|_{\cH_K}^2\\
 &=\frac{\E K(U_1,U_1)-\|\mu_q\|_{\cH_K}^2}{m}
 \le\frac{\kappa^2}{m},
 \end{aligned}
\]
where $\E\Phi(U_1)=\mu_q$ and $\|\Phi(x)\|_{\cH_K}^2=K(x,x)\le\kappa^2$.
\end{proof}

\subsection{General error bound and kernel features}\label{app:moment-proof}\label{app:feature-proof}

\begin{proof}[Proof of Theorem~\ref{thm:moment-synthesis}]
The final output is $(\eps,\delta)$-DP by postprocessing of
$\mathcal A(X)$ \cite[Proposition~2.1]{dwork2014}.

Fix $X\in\Omega^n$ and write
\[
 u=\theta(p_X),\qquad \widehat\theta=\mathcal A(X),\qquad
 v_q=\theta(q),
\]
where $v_q$ is the moment vector fitted in Step~1 of
Algorithm~\ref{alg:features}.

The fitting condition in Step~1 of Algorithm~\ref{alg:features} and the
triangle inequality give
\begin{align*}
 \|v_q-u\|_*
 &\le\operatorname{dist}_*(\widehat\theta,\mathcal C_G)
      +\|\widehat\theta-u\|_*+\eopt\\
 &\le\operatorname{dist}_*(u,\mathcal C_G)
      +2\|\widehat\theta-u\|_*+\eopt\le2\|\widehat\theta-u\|_*+\egrid+\eopt.
\end{align*}
For the second inequality, take the infimum over $v\in\mathcal C_G$ in
$\|\widehat\theta-v\|_*\le\|\widehat\theta-u\|_*+\|u-v\|_*$.
The last inequality uses the grid bound \eqref{eq:moment-geometry}.

We now convert this moment error into query error using
\eqref{eq:moment-comparison}. Taking expectations and using the
definition of $\epriv$ in \eqref{eq:moment-risk} gives
\[
 \sup_X\E d_{\mathcal F}(p_X,q)
 \le\etail+C(2\epriv+\egrid+\eopt).
\]
This proves \eqref{eq:moment-oracle}, with the factor two included in $C$.

Next we account for sampling. Given $q$, the algorithm draws $m$
independent records from $q$. By the definition in \eqref{eq:ipm-sampling},
\[
 \E\bigl[d_{\mathcal F}(q,p_Y)\mid q\bigr]\le s_{\mathcal F}(m).
\]
The triangle inequality then gives
\[
 \begin{aligned}
 \E d_{\mathcal F}(p_X,p_Y)
 &\le\E d_{\mathcal F}(p_X,q)+s_{\mathcal F}(m).
 \end{aligned}
\]
Substituting the bound in \eqref{eq:moment-oracle} proves
\eqref{eq:moment-dataset}.
For MMD, Lemma~\ref{lem:sampling-mean} bounds the sampling term by
$\kappa/\sqrt m$.

With the Euclidean norm, we can reduce the factor on the private
estimation error from two to one. For this bound, we choose the weights
so that $v_q$ is within distance $\eopt$ of the point in $\mathcal C_G$
closest to $\widehat\theta$. That is,
\[
 \|v_q-\Pi_{\mathcal C_G}(\widehat\theta)\|_2\le\eopt.
\]
Projection onto this convex set does not increase distances
\cite{bauschke2017}.
Also, the distance from $u$ to the set is at most $\egrid$. Therefore,
\begin{equation}\label{eq:moment-projection-bound}
 \begin{aligned}
 \|v_q-u\|_2
 &\le\eopt+\|\Pi_{\mathcal C_G}(\widehat\theta)-\Pi_{\mathcal C_G}(u)\|_2
       +\operatorname{dist}(u,\mathcal C_G)\\
 &\le\eopt+\|\widehat\theta-u\|_2+\egrid.
 \end{aligned}
\end{equation}

Using this bound in the same expectation calculation replaces
$2\epriv$ by $\epriv$.
\end{proof}

\begin{proof}[Proof of Corollary~\ref{thm:feature-oracle}]

Here $f=\phi$, so $\theta(p)=\int\phi(x)\,dp(x)$.
The coordinates $\phi$ express $P_V\Psi(x)$ in an orthonormal basis
of $V$, after removing constant blocks. Constant coordinates contribute
nothing to the difference of two probability averages. The norm of the
difference in the remaining coordinates is therefore
$\|\theta(p)-\theta(q)\|_2$. For the part outside $V$,
\eqref{eq:feature-errors} gives
\[
 \begin{aligned}
 &\left\|\int(I-P_V)\Psi(x)\,d(p-q)(x)\right\|_H=\left\|\iint(I-P_V)(\Psi(x)-\Psi(y))\,dp(x)dq(y)\right\|_H
 \le\etail.
 \end{aligned}
\]
MMD is the norm of {$\int\Psi\,d(p-q)$, so} the {triangle inequality} gives
\[
 \MMD_K(p,q)\le\|\theta(p)-\theta(q)\|_2+\etail.
\]
This is \eqref{eq:moment-comparison} with $C=1$.

Distance to the nonempty closed convex set $\mathcal C_G$ is a convex
function. Jensen's inequality and \eqref{eq:feature-errors} therefore give
\[
 \begin{aligned}
 \operatorname{dist}(\theta(p),\mathcal C_G)
 &\le\int\operatorname{dist}(\phi(x),\mathcal C_G)\,dp(x)\le\int\min_{g\in G}\|\phi(x)-\phi(g)\|_2\,dp(x)
 \le\egrid,
 \end{aligned}
\]
which proves the grid bound \eqref{eq:moment-geometry}.

If no nonconstant block remains, no noise is needed: use a fixed
distribution on $G$. The comparison above and Lemma~\ref{lem:sampling-mean}
give the error bound with $\Gamma=0$, and the minimum trace is zero.
Otherwise, $\Gamma>0$ and every remaining $\Delta_j>0$.

 The added noise has covariance $\Sigma/\Neff^2$, where block $j$ of
$\Sigma$ is
\[
 (\Gamma\Delta_j/\sqrt{r_j})I_{r_j}
\]
by \eqref{eq:feature-allocation}.
Using the bound $\|\phi_j(x)-\phi_j(y)\|_2\le\Delta_j$ from
\eqref{eq:block-diameter}, we get
\[
 \begin{aligned}
 \|\Sigma^{-1/2}(\phi(x)-\phi(y))\|_2^2
 &=\sum_{j\in\mathcal I}
   \frac{\sqrt{r_j}}{\Gamma\Delta_j}\|\phi_j(x)-\phi_j(y)\|_2^2\le\frac1\Gamma\sum_{j\in\mathcal I}\sqrt{r_j}\Delta_j=1.
 \end{aligned}
\]
 {Thus} \eqref{eq:moment-privacy} {holds}, {so} the Gaussian estimator and Algorithm~\ref{alg:features} {satisfy $(\eps,\delta)$-DP}.
The trace is
\[
 \begin{aligned}
 \tr(\Sigma)
 &=\sum_{j\in\mathcal I}r_j\frac{\Gamma\Delta_j}{\sqrt{r_j}}=\Gamma\sum_{j\in\mathcal I}\sqrt{r_j}\Delta_j=\Gamma^2.
 \end{aligned}
\]
Thus \eqref{eq:gaussian-moment-risk} gives
$\epriv\le\Gamma/\Neff$.
We choose $\theta(q)$ within Euclidean distance $\eopt$ of
$\Pi_{\mathcal C_G}(\widehat\theta)$. Thus
\eqref{eq:moment-projection-bound} applies, with factor one on the
private estimation error. {Taking expectations and using} the comparison and grid bounds gives
\[
 \sup_X\E\MMD_K(p_X,q)
 \le\etail+\egrid+\eopt+\frac{\Gamma}{\Neff}.
\]
Sampling adds at most $\kappa/\sqrt m$ by
Lemma~\ref{lem:sampling-mean}, {which gives} \eqref{eq:feature-oracle}.

Finally, we show that the chosen covariance minimizes the mean squared
norm of the noise among the stated block covariances. Take
\[
 A=\bigoplus_{j\in\mathcal I}a_jI_{r_j},\qquad
 a_j>0,\qquad \sum_{j\in\mathcal I}\Delta_j^2/a_j\le1.
\]
Zero-mean noise with covariance $A/\Neff^2$ has mean squared error
$\tr(A)/\Neff^2$. Since $\tr A=\sum_{j\in\mathcal I}r_ja_j$,
Cauchy--Schwarz gives
\[
 \begin{aligned}
 \Gamma^2
 &=\left(\sum_{j\in\mathcal I}
       \sqrt{r_ja_j}\,\frac{\Delta_j}{\sqrt{a_j}}\right)^2\le\left(\sum_{j\in\mathcal I}r_ja_j\right)
       \left(\sum_{j\in\mathcal I}\frac{\Delta_j^2}{a_j}\right)
 \le\tr A.
 \end{aligned}
\]
 {The chosen} $\Sigma$ satisfies {the} constraint and has trace $\Gamma^2${, so it minimizes the trace among these block covariances}.
Equality requires $a_j=\Gamma\Delta_j/\sqrt{r_j}$, exactly the blocks
chosen in \eqref{eq:feature-allocation}.
\end{proof}

\subsection{Convex optimization for moment fitting}\label{app:moment-fitting}

For Euclidean fitting with positive tolerance, apply Frank--Wolfe
\cite[Algorithm~1 and Theorem~1]{jaggi2013} to
\[
 F(v)=\frac12\|v-z\|_2^2,\qquad v\in\mathcal C_G,
\]
where $z=\widehat\theta$ is the private moment estimate. Put
\[
 v_*=\Pi_{\mathcal C_G}(z),\qquad
 B_0=\max_{g\in G}\|f(g)\|_2.
\]
Start with unit weight on a fixed public candidate $g_{\mathrm{init}}\in G$,
so $v_0=f(g_{\mathrm{init}})$. For $t=0,1,2,\ldots$, use
\[
 \begin{aligned}
 g_t&\in\arg\min_{g\in G}\langle v_t-z,f(g)\rangle,\\
 v_{t+1}&=(1-\gamma_t)v_t+\gamma_t f(g_t),
 \qquad \gamma_t=\frac{2}{t+2}.
 \end{aligned}
\]
Break ties by a fixed public order on $G$ and update the probability
weights by the same convex combination. Each iterate therefore has the form
\[
 v_t=\sum_{g\in G}a_{t,g}f(g),\qquad
 a_{t,g}\ge0,\qquad \sum_{g\in G}a_{t,g}=1.
\]

Put $\delta_t=F(v_t)-F(v_*)$. The linear minimization rule and convexity give
\[
 \langle v_t-z,f(g_t)-v_t\rangle
 \le\langle v_t-z,v_*-v_t\rangle
 \le-\delta_t.
\]
Since $f(g_t),v_t\in\mathcal C_G$, their distance is at most $2B_0$.
Expanding the quadratic objective therefore gives
\[
 \begin{aligned}
 \delta_{t+1}
 &=\delta_t+\gamma_t\langle v_t-z,f(g_t)-v_t\rangle
   +\frac{\gamma_t^2}{2}\|f(g_t)-v_t\|_2^2\\
 &\le(1-\gamma_t)\delta_t+2B_0^2\gamma_t^2.
 \end{aligned}
\]
In particular, $\gamma_0=1$ gives $\delta_1\le2B_0^2$.
Induction yields $\delta_t\le8B_0^2/(t+2)$ for $t\ge1$: if the bound holds
at $t$, then
\[
 \delta_{t+1}
 \le\frac{8B_0^2(t+1)}{(t+2)^2}
 \le\frac{8B_0^2}{t+3}.
\]
The projection inequality
\[
 F(v)-F(v_*)\ge\frac12\|v-v_*\|_2^2
\]
thus gives
\[
 F(v_t)-F(v_*)\le\frac{8B_0^2}{t+2},\qquad
 \|v_t-v_*\|_2\le\frac{4B_0}{\sqrt{t+2}},\qquad t\ge1.
\]
For a prescribed $\eopt>0$, choose an integer $t\ge1$ with
$t+2\ge16B_0^2/\eopt^2$. Then $\|v_t-v_*\|_2\le\eopt$, which meets the
fitting condition in Algorithm~\ref{alg:features}, since
\[
 \|v_t-z\|_2\le\|v_*-z\|_2+\eopt.
\]

\subsection{General moment selections}
\label{app:nystrom}

In this section, we discuss choices of the finite-dimensional feature map
$\phi$ that is used as the moment map $f=\phi$ in
Algorithm~\ref{alg:features}.
For a finite public grid $G$, define its
feature-space covering radius by

\begin{equation}\label{eq:public-feature-cover}
 \rho\coloneqq\sup_x\min_{g\in G}\|\Phi(x)-\Phi(g)\|_{\cH_K}.
\end{equation}

Since $K$ is continuous, the canonical feature map $\Phi$ is continuous.
Hence $\Phi(\Omega)$ is compact and
$G$ can be chosen such that $\rho$ is arbitrarily small. The distance in this definition can be computed
from the kernel:
\[
 \|\Phi(x)-\Phi(g)\|_{\cH_K}^2=K(x,x)+K(g,g)-2K(x,g).
\]

\paragraph{Nystr\"om coordinates.}

Choose a data-independent subset $Z\subseteq G$ of Nystr\"om landmark
points \cite{chatalic2022}.
Let
\[
 V_Z=\operatorname{span}\{\Phi(z):z\in Z\}.
\]
We approximate $\Phi(x)$ by its closest vector in $V_Z$, namely
$P_{V_Z}\Phi(x)$. Using kernel values only, we can compute its coordinates and the approximation error as follows.

Form the Gram matrix and write its reduced eigendecomposition as
\[
 K_Z=(K(z,z'))_{z,z'\in Z}=U\Lambda U^\top,
\]
where $\Lambda$ contains the positive eigenvalues and the columns of $U$
are the corresponding orthonormal eigenvectors.
For a record $x$, use
\[
 k_Z(x)=(K(z,x))_{z\in Z},\qquad
 K_Z^\dagger=U\Lambda^{-1}U^\top,
\]
and take
\begin{equation}\label{eq:nystrom-coordinates}
 \begin{aligned}
 \phi(x)&=\Lambda^{-1/2}U^\top k_Z(x),\\
 \|(I-P_{V_Z})\Phi(x)\|_{\cH_K}^2
 &=K(x,x)-k_Z(x)^\top K_Z^\dagger k_Z(x).
 \end{aligned}
\end{equation}

The vector $\phi(x)$ gives the coordinates of $P_{V_Z}\Phi(x)$ in an
orthonormal basis, with one coordinate for each positive eigenvalue.
The second formula gives the squared norm of the
part left out, $\Phi(x)-P_{V_Z}\Phi(x)$.

\paragraph{Using all grid points.}
As a simple choice, take $Z=G$ and use one coordinate block. Writing
$r=\dim V_G$, we then have
\[
\Delta\le 2\kappa,\qquad
\etail\le 2\rho,\qquad
\egrid\le\rho.
\]
The first and third bounds follow from contraction of orthogonal projection
and $\|\Phi(x)\|_{\cH_K}\le\kappa$.
For the second, $(I-P_{V_G})\Phi(g)=0$ at every grid point, so
\eqref{eq:public-feature-cover} bounds each residual by $\rho$ and each
difference of residuals by $2\rho$.
Thus, with exact fitting on $G$, Corollary~\ref{thm:feature-oracle} gives
\[
\E\MMD_K(p_X,q)
\le 3\rho+\frac{2\kappa\sqrt r}{\Neff},
\]
with an additional $\kappa/\sqrt m$ when $m$ records are sampled from $q$.

\section{General kernel lower bounds}
\label{app:kernel-tools}\label{app:general-lower-proof}

We use an Assouad argument \cite{acharya2021}, encoding independent
binary choices by moving groups of records between selected pairs of
points. Group privacy limits how accurately these
choices can be recovered, while the separation parameter $\lambda$ below
converts errors in recovering the choices into MMD error.
We use the following consequence of differential privacy
\cite[Lemma~11]{ding2026}.
\begin{lemma}[Group privacy]\label{lem:group-privacy}
Let $P$ and $Q$ be the output distributions of an $(\eps,\delta)$-DP
mechanism on two databases differing in $h$ rows, where $0\le h\le n$.
For every measurable event $B$,
\[
 \begin{aligned}
 P(B)&\le e^{h\eps}Q(B)+\delta_h,\qquad
 Q(B)\le e^{h\eps}P(B)+\delta_h, \qquad
 \delta_h\coloneqq\delta\sum_{r=0}^{h-1}e^{r\eps}.
 \end{aligned}
\]
The empty sum is zero. If $\eps>0$, then
$\delta_h=\delta(e^{h\eps}-1)/(e^\eps-1)$.
\end{lemma}
\begin{proof}
Connect the two databases by a chain of $h$ single-row replacements.
Iterating the defining privacy inequality along this chain multiplies
the probability term by $e^{h\eps}$ and adds
$\delta\sum_{r=0}^{h-1}e^{r\eps}$. Reversing the chain gives the other
inequality; for $h=0$, the distributions coincide.
\end{proof}

\begin{theorem}[General kernel lower bound]\label{thm:general-lower}
Let $K$ be any continuous positive definite kernel on $\Omega$, and let
$0<\eps\le1$, $n\eps\ge2$, and
$0\le\delta\le c_0(e^\eps-1)$, where $c_0=1/[8(e-1)]$.
Choose $1\le k\le n\eps/2$ pairs of points $(a_i,b_i)\in\Omega^2$ and define
\[
 \lambda\coloneqq\inf_{\|u\|_2=1}
 \left\|\sum_{i=1}^k u_i\bigl(\Phi(b_i)-\Phi(a_i)\bigr)\right\|_{\cH_K}.
\]
With the universal constant $c_A=7/[32(e+1)]$, we have
\begin{equation}\label{eq:general-lower}
 \mathcal R_K(n,\eps,\delta)\ge c_A\frac{\lambda\sqrt k}{n\eps}.
\end{equation}
In particular,
\[
 \mathcal R_K(n,\eps,\delta)\ge \frac{c_A}{n\eps}\sup_{x,y}\|\Phi(x)-\Phi(y)\|_{\cH_K}.
\]
\end{theorem}

\begin{proof}
Put $h=\lfloor1/\eps\rfloor$.
Then $(2\eps)^{-1}\le h\le\eps^{-1}$ and $kh\le n/2$.

For each bit vector $\theta\in\{0,1\}^k$, construct a database
$X_\theta$ by partitioning $kh$ row positions into $k$ disjoint blocks of
size $h$. In block $i$, put $h$ copies of $a_i$ if $\theta_i=0$ and
$h$ copies of $b_i$ if $\theta_i=1$. Fill the remaining $n-kh$ rows with
$a_1$. Its kernel mean is
\[
 \mu_\theta=\mu_0+\frac hn\sum_{i=1}^k
 \theta_i\bigl(\Phi(b_i)-\Phi(a_i)\bigr),
\]
where
$\mu_0$ is the mean when all bits are zero.

Any private probability-measure output $\widehat q$ yields a private
kernel mean $\widehat\mu=\mu_{\widehat q}$ by postprocessing, and
\[
 \MMD_K(p_{X_\theta},\widehat q)
 =\|\widehat\mu-\mu_\theta\|_{\cH_K}.
\]
It therefore suffices to prove the bound for an arbitrary private
kernel-mean estimator.

Let $\widehat\mu$ be the output of any $(\eps,\delta)$-DP estimator of the kernel mean of $X_\theta$. Given the output $\widehat\mu$, choose the bit vector $\widehat\theta$ whose mean $\mu_{\widehat\theta}$ is closest to $\widehat\mu$ in $\cH_K$. Break ties by a fixed ordering of the $2^k$ bit vectors.
Let $B_0$ count the bits where $\widehat\theta$ differs from $\theta$. Then $\|\widehat\theta-\theta\|_2=\sqrt{B_0}$. The nearest-mean choice, the triangle inequality, and the definition of $\lambda$ give
\[
 2\|\widehat\mu-\mu_\theta\|_{\cH_K}
 \ge\|\mu_{\widehat\theta}-\mu_\theta\|_{\cH_K}
 \ge\frac{h\lambda}{n}\sqrt{B_0}.
\]

Choose the bits of $\theta$ independently from $\mathrm{Ber}(1/2)$.
Condition on all bits except $\theta_i$, and let $P$ and $Q$ be the
distributions of $\widehat\mu$ when $\theta_i=0$ and $\theta_i=1$,
respectively. Write
\[
 e_0=P(\widehat\theta_i=1),\qquad
 e_1=Q(\widehat\theta_i=0)
\]
for the two conditional error probabilities. The corresponding databases
differ in at most $h$ rows. Lemma~\ref{lem:group-privacy}, applied to the two possible
guesses, gives
\[
 1-e_0\le e^{h\eps}e_1+\delta_h,\qquad
 1-e_1\le e^{h\eps}e_0+\delta_h.
\]
Adding these inequalities bounds the conditional error under the uniform
prior on $\theta_i$:
\[
 \frac{e_0+e_1}{2}
 \ge\frac{1-\delta_h}{1+e^{h\eps}}
 \ge\frac{7}{8(e+1)}\eqqcolon\eta,
\]
where $h\eps\le1$ and $\delta_h\le c_0(e-1)=1/8$.
Average over the other
bits, then sum over $i$, to get $\E B_0\ge\eta k$. The expectation includes
both the random bits and the estimator's randomness. Since $B_0\le k$,
$\E\sqrt{B_0}\ge\E B_0/\sqrt k$.
Taking expectations in the displayed inequality gives a kernel-mean error
\[
 \begin{aligned}
 \E\|\widehat\mu-\mu_\theta\|_{\cH_K}
 &\ge \frac{h\lambda}{2n}\E\sqrt{B_0}\ge \frac{\eta h\lambda\sqrt k}{2n}
 \ge \frac{c_A\lambda\sqrt k}{n\eps},
 \end{aligned}
\]
where
$c_A=7/[32(e+1)]$. At least one database $X_\theta$ therefore has
expected error at least this large.
For $k=1$,
\[
 \lambda=\|\Phi(b_1)-\Phi(a_1)\|_{\cH_K}.
\]
Choosing a pair that maximizes this distance gives the lower bound for $\|\widehat\mu-\mu_\theta\|_{\cH_K}$ and hence $\mathcal R_K(n,\eps,\delta)$.
\end{proof}

\section{Gaussian and polynomial kernels}

\label{app:smooth-kernels}

\label{app:smooth}

For both kernels, choose $\Sigma$ as in \eqref{eq:feature-allocation}
and add Gaussian noise with covariance $\Sigma/\Neff^2$ to the feature
averages to obtain $\widehat\theta$. Use the Euclidean fitting rule
from Corollary~\ref{thm:feature-oracle}: choose a probability measure
$q$ on $G$ satisfying
\begin{equation}\label{eq:smooth-projection-fit}
 v_*\coloneqq\Pi_{\mathcal C_G}(\widehat\theta),\qquad
 \|\theta(q)-v_*\|_2\le\eopt.
\end{equation}
The constructions below specify the features, grid, and fitting tolerance.

\subsection{\texorpdfstring{{Gaussian features}}{Gaussian features}}

Use the Gaussian coordinates $\Psi_\alpha$ from
Section~\ref{sec:feature-specialization}. Expanding
$\exp(x^\top y/\sigma^2)$ verifies
$\langle\Psi(x),\Psi(y)\rangle=K_\sigma(x,y)$ and $\|\Psi(x)\|_2=1$
\cite{cotter2011}.
Group the coordinates with $|\alpha|=k$ in one block, including $k=0$.
Write $\Psi_{\ge J}(x)$ for the coordinates in blocks $k\ge J$.

\begin{lemma}[{Bounds for Gaussian features}]

\label{lem:gaussian-blocks-app}

Put $a=d/\sigma^2$.
The degree-$k$ block has dimension $r_k=\binom{d+k-1}{k}$.
For any two records, their degree-$k$ coordinate vectors differ by at most
\[
 \Delta_k=2\max_{0\le u\le a}(e^{-u}u^k/k!)^{1/2}.
\]
The following sum is finite, and the omitted coordinates satisfy

\begin{align}
 C_{d,\sigma}&:=\sum_{k\ge0}\sqrt{r_k}\Delta_k<\infty,
 \label{eq:gaussian-constant-app}\\
 \sup_{x\in\Omega}\|\Psi_{\ge J}(x)\|_2&\le(ea/J)^{J/2},\qquad J\ge a.
 \label{eq:gaussian-tail-app}
\end{align}

\end{lemma}

\begin{proof}

The multinomial theorem gives
\[
 \sum_{|\alpha|=k}\Psi_\alpha(x)^2=p_u(k),\qquad
 p_u(k)=e^{-u}u^k/k!,
\]
where $u=\|x\|_2^2/\sigma^2\le a$.
Thus the block vector has norm $\sqrt{p_u(k)}$.
The distance between two block vectors is at most the sum of their norms,
which gives the factor $2$ in $\Delta_k$.
Since $\Delta_k\le2(a^k/k!)^{1/2}$, Cauchy--Schwarz gives
\[
 \begin{aligned}
 C_{d,\sigma}
 &\le2\sum_{k\ge0}\sqrt{r_ka^k/k!}\le2\left(\sum_{k\ge0}r_k2^{-k}\right)^{1/2}
       \left(\sum_{k\ge0}\frac{(2a)^k}{k!}\right)^{1/2}
 =2^{1+d/2}e^a<\infty.
 \end{aligned}
\]
Here
\[
 \sum_{k\ge0}r_k2^{-k}
=\sum_{\alpha\in\mathbb N_0^d}2^{-|\alpha|}
=\left(\sum_{j\ge0}2^{-j}\right)^d=2^d.
\]

For every integer $J\ge1$, the inequality $(J+\ell)!\ge J!\,\ell!$ gives
\[
 \begin{aligned}
 \|\Psi_{\ge J}(x)\|_2^2
 &=e^{-u}\sum_{\ell\ge0}\frac{u^{J+\ell}}{(J+\ell)!}\le\frac{u^J}{J!}\,e^{-u}
       \sum_{\ell\ge0}\frac{u^\ell}{\ell!}
 =\frac{u^J}{J!}
 \le\left(\frac{ea}{J}\right)^J.
 \end{aligned}
\]
The last inequality uses $u\le a$ and $J!\ge(J/e)^J$.
Taking square roots proves \eqref{eq:gaussian-tail-app}.
\end{proof}

\begin{proof}[Proof of the Gaussian upper bound in Corollary~\ref{cor:smooth-optimal}]

For $\Neff\ge2$, choose an integer $J\ge a$, using only public
parameters, so that the omitted coordinates have norm at most
$\Neff^{-1}$. By \eqref{eq:gaussian-tail-app}, it is enough to require

\begin{equation}\label{eq:gaussian-cutoff}
 J\log(J/(ea))\ge2\log\Neff.
\end{equation}

For example,
\[
 J=\lceil\max\{e^2a,2\log\Neff\}\rceil
\]
satisfies this inequality.
Keep the coordinates with $|\alpha|<J$.
Each omitted part has norm at most $\Neff^{-1}$, so two omitted parts
differ by at most $2/\Neff$. Hence $\etail\le2/\Neff$.
The blocks we keep satisfy $\Gamma\le C_{d,\sigma}$.

To control the error from moving records to grid points, choose the output grid

\[
 L=\max\{1,\lceil\sqrt d\,\Neff/\sigma\rceil\},\qquad
 G=\{-1+2i/L:0\le i\le L\}^d,
\]
where rounding each coordinate to the nearest grid value moves a record by
at most $\sqrt d/L\le\sigma/\Neff$ in Euclidean distance.
Since
\[
 \begin{aligned}
 \|\Psi(x)-\Psi(y)\|_2^2
 &=2(1-e^{-\|x-y\|_2^2/(2\sigma^2)})\le\|x-y\|_2^2/\sigma^2,
 \end{aligned}
\]
the full feature vector changes by at most $\Neff^{-1}$.
Keeping only the selected coordinates cannot increase this distance,
so $\egrid\le\Neff^{-1}$.

Apply \eqref{eq:smooth-projection-fit} with $\eopt\le\Neff^{-1}$.
The omitted coordinates contribute at most $2/\Neff$.
Moving records to $G$ contributes at most $1/\Neff$.
The distance $\|\theta(q)-v_*\|_2$ contributes at most $1/\Neff$.
The expected norm of the Gaussian noise is at most $C_{d,\sigma}/\Neff$.
Since $\|\Psi(x)\|_2=1$, we have $\kappa=1$, and sampling adds at most
$m^{-1/2}$. Thus the expected MMD is at most
\[
 (C_{d,\sigma}+4)/\Neff+m^{-1/2}.
\]
For $m\ge\lceil\Neff^2\rceil$, Corollary~\ref{cor:smooth-optimal} holds
with $C_{d,K}=C_{d,\sigma}+5$.
\end{proof}

\subsection{\texorpdfstring{{Polynomial features}}{Polynomial features}}

\begin{proof}[Proof of the polynomial upper bound in Corollary~\ref{cor:smooth-optimal}]

For $K_{p,c}(x,y)=(c+x^\top y)^p$, with $c\ge0$ and integer $p\ge1$,
binomial and multinomial expansion give the feature coordinates
\[
 \Psi_\alpha(x)=
 \left[\binom p{|\alpha|}c^{p-|\alpha|}
       \frac{|\alpha|!}{\alpha!}\right]^{1/2}x^\alpha,
 \qquad |\alpha|\le p.
\]
Let $\phi(x)$ contain these coordinates except those that are identically zero or constant.
The moment vector is $\theta(p_X)=n^{-1}\sum_{i=1}^n\phi(x_i)$.
Constant coordinates have the same average under every probability measure,
so omitting them leaves all differences of feature means unchanged.
Thus $\etail=0$.

Group the coordinates by their total degree $k=|\alpha|$.
Each nonzero degree-$k$ group has dimension and squared norm
\[
 r_k=\binom{d+k-1}{k},\qquad
 \sum_{|\alpha|=k}\Psi_\alpha(x)^2
 =\binom pk c^{p-k}\|x\|_2^{2k}.
\]
The distance between two group vectors is at most the sum of their norms.
Since $\|x\|_2^2\le d$, the degree-$k$ coordinate vectors of any two
records differ by at most $2[\binom pk c^{p-k}d^k]^{1/2}$.
Hence $\Gamma\le C_{p,c,d}$, where

\begin{equation}\label{eq:polynomial-constant-app}
 C_{p,c,d}:=2\sum_{k=1}^p
 [\binom{d+k-1}{k}\binom pk c^{p-k}d^k]^{1/2}.
\end{equation}

Only degree $p$ remains when $c=0$.
To choose the grid spacing, put $u_x=(\sqrt c,x)$. The tensor features
$u_x^{\otimes p}$ have the same kernel as $\Psi$ because
\[
 \langle u_x^{\otimes p},u_y^{\otimes p}\rangle
 =(c+x^\top y)^p=K_{p,c}(x,y).
\]
Their pairwise distances therefore equal those of $\Psi$, and removing
constant or zero coordinates does not change these distances. The
telescoping identity
\[
 u_x^{\otimes p}-u_y^{\otimes p}
 =\sum_{a=0}^{p-1}
 u_y^{\otimes a}\otimes(u_x-u_y)\otimes u_x^{\otimes(p-1-a)}
\]
and $\|u_x\|_2,\|u_y\|_2\le\sqrt{c+d}$ give
\[
 \begin{aligned}
 \|\phi(x)-\phi(y)\|_2
 &=\|u_x^{\otimes p}-u_y^{\otimes p}\|_2\\
 &\le L_{p,c,d}\|x-y\|_2,
 \qquad L_{p,c,d}\coloneqq p(c+d)^{(p-1)/2}.
 \end{aligned}
\]

Choose a finite public grid $G$ so that every record $x$ is within
Euclidean distance $(L_{p,c,d}\Neff)^{-1}$ of some $g\in G$.
Then $\|\phi(x)-\phi(g)\|_2\le\Neff^{-1}$, so $\egrid\le\Neff^{-1}$.
Apply \eqref{eq:smooth-projection-fit} with $\eopt\le\Neff^{-1}$.
Corollary~\ref{thm:feature-oracle} gives
\[
 \E\MMD_K(p_X,p_Y)
 \le\frac{C_{p,c,d}+2}{\Neff}+\frac{(c+d)^{p/2}}{\sqrt m}.
\]
Here $\kappa=(c+d)^{p/2}$ gives the sampling term.
For $m\ge\lceil\Neff^2\rceil$, the bound for $K_{p,c}$ in
Corollary~\ref{cor:smooth-optimal} holds with
\[
 C_{d,K}=C_{p,c,d}+2+(c+d)^{p/2}.
\]
\end{proof}

\begin{proof}[Proof of Corollary~\ref{cor:smooth-optimal}]

The preceding constructions give the upper bounds. For the lower bound,
use $k=1$ and the records $0$ and
$e_1=(1,0,\ldots,0)$ in Theorem~\ref{thm:general-lower}.
Their squared feature distance is $2(1-e^{-1/(2\sigma^2)})>0$ for
the Gaussian kernel and $(c+1)^p-c^p>0$ for the polynomial kernel.
Both are positive constants for fixed kernel parameters, so the theorem
gives a lower bound of $c_K/(n\eps)$ with $c_K>0$.
\end{proof}

\begin{remark}[{Using bounds for smooth functions}]
\label{app:smooth-query-reduction}

By the derivative reproducing property \cite{zhou2008derivative},
functions in the Gaussian and polynomial RKHS unit balls and their
derivatives through any fixed integer order $k>d$ are uniformly bounded
by a constant depending only on $d$, $k$, and the kernel parameters.
After rescaling, \cite[Theorem~2]{ding2026} gives an
$O(\Neff^{-1})$ bound on the expected MMD.

For Mat\'ern kernels, derivatives through order $k$ are uniformly bounded
on the RKHS unit ball for integers $1\le k<\nu$, so the same theorem
applies, but does not recover the
transition at $\nu=d/2$.
When $\nu\ge d/2$, \cite[Theorem~4]{ding2026}
gives the error rate $O(\Neff^{-1})$, matching lower-bound.
When $0<\nu<d/2$, it gives the matching exponent
$1/2+\nu/d$, but it requires $s=\nu+d/2$ is an integer.
Our wavelet proof covers every $\nu>0$, including noninteger $s$.
\end{remark}

\section[Matern, Laplace, and neural tangent kernels]{Mat\'ern, Laplace, and neural tangent kernels}

\label{app:matern-complete}

We prove Theorem~\ref{thm:matern-rates} by comparing Mat\'ern MMD with
wavelet coefficient errors, denoising the private estimates, and choosing
the resolution and noise allocation at each scale.
An explicit construction of record pairs gives the minimax lower bound.
We then specialize to the Laplace kernel and transfer its bounds to the
neural tangent kernel through RKHS norm equivalence.

\subsection{\texorpdfstring{{Wavelet coefficients and MMD}}{Wavelet coefficients and MMD}}

\label{mat:app-wavelets}

\label{app:wavelet-proofs}

Fix $K=K_{\nu,\ell}$ on $\Omega=[-1,1]^d$ and put $s=\nu+d/2$.
For $\alpha=\sqrt{2\nu}/\ell$, use the Fourier convention
\[
 \widehat f(\omega)\coloneqq\int_{\R^d}e^{-i\omega^\top x}f(x)\,dx
\]
with inverse factor $(2\pi)^{-d}$.
With this convention, the Mat\'ern spectral density
\cite[equation~(4.15)]{rasmussen2006} is
\begin{equation}\label{eq:matern-spectrum}
 \widehat k_{\nu,\ell}(\omega)=
 \frac{(4\pi)^{d/2}\Gamma(\nu+d/2)\alpha^{2\nu}}{\Gamma(\nu)}
 \left(\alpha^2+\norm{\omega}_2^2\right)^{-(\nu+d/2)}.
\end{equation}

Recall the public wavelet basis $\mathcal W$, wavelet functions $\psi_{j,k}$,
and coefficient vectors $\beta_j$ defined in
Section~\ref{sec:wavelets}.
We use inhomogeneous Sobolev norms with the Fourier convention above.
Constants may depend on $d,\nu,\ell,\mathcal W$, but not on
$J,n,\eps,\delta$.

\begin{lemma}[{MMD from wavelet averages}]

\label{mat:wavelet-norm}

For every finite signed measure $\rho$ supported on $\Omega$,

\begin{equation}\label{mat:eq-wavelet-norm}
 \|\rho\|_K^2\coloneqq\iint K(x,y)\,d\rho(x)\,d\rho(y)
 \asymp_{d,\nu,\ell,\mathcal W}\sum_{j\ge0}2^{-2js}\|\beta_j(\rho)\|_2^2.
\end{equation}

For $\rho=p-q$, the left side is $\MMD_K(p,q)^2$.

\end{lemma}

\begin{proof}

Denote the constant in \eqref{eq:matern-spectrum} by
\[
 A_{d,\nu,\ell}=(4\pi)^{d/2}\Gamma(s)\alpha^{2\nu}/\Gamma(\nu).
\]
Since $s>d/2$ and $|\widehat\rho|\le|\rho|(\Omega)$,
the integrals below are finite. Fourier inversion gives
\[
 \|\rho\|_K^2=(2\pi)^{-d}\int|\widehat\rho(\omega)|^2
 \widehat k_{\nu,\ell}(\omega)\,d\omega.
\]
Comparing $\alpha^2+\|\omega\|_2^2$ with $1+\|\omega\|_2^2$ gives

\begin{equation}\label{eq:matern-explicit-spectral-comparison}
 \frac{A_{d,\nu,\ell}}{\max\{1,\alpha^2\}^s}\|\rho\|_{H^{-s}}^2
 \le\|\rho\|_K^2
 \le\frac{A_{d,\nu,\ell}}{\min\{1,\alpha^2\}^s}\|\rho\|_{H^{-s}}^2.
\end{equation}

The wavelet formula for $H^{-s}$ norms
\cite[Theorem~1.20]{triebel2008} bounds the squared norm above and below
by constant multiples of the weighted sum in \eqref{mat:eq-wavelet-norm}.
Combining this comparison with the spectral bounds proves the claim.
The constants come from the wavelet norm comparison and the displayed
spectral factors. Wavelets supported away from $\Omega$ have zero
coefficients. Adding the scaling functions at level zero changes only
a fixed constant.
\end{proof}

\begin{lemma}[Wavelet coefficient bounds]

\label{mat:wavelet-geometry}

There are positive $C_{\mathrm{dim}},C_R,C_\Delta$, depending only on
$d$ and $\mathcal W$, such that, for every
probability measure $p$, records $x,y\in\Omega$, and $j,J\ge0$,

\begin{align}
 d_j&\le C_{\mathrm{dim}}2^{jd},
 &\|\beta_j(p)\|_1&\le R_j\coloneqq C_R2^{jd/2},\label{mat:eq-size-l1}\\
 \|\beta_j(\delta_x-\delta_y)\|_2
 &\le\Delta_j\coloneqq C_\Delta2^{jd/2},\label{mat:eq-sensitivity}\\
 \sum_{j>J}2^{-2js}\|\beta_j(p-q)\|_2^2
 &\le\frac{4C_R^2}{2^{2\nu}-1}2^{-2\nu J}\qquad(q\in\mathcal P(\Omega)).\label{mat:eq-tail}
\end{align}

\end{lemma}

\begin{proof}

At level $j$, only $O(2^{jd})$ shifted wavelets meet the cube, which
bounds $d_j$. At any point, only a bounded number of wavelets are nonzero,
and each has absolute value $O(2^{jd/2})$. Hence

\[
 \sup_x\sum_k|\psi_{j,k}(x)|\le C2^{jd/2},\qquad
 \sup_x\sum_k\psi_{j,k}(x)^2\le C2^{jd}.
\]

Integrate the first bound against $p$ to obtain $\|\beta_j(p)\|_1\le R_j$.
The second bounds the Euclidean norm of the wavelet values at one record.
The triangle inequality for the values at $x$ and $y$ then gives
\eqref{mat:eq-sensitivity}.
For two probability measures,
\[
 \|\beta_j(p-q)\|_2
 \le\|\beta_j(p)\|_1+\|\beta_j(q)\|_1\le2R_j.
\]
After squaring and multiplying by $2^{-2js}$, this becomes
\[
 2^{-2js}\|\beta_j(p-q)\|_2^2\le4C_R^2 2^{-2\nu j}.
\]
Sum the geometric series over $j>J$ to obtain \eqref{mat:eq-tail}.
The same estimates hold when a wavelet support crosses the boundary of $\Omega$.
\end{proof}

\subsection{\texorpdfstring{{Reducing Gaussian noise}}{Reducing Gaussian noise}}

\label{mat:app-denoising}

\label{app:denoise-proof}

The following elementary bound explains how the public $\ell_1$
constraint controls the error of soft thresholding \cite{donoho1995soft}.
It also implies the general
three-way denoising bound, although the Mat\'ern construction below uses
only the identity estimator at coarse levels and soft thresholding at
fine levels.

\begin{lemma}[Gaussian denoising bound]\label{lem:denoising}
Let $r\ge1$ be an integer, $R,\sigma>0$, and $\theta\in\R^r$ with
$\|\theta\|_1\le R$. Observe $Z=\theta+\sigma g$, where
$g\sim\mathcal N(0,I_r)$. For a public threshold $t\ge\sigma$, define
\[
 \eta_t(z)=\arg\min_v\left\{\tfrac12\|v-z\|_2^2+t\|v\|_1\right\}
 =\bigl(\sgn(z_k)(|z_k|-t)_+\bigr)_{k=1}^r.
\]
Then
\begin{equation}\label{eq:soft-threshold-risk}
 \E\|\eta_t(Z)-\theta\|_2^2
 \le6tR+2r\sigma^2e^{-t^2/(8\sigma^2)}.
\end{equation}
Consequently, a choice depending only on $r,R,\sigma$ among returning
zero, keeping $Z$, and applying soft thresholding gives an estimator
$D_{r,R,\sigma}(Z)$ satisfying
\begin{equation}\label{eq:denoising-risk}
 \sup_{\|\theta\|_1\le R}\E\|D_{r,R,\sigma}(Z)-\theta\|_2^2
 \le C\min\left\{r\sigma^2,
 R\sigma\sqrt{\log(1+r\sigma/R)},R^2\right\},
\end{equation}
where $C$ is universal.
\end{lemma}

\begin{proof}
First consider one coordinate, with true value $\theta$ and standard
Gaussian noise $g$. Soft thresholding is $1$-Lipschitz and has bias
$|\eta_t(\theta)-\theta|=\min\{|\theta|,t\}$. If
$|\theta|\ge t/2$, then $t\ge\sigma$ gives
\[
 \E|\eta_t(\theta+\sigma g)-\theta|^2
 \le2\sigma^2+2\min\{\theta^2,t^2\}
 \le6t|\theta|.
\]
If $|\theta|<t/2$, then
$|\eta_t(\theta+\sigma g)|\le(|\sigma g|-t/2)_+$ and
$2\theta^2\le t|\theta|$. For $a\ge0$, the Gaussian tail integral obeys
\[
 \begin{aligned}
 \E(|g|-a)_+^2
 &=\frac{2e^{-a^2/2}}{\sqrt{2\pi}}
   \int_0^\infty u^2e^{-u^2/2-au}\,du\le e^{-a^2/2}.
 \end{aligned}
\]
Thus, in this second case,
\[
 \E|\eta_t(\theta+\sigma g)-\theta|^2
 \le t|\theta|+2\sigma^2e^{-t^2/(8\sigma^2)}.
\]
Combining the two cases and summing over coordinates proves
\eqref{eq:soft-threshold-risk}.

For the three-way bound, returning zero has squared error at most $R^2$,
and keeping $Z$ has expected squared error $r\sigma^2$.
Put $x=r\sigma/R$. If $x\ge1$, choose
$t=\sigma\sqrt{8\log(1+x)}\ge\sigma$. The thresholding bound gives
\[
 \E\|\eta_t(Z)-\theta\|_2^2
 \le6R\sigma\sqrt{8\log(1+x)}+
       \frac{2r\sigma^2}{1+x}
 \le CR\sigma\sqrt{\log(1+x)}.
\]
If $x<1$, then $\log(1+x)\ge x/2$ gives
\[
 r\sigma^2\le\sqrt2R\sigma\sqrt{\log(1+x)},
\]
so keeping $Z$ already controls the middle term in
\eqref{eq:denoising-risk} up to a universal factor. Choose the available
procedure with the smallest of these public risk bounds, breaking ties
in a fixed order. This choice uses only $r,R,\sigma$ and proves
\eqref{eq:denoising-risk}.
\end{proof}

\subsection{Proof of the Matérn upper bound}

\label{app:wavelet-risk}

Assume $0<\eps\le1$, $0<\delta<1$, and $\Neff\ge2$, and put
$s=\nu+d/2$. We will bound the expected MMD of the probability measure $q$ constructed below
by $Cb_{\Neff}$, where

\[
 b_{\Neff}=\begin{cases}
  \Neff^{-s/d},&s<d,\\
  \Neff^{-1}\log(e\Neff),&s=d,\\
  \Neff^{-1},&s>d.
 \end{cases}
\]

We apply Theorem~\ref{thm:moment-synthesis} to the internal measure $q$.
We first fix arbitrary public $J$ and weights, specify the estimator
$\mathcal A$, and bound $\etail$, $\epriv$, $\egrid$, and $\eopt$.
We then tune $J$ and the weights separately in the two smoothness regimes.
The moment vector is
\[
 f_J(x)=(2^{-js}\psi_{j,k}(x))_{j\le J,k\le d_j},
 \qquad \theta_J(p)=\int f_J\,dp.
\]
The wavelet moment certificate \eqref{eq:wavelet-moment-certificate} gives
the comparison condition with
\begin{equation}\label{mat:eq-tail-error}
 \etail=C2^{-\nu J}.
\end{equation}

Choose positive weights $w_0,\ldots,w_J$ (to be determined) publicly
with $\sum_{j=0}^Jw_j=1$. These weights determine the noise added to
each level. Set $\xi\sim\mathcal N(0,I_{D_J})$ and release
\begin{equation}\label{mat:eq-joint-release}
 Z\coloneqq\theta_J(p_X)+\frac{\Sigma_J^{1/2}\xi}{\Neff},
 \qquad \text{where }
 \Sigma_J\coloneqq\bigoplus_{j=0}^J2^{-2js}\sigma_j^2I_{d_j},
 \qquad \sigma_j\coloneqq\frac{\Delta_j}{\sqrt{w_j}},
\end{equation}
Equivalently, its level-$j$ block is
\[
 Z_j=2^{-js}\left(\beta_j(p_X)+\frac{\sigma_j}{\Neff}\xi_j\right).
\]
Replacing one record changes the level-$j$ block of the average by
$\bigl[\theta_J(\delta_x-\delta_y)\bigr]_j/n$.
Lemma~\ref{mat:wavelet-geometry} gives
\[
 \sum_{j\le J}\frac{w_j}{2^{-2js}\Delta_j^2}
 \|\bigl[\theta_J(\delta_x-\delta_y)\bigr]_j\|_2^2\le1.
\]
The Gaussian release $Z$ is therefore $(\eps,\delta)$-DP.
For the estimators specified in the two regimes below, write the private estimator as
\[
 \widehat\theta=\mathcal A(X)
 =\begin{cases}
   \mathcal D(Z),&s<d,\\
   Z,&s\ge d,
  \end{cases}
 \qquad \widehat\theta=(\widehat\theta_j)_{j\le J}.
\]
In either case, it is $(\eps,\delta)$-DP by postprocessing. Its error in the notation of
\eqref{eq:moment-risk} satisfies
\begin{equation}\label{mat:eq-private-error}
 \epriv\le\left(
 \sup_X\sum_{j\le J}
 \E\|\widehat\theta_j-\bigl[\theta_J(p_X)\bigr]_j\|_2^2
 \right)^{1/2}.
\end{equation}

At each level, only a fixed number of wavelets are nonzero at any point.
Their derivative bounds give
\[
 \sup_x\|Df_J(x)\|_{\mathrm{op}}^2
 \le C\sum_{j\le J}2^{2(1-\nu)j}=:L_J^2.
\]
Thus moving a record by distance $h$ changes $f_J(x)$ by at most $L_Jh$.
Choose a finite public grid $G$ so that every record is within distance
$b_{\Neff}/(2L_J)$ of a grid point. This gives
$\egrid\le b_{\Neff}/2$.
For the projection of $\widehat\theta$ onto $\mathcal C_G$, we can set $\eopt=b_{\Neff}/2$ and apply Section~\ref{app:moment-fitting}.
Then $\egrid+\eopt\le b_{\Neff}$.
Applying \eqref{eq:moment-oracle} and then \eqref{mat:eq-private-error}
gives

\begin{equation}\label{mat:eq-risk-transfer}
 \sup_X\E\MMD_{K_{\nu,\ell}}(p_X,q)\le C\left(
 2^{-\nu J}
 +\left[\sup_X\sum_{j\le J}
 \E\|\widehat\theta_j-\bigl[\theta_J(p_X)\bigr]_j\|_2^2\right]^{1/2}
 +b_{\Neff}\right).
\end{equation}
The three terms on the right are $\etail$, $\epriv$, and
$\egrid+\eopt$, respectively, up to constants. We will next choose $J$ and bound the $\epriv$ term.

\paragraph[The case 0<nu<d/2.]{\texorpdfstring{{The case $0<\nu<d/2$.}}{The case 0<nu<d/2.}}

When $s<d$, split the levels at $j_*\leq J$, where

\[
 j_*=\left\lfloor\frac{\log_2\Neff}{d}\right\rfloor,
 \qquad J=\left\lceil\frac{s}{d\nu}\log_2\Neff\right\rceil.
\]
Then $2^{dj_*}\le\Neff<2^{d(j_*+1)}$. Choose weights:
\[
 \overline w_j=\begin{cases}
  2^{-(d-s)(j_*-j)},&j\le j_*,\\
  2^{-2\nu(j-j_*)},&j>j_*,
 \end{cases}
 \qquad w_j=\frac{\overline w_j}{\sum_{a=0}^J\overline w_a}.
\]
Note that $\sum_{a=0}^J\overline w_a$ is bounded
above and below by positive constants depending only on $d$ and $\nu$.
Apply Lemma~\ref{lem:denoising} to obtain the estimator $\widehat\theta_j$ levelwise from $Z_j$, with radius $2^{-js}R_j$ and noise standard
deviation $2^{-js}\sigma_j/\Neff$, and set
$\mathcal A(X)=\mathcal D(Z)\coloneqq(\widehat\theta_j)_{j\le J}$.
For $j\le j_*$, the variance term in that lemma gives
\[
 \begin{aligned}
 \E\|\widehat\theta_j-\bigl[\theta_J(p_X)\bigr]_j\|_2^2
 &\le\frac{2^{-2js}d_j\sigma_j^2}{\Neff^2}
 \le\frac{C}{\Neff^2}\frac{2^{2(d-s)j}}{w_j}
 \le C\Neff^{-2s/d}2^{(d-s)(j-j_*)}.
 \end{aligned}
\]
Its sum over $j\le j_*$ is $O(\Neff^{-2s/d})$ because $d-s>0$.

For $j=j_*+r>j_*$, the middle term in
Lemma~\ref{lem:denoising} gives
\[
 \begin{aligned}
 \E\|\widehat\theta_j-\bigl[\theta_J(p_X)\bigr]_j\|_2^2
 &\le C\frac{2^{-2js}R_j\sigma_j}{\Neff}
 \sqrt{\log\left(1+\frac{d_j\sigma_j}{\Neff R_j}\right)}\le C\Neff^{-2s/d}2^{-\nu r}\sqrt{1+r},
 \end{aligned}
\]
where we used $w_j\asymp_{d,\nu}2^{-2\nu r}$ and
$2^{dj_*}\asymp_d\Neff$ in the last step.
This is summable over $r\ge1$ because $\nu>0$.
Together, the coarse and fine levels bound the block-error sum by
$C\Neff^{-2s/d}$, so \eqref{mat:eq-private-error} gives
$\epriv\le C\Neff^{-s/d}$.
Substituting this and the choice of $J$ into \eqref{mat:eq-risk-transfer} gives
\[
 \sup_X\E\MMD_{K_{\nu,\ell}}(p_X,q)\le C\Neff^{-s/d}.
\]

\paragraph[The case nu>= d/2.]{\texorpdfstring{{The case $\nu\ge d/2$.}}{The case nu>= d/2.}}

Set $\mathcal A(X)=Z$, so $\widehat\theta_j=Z_j$.
Choose the cutoff to
control the omitted levels and the weights to control the total noise error:

\[
 J=\left\lceil\frac{\log_2\Neff}{\nu}\right\rceil,
 \qquad \overline w_j=2^{(d-s)j},
 \qquad w_j=\frac{\overline w_j}{\sum_{a=0}^J\overline w_a}.
\]

By \eqref{mat:eq-private-error},
\[
 \epriv
 \le\frac C\Neff\left(\sum_{j=0}^J
 \frac{2^{2(d-s)j}}{w_j}\right)^{1/2}
 =\frac C\Neff\sum_{a=0}^J\overline w_a.
\]
If $\nu=d/2$, then $s=d$ and
$\sum_{a=0}^J\overline w_a=J+1=O(\log(e\Neff))$.
If $\nu>d/2$, then $s>d$ and
$\sum_{a=0}^J\overline w_a\le(1-2^{d-s})^{-1}=O_{d,\nu}(1)$.
Substituting these estimates for $\epriv$ and the choice of $J$ into
\eqref{mat:eq-risk-transfer} gives the other two bounds in
\eqref{eq:matern-rate}.

The remaining sampling error from drawing $m$ records from $q$ is at most
$m^{-1/2}$ by
Theorem~\ref{thm:moment-synthesis}.
With $m$ as in \eqref{eq:matern-output-size}, this term is no larger
than the MMD bound for $q$. This proves the upper bound in
Theorem~\ref{thm:matern-rates}.

\subsection[Proof of Matern lower bounds]{Proof of Mat\'ern lower bounds}

\label{app:lower-proofs}

Put $K=K_{\nu,\ell}$ for simplicity.
We aim to apply Theorem~\ref{thm:general-lower} with a suitable choice of $k$ and pairs of records $(a_i,b_i)$.

First suppose $0<\nu<d/2$, and put $s=\nu+d/2$ and
$k=\lfloor n\eps/2\rfloor$. Since $n\eps\ge2$, we have
$n\eps/4\le k\le n\eps/2$.
Fix, in terms only of $d,\nu$, a real compactly supported orthonormal
wavelet basis to which Lemma~\ref{mat:wavelet-norm} applies. Choose a nonzero
continuous mother wavelet $\psi$ and points $u,v$ in its support with
$\psi(v)-\psi(u)>0$.

Choose a scale $j$ with $2^{jd}\asymp_{d,\nu}k$. By compact support and
lattice packing, there are $k$ translates
$\psi_{j,z_1},\ldots,\psi_{j,z_k}$ with pairwise disjoint supports contained
in $\Omega$. Set
\[
 \psi_{j,z_i}(x)=2^{jd/2}\psi(2^jx-z_i),\qquad
 a_i=2^{-j}(z_i+u),\qquad b_i=2^{-j}(z_i+v).
\]

For any $c\in\R^k$, let
$\rho_c=\sum_{i=1}^k c_i(\delta_{b_i}-\delta_{a_i})$.
Disjoint supports ensure that only pair $i$ contributes to the selected
coefficient:
\[
 \int\psi_{j,z_i}\,d\rho_c
 =c_i2^{jd/2}(\psi(v)-\psi(u)).
\]
Keeping these coefficients in Lemma~\ref{mat:wavelet-norm} gives
\[
 \begin{aligned}
 \left\|\sum_{i=1}^k c_i\bigl(\Phi(b_i)-\Phi(a_i)\bigr)\right\|_{\cH_K}^2
 =\|\rho_c\|_K^2 \gtrsim_{d,\nu,\ell}\; 2^{-2js}2^{jd}\|c\|_2^2
 =2^{-2\nu j}\|c\|_2^2.
 \end{aligned}
\]
Thus the least singular value satisfies
$\lambda\gtrsim_{d,\nu,\ell} 2^{-\nu j}
\asymp_{d,\nu,\ell}k^{-\nu/d}$.
Theorem~\ref{thm:general-lower} now yields
\[
 \mathcal R_{K_{\nu,\ell}}(n,\eps,\delta)
 \gtrsim_{d,\nu,\ell}\frac{k^{1/2-\nu/d}}{n\eps}
 \asymp_{d,\nu,\ell}(n\eps)^{-1/2-\nu/d}.
\]

For $\nu\ge d/2$, apply Theorem~\ref{thm:general-lower} with $k=1$
and the single pair $(0,e_1)$.
Its squared feature distance is
$2(1-k_{\nu,\ell}(e_1))>0$, since the Mat\'ern spectral density is
strictly positive. Therefore
\[
 \mathcal R_{K_{\nu,\ell}}(n,\eps,\delta)
 \gtrsim_{d,\nu,\ell}(n\eps)^{-1}.
\]
Combining the two cases proves \eqref{eq:matern-lower}, with a constant
depending only on $d,\nu,\ell$ and independent of the algorithm's wavelet
basis.

\subsection{ReLU feature kernel  and neural tangent kernels}

\label{app:neural-comparison}

The following standard spherical-harmonic identities and ReLU multipliers
follow from \cite[Chapter~5]{axler2001harmonic} and the cosine-transform
multiplier formula \cite[Section~2]{olafsson2012cosine}, with the latter
specialized to exponent one and normalized for uniform probability measure.

\begin{lemma}[Spherical ReLU multipliers]
\label{lem:relu-harmonic-multipliers}
Let $d\ge1$, let $\tau$ be the uniform probability measure on $\mathbb S^d$,
and consider the ReLU activation function $\sigma(t)=t_+$.
Choose real spherical harmonics $Y_{ja}$ orthonormal in $L^2(\tau)$,
with $1\le a\le M_{j,d}$ at degree $j$. The dimension formula and
addition identity give
\[
 \begin{aligned}
 M_{0,d}&=1,\qquad M_{1,d}=d+1,\qquad
 M_{j,d}=\binom{d+j}{d}-\binom{d+j-2}{d}\quad(j\ge2);
 \end{aligned}
\]
\[\sum_{a=1}^{M_{j,d}}Y_{ja}(v)^2=M_{j,d}.\]
Moreover, $-\Delta_{\mathbb S^d}Y_{ja}=j(j+d-1)Y_{ja}$ and
\[
 \int_{\mathbb S^d}\sigma(v^\top u)Y_{ja}(u)\,d\tau(u)
 =\lambda_jY_{ja}(v),
\]
where
\begin{equation}\label{eq:neural-harmonic-eigenvalues}
 \begin{aligned}
 \lambda_1=\frac1{2(d+1)},\qquad
 \lambda_{2k+1}&=0\quad(k\ge1),\\
 \lambda_{2k}&=\frac{(-1)^{k+1}\Gamma((d+1)/2)}{4\pi}
 \frac{\Gamma(k-1/2)}{\Gamma(k+(d+2)/2)}\quad(k\ge0).
 \end{aligned}
\end{equation}
Hence there are estimates
\[
 |\lambda_{2k}|\asymp_d(1+k)^{-(d+3)/2},\qquad
 |\lambda_j|\le C_d(1+j)^{-(d+3)/2}\quad(j\ne1).
\]
\end{lemma}

The spherical kernels used below have nonzero eigenvalues on the even
harmonics. The following even Sobolev extension lets us use these harmonics
while preserving the required function values on the cube. We also use it
in the ReLU IPM lower bound.

\begin{lemma}[Even Sobolev extensions to the sphere]
\label{lem:even-sobolev-extension}
Let $\Omega=[-1,1]^d$, $s>d/2$, and put
$r(x)=\sqrt{1+\|x\|_2^2}$ and $z(x)=(1,x)/r(x)$.
For every $h\in {\mathcal H}^s(\Omega)$, there is an even function
$H\in {\mathcal H}^s(\mathbb S^d)$ satisfying
\[
 \forall x\in\Omega,\quad H(z(x))=\frac{h(x)}{r(x)};\qquad
 \|H\|_{{\mathcal H}^s(\mathbb S^d)}\le C_{d,s}\|h\|_{{\mathcal H}^s(\Omega)}.
\]
\end{lemma}

\begin{proof}
The Sobolev extension theorem for Lipschitz domains
\cite{rychkov1999} gives an extension $\widetilde h\in {\mathcal H}^s(\R^d)$
of $h$ with
\[
 \|\widetilde h\|_{{\mathcal H}^s(\R^d)}
 \le C_{d,s}\|h\|_{{\mathcal H}^s(\Omega)}.
\]
Choose a fixed $\chi\in C_c^\infty(\R^d)$ equal to one on a
neighborhood of $\Omega$, with support in a ball of radius depending
only on $d$. Define
\[
 H(\pm z(x))=\frac{\chi(x)\widetilde h(x)}{r(x)}
 \quad(x\in\R^d),\qquad H(u)=0\quad(u_0=0).
\]
The two hemisphere charts make $H$ well defined and even, and the compact
support of $\chi$ keeps its support a fixed positive distance from the
equator. On two fixed caps containing this support, the maps
$x\mapsto\pm z(x)$ and their inverse $x=u_{1:d}/u_0$ have bounded
derivatives of every required order. Changing coordinates bounds each
piece's Sobolev norm by $C_{d,s}\|(\chi/r)\widetilde h\|_{{\mathcal H}^s(\R^d)}$;
extending by zero adds no boundary term because $H$ vanishes near the
equator. The pieces are antipodal reflections, so the triangle inequality
costs at most a factor two. Since multiplication by $\chi/r$ is bounded
on ${\mathcal H}^s(\R^d)$, these estimates give
\[
 \begin{aligned}
 \|H\|_{{\mathcal H}^s(\mathbb S^d)}
 &\le C_{d,s}\|(\chi/r)\widetilde h\|_{{\mathcal H}^s(\R^d)}\le C_{d,s}\|\widetilde h\|_{{\mathcal H}^s(\R^d)}
 \le C_{d,s}\|h\|_{{\mathcal H}^s(\Omega)}.
 \end{aligned}
\]
For $x\in\Omega$, $\chi(x)=1$ and $\widetilde h(x)=h(x)$, so
$H(z(x))=h(x)/r(x)$. Since $s>d/2$, the Sobolev functions on the cube
and the sphere have continuous representatives, so this identity holds
at every $x\in\Omega$.
\end{proof}

The covariance kernel obtained by averaging products of ReLU outputs gives
a direct connection between the ReLU IPM and Mat\'ern MMD.

\begin{lemma}[ReLU feature kernel and Mat\'ern comparison]
\label{lem:relu-feature-comparison}
Use $\tau$, $r(x)$, and $z(x)$ as above, and define
\begin{equation}\label{eq:relu-feature-kernel}
 \begin{aligned}
 K_{\mathrm{RF}}(x,y)
 &\coloneqq\int_{\mathbb S^d}
 \sigma(v^\top(1,x))\sigma(v^\top(1,y))\,d\tau(v)=r(x)r(y)B(z(x),z(y)),\\
 \text{where }B(u,w)&\coloneqq\int_{\mathbb S^d}
       \sigma(v^\top u)\sigma(v^\top w)\,d\tau(v)
       =b(u^\top w),\\
 b(t)&\coloneqq
 \frac{\sqrt{1-t^2}+t(\pi-\arccos t)}{2\pi(d+1)}.
 \end{aligned}
\end{equation}
Then $H^{(d+3)/2}(\Omega)$ embeds continuously into
$\mathcal H_{K_{\mathrm{RF}}}(\Omega)$. Consequently, for all probability
measures $p,q$ on $\Omega$,
\begin{equation}\label{eq:relu-feature-matern-comparison}
 \MMD_{K_{3/2,1}}(p,q)\le C_d\MMD_{K_{\mathrm{RF}}}(p,q).
\end{equation}
\end{lemma}

\begin{proof}
To compute $B$, write a standard Gaussian vector in $\R^{d+1}$ as
$g=Rv$, where $v\sim\tau$ is independent of $R$ and
$\E R^2=d+1$. For unit vectors $u,w$ with angle $\vartheta$, polar
integration in their plane gives
\[
 \begin{aligned}
 (d+1)B(u,w)
 &=\E\sigma(g^\top u)\sigma(g^\top w)\\
 &=\frac1\pi\int_{\vartheta-\pi/2}^{\pi/2}
       \cos\varphi\cos(\varphi-\vartheta)\,d\varphi=\frac{\sin\vartheta+(\pi-\vartheta)\cos\vartheta}{2\pi}.
 \end{aligned}
\]
The factor $1/\pi$ uses the mean squared radius $2$ of a
two-dimensional standard Gaussian. Continuity covers the collinear
cases. This proves the formula for $b$; ReLU homogeneity gives the
identity for $K_{\mathrm{RF}}$.

By Fubini and Lemma~\ref{lem:relu-harmonic-multipliers}, the integral
operator with kernel $B$ has eigenvalue $\lambda_j^2$ on each degree-$j$
harmonic. Put $s=(d+3)/2$. For $h\in {\mathcal H}^s(\Omega)$, take its even
extension $H$ from Lemma~\ref{lem:even-sobolev-extension} and write
$H_{ja}=\langle H,Y_{ja}\rangle_{L^2(\tau)}$. All odd coefficients
vanish, and every even $\lambda_{2k}$ is nonzero. The spectral
characterization of the RKHS \cite[Section~6.1]{rasmussen2006} therefore gives
\[
 \begin{aligned}
 \|H\|_{\mathcal H_B}^2
 &=\sum_{k\ge0}\sum_{a=1}^{M_{2k,d}}
       \frac{|H_{2k,a}|^2}{\lambda_{2k}^2}\\
 &\le C_d\sum_{k\ge0}(1+2k)^{2s}
       \sum_{a=1}^{M_{2k,d}}|H_{2k,a}|^2\\
 &\le C_d\|H\|_{{\mathcal H}^s(\mathbb S^d)}^2
 \le C_d\|h\|_{{\mathcal H}^s(\Omega)}^2.
 \end{aligned}
\]
The weighted restriction $r(x)H(z(x))=h(x)$ belongs to
$\mathcal H_{K_{\mathrm{RF}}}$ with norm at most $\|H\|_{\mathcal H_B}$,
as follows from the kernel identity in \eqref{eq:relu-feature-kernel}.
This proves the embedding. In particular, the even extension avoids the
zero eigenvalues at odd degrees $j\ge3$; no inversion on those degrees
is needed.
Finally, the Mat\'ern RKHS at $\nu=3/2$ and length scale $1$ is
${\mathcal H}^s(\Omega)$ with equivalent norms by \eqref{eq:matern-spectrum}
and the native-space characterization \cite[Chapter~10]{wendland2005}.
Its unit ball thus lies in a constant multiple of the
$K_{\mathrm{RF}}$ RKHS unit ball. Taking the corresponding suprema of
$|\int h\,d(p-q)|$ proves \eqref{eq:relu-feature-matern-comparison}.
\end{proof}

\begin{lemma}[{Comparing the NTK and Laplace kernels}]

\label{lem:relu-ntk-comparison}

Let $\Omega=[-1,1]^d$, $d\ge1$.
The RKHS of $K_{\mathrm{NTK}}$ in \eqref{eq:relu-ntk} is $H^{(d+1)/2}(\Omega)$, with equivalent norms and equivalence constants depending only on $d$.
Consequently, there are $0<c_d\le C_d<\infty$ such that, for every
finite signed measure $\mu$ on $\Omega$,

\begin{equation}\label{eq:relu-ntk-dual-comparison}
 c_d\|\mu\|_{\mathcal H_{K_{1/2,1}}^*}
 \le \|\mu\|_{\mathcal H_{K_{\mathrm{NTK}}}^*}
 \le C_d\|\mu\|_{\mathcal H_{K_{1/2,1}}^*}.
\end{equation}

\end{lemma}

\begin{proof}

To use eigenvalue bounds on the sphere, divide $(1,x)$ by its length
to obtain a point on $S^d$.
Set

\[
 r(x)=\sqrt{1+\|x\|_2^2},\qquad
 z(x)=\frac{(1,x)}{r(x)},\qquad
 \theta(t)=\frac{\sqrt{1-t^2}+2t(\pi-\arccos t)}{\pi}.
\]

With $t=z(x)^\top z(y)$, the NTK becomes

\begin{equation}\label{eq:relu-ntk-normalized}
 K_{\mathrm{NTK}}(x,y)
 =\frac{r(x)r(y)}{2(d+1)}\theta(t).
\end{equation}

Let $s=(d+1)/2$. The covariance kernel $B(u,w)=b(u^\top w)$ in
\eqref{eq:relu-feature-kernel} has harmonic eigenvalues $\lambda_j^2$.
We obtain the NTK eigenvalues by differentiating $b$:
\[
 b'(t)=\frac{\pi-\arccos t}{2\pi(d+1)},\qquad
 b''(t)=\frac1{2\pi(d+1)\sqrt{1-t^2}}\quad(-1<t<1).
\]
For a function of $t=u^\top w$, the spherical Laplacian in $u$ is
$(1-t^2)b''(t)-dtb'(t)$. Substitution gives
\[
 \theta(u^\top w)
 =2\bigl[(d+2)B(u,w)-\Delta_{\mathbb S^d,u}B(u,w)\bigr].
\]
This identity also holds at $u=\pm w$ by continuity. Indeed, near either
pole, the angular formula for $b$ is a smooth function plus geodesic
distance cubed times a smooth function of distance squared. This is
$C^2$ in local coordinates, so $B(\cdot,w)$ is $C^2$ on the sphere and integration by
parts introduces no boundary terms. Applying the identity to each
harmonic and using $-\Delta_{\mathbb S^d}Y_{ja}=j(j+d-1)Y_{ja}$ gives
\[
 \begin{aligned}
 \Lambda_j&=2\bigl[d+2+j(j+d-1)\bigr]\lambda_j^2,\\
 \Lambda_{2k}&\asymp_d(1+k)^{-d-1},\qquad
 \Lambda_{2k+1}=0\quad(k\ge1).
 \end{aligned}
\]
Every even eigenvalue and the degree-one eigenvalue are positive.
The calculation applies also when $d=1$.

These eigenvalues bound the ${\mathcal H}^s(\mathbb S^d)$ norm by the spherical
RKHS norm; on even functions, the norms are equivalent. The even-extension
and weighted-restriction argument in the proof of
Lemma~\ref{lem:relu-feature-comparison}, now with $s=(d+1)/2$, therefore
gives ${\mathcal H}^s(\Omega)\subseteq\mathcal H_{K_{\mathrm{NTK}}}(\Omega)$ with
bounded norm, using \eqref{eq:relu-ntk-normalized}.
Conversely, restricting a spherical RKHS function to $z(\Omega)$,
changing variables back to $x$, and multiplying by the smooth positive
function $r$ bounds its ${\mathcal H}^s(\Omega)$ norm by its spherical RKHS norm.
Taking the infimum over spherical representatives of the same NTK
function gives the reverse norm inequality. Hence
\[
 \mathcal H_{K_{\mathrm{NTK}}}(\Omega)\simeq {\mathcal H}^s(\Omega).
\]

For the Laplace kernel with length scale $1$, the Mat\'ern spectral density
satisfies

\[
 \widehat k_{1/2,1}(\omega)\asymp_d
 (1+\|\omega\|_2^2)^{-s},
\]
so $\mathcal H_{K_{1/2,1}}(\Omega)\simeq {\mathcal H}^s(\Omega)$ as well.
The NTK and Laplace RKHS norms therefore bound each other by fixed
factors. By the definition

\[
 \|\mu\|_{\mathcal H_K^*}
 =\sup_{\|h\|_{\mathcal H_K}\le1}\left|\int h\,d\mu\right|,
\]
rescaling the functions in the supremum gives
\eqref{eq:relu-ntk-dual-comparison}.
\end{proof}

Since the Laplace kernel is the Mat\'ern kernel with $\nu=1/2$,
Theorem~\ref{thm:matern-rates} together with
Lemma~\ref{lem:relu-ntk-comparison} proves
Corollary~\ref{cor:laplace-neural}.

\section{Sliced Wasserstein distance}

\label{app:sliced-wasserstein}

Let $d\ge2$ and let $\tau$ be the uniform probability measure
on $\mathbb S^{d-1}$. Write $|\mathbb S^{d-1}|$ for the sphere's surface area.
For a direction $\theta$, project each record $x$ to the number $\theta^\top x$.
Write $\theta_\#p$ for the distribution of this number when $x\sim p$,
and use the same notation to project a signed measure.
For a compactly supported finite signed measure $\mu$ of total mass zero
on $\R^d$, define

\[
 H_\mu(\theta,t)=\mu\{x:\theta^\top x\le t\},\qquad
 \|\mu\|_{\mathrm{SW}}=\int |H_\mu(\theta,t)|\,dt\,d\tau(\theta).
\]

Here $dt$ is Lebesgue measure.
For $\mu=p-q$, $H_\mu(\theta,t)$ is the mass that $p$ puts below the
projected threshold $t$, minus the mass that $q$ puts there.
If $\mu$ is supported on $\Omega$, then
$H_\mu(\theta,t)=0$ for $|t|>\sqrt d$.
The one-dimensional cumulative-distribution formula for $W_1$
\cite{villani2009} gives $\SW_1(p,q)=\|p-q\|_{\mathrm{SW}}$.

For the upper bound in Theorem~\ref{thm:sliced-wasserstein}, we compare
$\SW_1$ with Laplace-kernel MMD and use the wavelet construction of
Algorithm~\ref{alg:features}.
For the lower bound, we move groups of records inside separate cells
and bound how accurately a private output can recover their positions.

\subsection{\texorpdfstring{{Query functions and upper bound}}{Query functions and upper bound}}

Let $\mathcal F_{\mathrm{SW}}$ consist of averages of one-dimensional
Lipschitz functions over projection directions:

\[
 h(x)=\int g_\theta(\theta^\top x)\,d\tau(\theta),
 \qquad \text{ where }g_\theta(0)=0,\quad \operatorname{Lip}(g_\theta)\le1,
\]
and assume that $g_\theta(t)$ is integrable over directions $\theta$
for each fixed $t$.
We will first show that the IPM over $\mathcal F_{\mathrm{SW}}$ is the sliced Wasserstein distance.

Since $|g_\theta(t)|\le|t|$, every $h\in\mathcal F_{\mathrm{SW}}$ is
bounded by $\sqrt d$ on $\Omega$. Its increments satisfy

\[
 |h(x)-h(y)|\le\int|\theta^\top(x-y)|\,d\tau(\theta)
 \le\|x-y\|_2,
\]
so every $h$ is $1$-Lipschitz and continuous.
Apply one-dimensional Kantorovich--Rubinstein duality to each
$g_\theta$ and integrate over directions to get

\[
 d_{\mathcal F_{\mathrm{SW}}}(p,q)\le\SW_1(p,q).
\]

For the reverse inequality, choose
\[
 g_\theta(t)=-\int_0^t\operatorname{sgn}H_{p-q}(\theta,u)\,du.
\]
These functions vanish at zero and are $1$-Lipschitz.
Their derivative is $-\operatorname{sgn}H_{p-q}(\theta,t)$ almost everywhere.
Integration by parts gives

\[
 \int g_\theta\,d(\theta_\#p-\theta_\#q)
 =-\int g_\theta'(t)H_{p-q}(\theta,t)\,dt
 =\int|H_{p-q}(\theta,t)|\,dt.
\]

Average this identity over $\theta$ to get the reverse inequality.
Hence $d_{\mathcal F_{\mathrm{SW}}}=\SW_1$.

\begin{lemma}[Sliced Wasserstein--Laplace comparison]
\label{lem:sw-laplace-comparison}
For every $p,q\in\cP(\Omega)$,
\[
 \SW_1(p,q)\le C_d\MMD_{K_{1/2,1}}(p,q).
\]
\end{lemma}

\begin{proof}

Let $\mu=p-q$ with $p,q\in\cP(\Omega)$ and put $s=(d+1)/2$.
For each $\theta$, the function $H_\mu(\theta,\cdot)$ is bounded and
vanishes outside $[-\sqrt d,\sqrt d]$, so it belongs to
$L^1(\R)\cap L^2(\R)$.
Differentiating the cumulative mass gives the projected measure:
$\partial_tH_\mu=\theta_\#\mu$ in the sense of distributions.
Using the Fourier convention of Appendix~\ref{app:wavelet-proofs},
we obtain, for $u\ne0$,

\[
 \widehat{H_\mu(\theta,\cdot)}(u)
 =\frac{\widehat\mu(u\theta)}{iu}.
\]

Apply one-dimensional Plancherel in $t$ and average over $\theta$:

\[
 \|H_\mu\|_2^2
 =\frac1{2\pi}\int_{\mathbb S^{d-1}}\int_{\R}
       \frac{|\widehat\mu(u\theta)|^2}{u^2}\,du\,d\tau(\theta).
\]

The integrals over $u<0$ and $u>0$ agree after replacing $\theta$ by
$-\theta$. Converting the positive part to polar coordinates gives

\begin{equation}\label{eq:sw-cdf-energy}
 \|H_\mu\|_{L^2(\tau\otimes dt)}^2
 =\frac{1}{\pi|\mathbb S^{d-1}|}
   \int_{\R^d}|\widehat\mu(\omega)|^2
          \|\omega\|_2^{-d-1}\,d\omega.
\end{equation}

The threshold interval $[-\sqrt d,\sqrt d]$ has length $2\sqrt d$.
Since $H_\mu$ vanishes outside it, Cauchy--Schwarz gives
$\|\mu\|_{\mathrm{SW}}\le\sqrt{2\sqrt d}\|H_\mu\|_2$.

Compact support also lets us avoid the singular Fourier weight near zero.
For each $\theta$, Cauchy--Schwarz gives
\[
 \left|\widehat{H_\mu(\theta,\cdot)}(u)\right|^2
 \le2\sqrt d\,\|H_\mu(\theta,\cdot)\|_2^2.
\]
Set $a=\pi/(4\sqrt d)$. Integrating this bound over $|u|<a$ gives at most
$\pi\|H_\mu(\theta,\cdot)\|_2^2$, whereas Plancherel gives
$2\pi\|H_\mu(\theta,\cdot)\|_2^2$ over the whole line.
Thus at least half the Fourier energy lies in $|u|\ge a$, and
\[
 \|H_\mu(\theta,\cdot)\|_2^2
 \le\frac1\pi\int_{|u|\ge a}
       \frac{|\widehat\mu(u\theta)|^2}{u^2}\,du.
\]
Average over $\theta$ and use polar coordinates as above. Since
$2s=d+1$ and $a$ depends only on $d$,
\[
 \begin{aligned}
 \|H_\mu\|_2^2
 &\le\frac{2}{\pi|\mathbb S^{d-1}|}
    \int_{\|\omega\|_2\ge a}|\widehat\mu(\omega)|^2
                  \|\omega\|_2^{-d-1}\,d\omega\\
 &\le C_d\int_{\R^d}|\widehat\mu(\omega)|^2
                    (1+\|\omega\|_2^2)^{-s}\,d\omega
 \le C_d\MMD_{K_{1/2,1}}(p,q)^2.
 \end{aligned}
\]
The last inequality is the Laplace spectral formula
\eqref{eq:matern-spectrum}. Combining this with
$\|\mu\|_{\mathrm{SW}}\le\sqrt{2\sqrt d}\|H_\mu\|_2$ proves the lemma.
\end{proof}

\begin{proof}[Proof of the upper bound in Theorem~\ref{thm:sliced-wasserstein}]

Combining Lemma~\ref{lem:sw-laplace-comparison} with the internal-measure
bound in the proof of Theorem~\ref{thm:matern-rates}, specialized to
$\nu=1/2$ and $\ell=1$, gives
\[
 \sup_X\E\SW_1(p_X,q)
 \le C_{d,\mathcal W}\Neff^{-\frac{d+1}{2d}}.
\]
Since every one-dimensional projection is supported on
$[-\sqrt d,\sqrt d]$, the standard empirical-CDF bound gives
$s_{\mathcal F_{\mathrm{SW}}}(m)\le\sqrt{d/m}$. Equation
\eqref{eq:moment-dataset} now proves \eqref{eq:sw-upper}, and privacy follows
from Theorem~\ref{thm:matern-rates}.
\end{proof}

\subsection{Proof of the lower bound in Theorem~\ref{thm:sliced-wasserstein}}

To prove the lower bound, we will construct a random product family of databases $X^U$, parametrized by $U$ uniformly distributed
over $A^k$ parameter vectors, and a radius $\rho$ such that
\[
 \sup_q\Prob\{\SW_1(p_{X^U},q)\le\rho\}\le A^{-k/2}.
\]
Thus no fixed output can approximate more than an exponentially small
fraction of the family. Steps 1--5 establish this property by moving groups
of records inside disjoint cells, smoothing and linearizing the movements,
and controlling the geometry of their projected CDFs. In Step 6, group
privacy bounds the information that a private output $Z$ carries about $U$.
Combining this information bound with the displayed small-ball bound shows that
$\SW_1(p_{X^U},Z)>\rho$ with constant probability. Finally, choosing the
number and width of the cells at the privacy scale gives
$\rho\asymp_d(n\eps)^{-(d+1)/(2d)}$.

\paragraph{Step 1: Moving records inside cells.}

All constants below depend only on $d$. Let
$M=\lceil(d+2)/2\rceil$, so that $2M>d+1$, and fix an even smooth
probability density $\kappa$ supported in the ball of radius
$\min\{1/64,1/(4\pi\sqrt{d+8})\}$. Since
$\|\omega\|_2\le\pi\sqrt{d+8}$ on $[-\pi,\pi]^d+2\pi e_1$, we have
$\widehat\kappa(\omega)\ge\cos(1/4)>1/2$ there. Set

\[
 \psi=\frac{\partial_1^{M-1}\kappa}
 {\|\partial_1^{M-1}\kappa\|_{L^1}},
 \qquad \eta=-\partial_1\psi,\qquad \zeta=\partial_1^2\psi.
\]

The denominator is nonzero, since otherwise $\kappa$ would be a compactly
supported polynomial in its first coordinate. Hence
$\|\psi\|_{\mathrm{TV}}=1$. Moreover, integration by parts shows that
$\eta$ and $\zeta$ annihilate every polynomial of total degree less than
$M$, and all three are supported in the ball of radius $1/64$.

For an integer $L\ge2$ to be chosen later, partition
$[-1/2,1/2]^d$ into $k=L^d$ equal cells of side length $h=1/L$, and let
$a_i$ be their centers. Write $\psi_h=(x\mapsto hx)_\#\psi$. For
$\nu\in\{\psi,\eta,\zeta\}$ and $|u|\le1/16$, define
\[
 \nu_{i,u}=(x\mapsto a_i+h(x+ue_1))_\#\nu.
\]
Thus $u$ moves a record from $a_i$ to $a_i+hue_1$, and
$\delta_{a_i+hue_1}*\psi_h=\psi_{i,u}$. These measures have total
variation bounded only in terms of $d$, and their supports lie within
distance $h/4$ of $a_i$, hence in $\Omega$.

\paragraph{Step 2: Convolution contraction.}
Convolution with $\psi_h$ cannot increase sliced distance.
For every zero-mass $\mu$,

\begin{equation}\label{eq:sw-convolution}
 \|\mu*\psi_h\|_{\mathrm{SW}}\le\|\mu\|_{\mathrm{SW}}.
\end{equation}

For each direction, the projected CDF of $\mu*\psi_h$ is the convolution
of $H_\mu(\theta,\cdot)$ with $\theta_\#\psi_h$.
The $L^1$ convolution inequality gives

\[
 \|H_{\mu*\psi_h}(\theta,\cdot)\|_{L^1(dt)}
 \le\|H_\mu(\theta,\cdot)\|_{L^1(dt)}
       |\theta_\#\psi_h|(\R)
 \le\|H_\mu(\theta,\cdot)\|_{L^1(dt)},
\]
since projection does not increase total variation and
$|\psi_h|(\R^d)\le1$.
Integrating over $\theta$ proves \eqref{eq:sw-convolution}.

\paragraph{Step 3: CDF geometry.}
Put
\[
 H_i=H_{\eta_{i,0}},\qquad H=(H_i)_{i=1}^k,\qquad
 F(c)=\left\|\sum_i c_iH_i\right\|_1,
\]
with all CDF norms taken under $\tau\otimes dt$.
Thus $F(c)$ measures the sliced-distance separation of the linearized
perturbations with coefficients $c$.
Let $G=\int HH^\top$ be the Gram matrix of these CDFs, so its quadratic
form is the squared $L^2$ norm of their linear combination. The following
matrix and pointwise bounds will give the volume estimate
\eqref{eq:sw-volume}:

\begin{equation}\label{eq:sw-stencil-estimates}
 c_dhI_k\preceq G\preceq C_dhI_k,
 \qquad \sup_{\theta,t}\|H(\theta,t)\|_2^2\le C_dkh.
\end{equation}

For the lower matrix bound, write $a_i=a_0+hz_i$ with distinct
$z_i\in\mathbb Z^d$, and put
\[
 T_c(\xi)=\sum_i c_i e^{-i\xi^\top z_i}.
\]
The Fourier transform of $\sum_i c_i\eta_{i,0}$ is
\[
 e^{-i\omega^\top a_0}\widehat\eta(h\omega)T_c(h\omega).
\]
Changing variables $\xi=h\omega$ in \eqref{eq:sw-cdf-energy} gives

\[
 c^\top Gc
 =\frac{h}{\pi|\mathbb S^{d-1}|}
       \int_{\R^d}w(\xi)|T_c(\xi)|^2\,d\xi,\qquad
 w(\xi)=|\widehat\eta(\xi)|^2\|\xi\|_2^{-d-1}.
\]

Restrict the nonnegative integral to $[-\pi,\pi]^d+2\pi e_1$.
On this shifted cube,

\[
 |\widehat\eta(\omega)|
 =\frac{|\omega_1|^M|\widehat\kappa(\omega)|}
 {\|\partial_1^{M-1}\kappa\|_{L^1}},
 \qquad \omega_1\in[\pi,3\pi].
\]

Since $|\omega_1|\ge\pi$, $\widehat\kappa\ge1/2$, and the
normalizing denominator depends only on $d$, we have
$|\widehat\eta(\omega)|\ge c_d>0$ on this cube.
The frequencies in the cube are bounded, so
$w(\omega)\ge c_d>0$ there. Since $T_c$ is
$2\pi\mathbb Z^d$-periodic, Parseval's identity gives

\[
 \begin{aligned}
 c^\top Gc
 &\ge c_dh\int_{[-\pi,\pi]^d+2\pi e_1}|T_c(\xi)|^2\,d\xi=c_dh(2\pi)^d\|c\|_2^2.
 \end{aligned}
\]

This proves the lower matrix bound in \eqref{eq:sw-stencil-estimates}.
The upper bound follows from the uniform estimate for shifted CDFs below.

For the pointwise estimate, $|H_i|\le\|\eta\|_{\mathrm{TV}}$.
Since $\eta$ has zero mass and the support of $\eta_{i,0}$ is within
distance $h/4$ of $a_i$, its CDF vanishes unless
$|t-\theta^\top a_i|\le h/4$.
For fixed $\theta,t$, these centers lie in a slab of width $h/2$.
Their lattice cubes have disjoint interiors and volume $h^d$.
Their union lies in the intersection of a fixed bounded cube and a slab
of width $C_dh$. This intersection has volume at most $C_dh$.
There are therefore at most $C_dh/h^d=C_dL^{d-1}=C_dkh$ such centers.
Summing $|H_i|^2$ proves the pointwise bound.

The Taylor remainder contains the second derivative $\zeta$ at shifted
positions. We therefore bound sums of the shifted measures uniformly:
for $\nu=\eta$ or $\zeta$ and $|u_i|\le1/16$,

\begin{equation}\label{eq:sw-shifted-synthesis}
 \left\|\sum_i c_i\nu_{i,u_i}\right\|_{\mathrm{SW}}
 \le C_d\sqrt h\|c\|_2.
\end{equation}

Let $\rho$ be the density of $\sum_i c_i\nu_{i,u_i}$, and put
\[
 f(x)=\sum_i c_i h^{-d}\kappa\bigl((x-a_i)/h-u_i e_1\bigr).
\]
For $m=M$ when $\nu=\eta$ and $m=M+1$ when $\nu=\zeta$,
\[
 \rho=\frac{\pm h^m\partial_1^m f}
 {\|\partial_1^{M-1}\kappa\|_{L^1}},
\]
with a negative sign for $\eta$ and a positive sign for $\zeta$.
The shifted supports lie within distance $h/4$ of their respective
centers and are disjoint. Scaling therefore gives
\[
 \|f\|_2^2=h^{-d}\|\kappa\|_2^2\|c\|_2^2,
 \qquad
 \|\rho\|_2^2=h^{-d}\|\nu\|_2^2\|c\|_2^2.
\]
Since $\rho$ has zero mass, apply \eqref{eq:sw-cdf-energy} and split
its Fourier integral at $\|\omega\|_2=h^{-1}$:
\[
 \begin{aligned}
 \|H_\rho\|_2^2
 &\le C_dh^{2m}\int_{\|\omega\|_2\le h^{-1}}
       \|\omega\|_2^{2m-d-1}|\widehat f(\omega)|^2\,d\omega+C_dh^{d+1}\int_{\|\omega\|_2>h^{-1}}
       |\widehat\rho(\omega)|^2\,d\omega\\
 &\le C_dh^{d+1}(\|f\|_2^2+\|\rho\|_2^2)
 \le C_dh\|c\|_2^2.
 \end{aligned}
\]
Here $2m\ge d+1$ gives
$h^{2m}\|\omega\|_2^{2m-d-1}\le h^{d+1}$ on the first region;
on the second, $\|\omega\|_2^{-d-1}\le h^{d+1}$.
Plancherel gives the second inequality. All constants depend only on
$d$, uniformly over the allowed shifts.
Taking $\nu=\eta$ and all $u_i=0$ proves $G\preceq C_dhI_k$ in
\eqref{eq:sw-stencil-estimates}.
Cauchy--Schwarz over
$\mathbb S^{d-1}\times[-\sqrt d,\sqrt d]$ gives
\eqref{eq:sw-shifted-synthesis}.

\paragraph{Step 4: Taylor linearization.}
The first-order Taylor formula for each translated density gives,
uniformly for $u\in[-a,a]^k$, $0<a\le1/16$,

\begin{equation}\label{eq:sw-linearization}
 \sum_i\psi_{i,u_i}=\sum_i\psi_{i,0}+\sum_i u_i\eta_{i,0}+\mathcal R(u),
 \qquad \|\mathcal R(u)\|_{\mathrm{SW}}\le C_da^2\sqrt{kh}.
\end{equation}

Indeed, $\partial_u\psi_{i,u}=\eta_{i,u}$ and
$\partial_u^2\psi_{i,u}=\zeta_{i,u}$.
Taylor's integral formula gives
\[
 \mathcal R(u)=\int_0^1(1-t)\sum_i u_i^2\zeta_{i,tu_i}\,dt.
\]
Apply \eqref{eq:sw-shifted-synthesis} with coefficients $u_i^2$ and use
$\|(u_i^2)_i\|_2\le a^2\sqrt k$ to obtain the remainder bound.

\paragraph{Step 5: Probability of a small error.}

To count record positions that give small error, we first bound the
volume of coefficient vectors with small
$F(c)=\|\sum_i c_i\eta_{i,0}\|_{\mathrm{SW}}$.
A standard volume inequality for subspaces of $L^1$, applied below,
and \eqref{eq:sw-stencil-estimates} give
\begin{equation}\label{eq:sw-volume}
 \operatorname{vol}_k\{c:F(c)\le t\}
 \le\left(\frac{C_dt}{\sqrt{kh}}\right)^k.
\end{equation}

Normalize the CDF vector by setting $g=G^{-1/2}H$.
The bounds in \eqref{eq:sw-stencil-estimates} give
\[
 \int gg^\top\,d\tau\,dt=I_k,
 \qquad \|g\|_2\le C_d\sqrt k.
\]
For each nonzero $g$, split the weight
$\|g\|_2^2\,d\tau\,dt$ equally between the antipodal directions
$g/\|g\|_2$ and $-g/\|g\|_2$.
Points with $g=0$ contribute no weight.
This defines a nonnegative finite even Borel measure $\gamma$ on
$\mathbb S^{k-1}$ satisfying
\[
 \int vv^\top\,d\gamma(v)=I_k,
 \qquad \gamma(\mathbb S^{k-1})
 =\int\|g\|_2^2\,d\tau\,dt=\tr(I_k)=k.
\]
Thus $\gamma$ is isotropic. Put
$N(y)=\int|v^\top y|\,d\gamma(v)$.
Ball's volume inequality, in the form for even isotropic Borel measures
in \cite[Theorem~1, $p=1$]{lutwak2004volume}, states that
\[
 \operatorname{vol}_k\{y:N(y)\le t\}
 \le\frac{(2t)^k}{k!},\qquad t\ge0.
\]
Indeed, the unit ball defined by $N$ is the body $Z_1^*(\gamma)$ in
that theorem, whose volume is at most $2^k/k!$; homogeneity gives the
bound for every $t$.

By the definition of $\gamma$ and $g=G^{-1/2}H$,
\[
 N(G^{1/2}c)=\int\|g\|_2|c^\top H|\,d\tau\,dt
 \le C_d\sqrt k F(c).
\]
Hence the map $c\mapsto G^{1/2}c$ sends $\{c:F(c)\le t\}$ into
$\{y:N(y)\le C_d\sqrt k\,t\}$.
Accounting for the change in volume gives
\[
 \operatorname{vol}_k\{F\le t\}
 \le(\det G)^{-1/2}\frac{(2C_d\sqrt k\,t)^k}{k!}
 \le\left(\frac{C_dt}{\sqrt{kh}}\right)^k,
\]
where the last step uses $\det G\ge(c_dh)^k$ and $k!\ge(k/e)^k$.
This proves \eqref{eq:sw-volume}.
The upper Gram bound also gives $F(c)\le C_d\sqrt h\|c\|_2$.

Let $A\ge e^{16}$ be an integer and then choose $a\in(0,1/16]$.
The probability bound below determines how large $A$ must be; the
Taylor remainder determines how small $a$ must be after $A$ is fixed.
Both choices use only $d$.
Let the $U_i$ be independent and uniform on the midpoints of the $A$
equal subintervals of $[-a,a]$.
To keep $r\eps\le1$ while placing $r$ records in each of $k$ cells, take

\[
 r=\lfloor1/\eps\rfloor,\qquad
 L=\left\lfloor(n/r)^{1/d}\right\rfloor,\qquad k=L^d,
 \qquad h=1/L.
\]

There are $A^k$ equally likely vectors $U=(U_1,\ldots,U_k)$.
In cell $i$, move the point $a_i$ by $hU_i$ along the first coordinate
and put $r$ copies of $a_i+hU_ie_1$ in the database $X^U$.
Since $k=L^d\le n/r$, fill the remaining $n-kr$ rows with zero.
The database has exactly $n$ records, all in $\Omega$, and

\[
 p_{X^U}=\frac rn\sum_i\delta_{a_i+hU_ie_1}
                   +\left(1-\frac{kr}{n}\right)\delta_0.
\]

Put $\rho=(r/n)(a/A)\sqrt{kh}$.
For any fixed probability measure $q$, few vectors $U$ put the empirical
measure $p_{X^U}$ within sliced Wasserstein distance $\rho$ of $q$:

\begin{equation}\label{eq:sw-prior-ball}
 \sup_q\Prob\{\SW_1(p_{X^U},q)\le\rho\}
 \le(C_d/A)^k\le A^{-k/2}.
\end{equation}

For fixed $A$, choose $a\le1/(C_dA)$ so that the remainder in
\eqref{eq:sw-linearization} is at most $(a/A)\sqrt{kh}$.
Fix $q$. If the event in \eqref{eq:sw-prior-ball} is nonempty, choose
one vector $u_0$ in it.
Every other vector $u$ in the event satisfies
$\SW_1(p_{X^u},p_{X^{u_0}})\le2\rho$ by the triangle inequality.
The records at zero cancel in the difference of the empirical measures.
Since $\delta_{a_i+hu_ie_1}*\psi_h=\psi_{i,u_i}$,
\eqref{eq:sw-convolution} gives

\[
 \left\|\sum_i(\psi_{i,u_i}-\psi_{i,(u_0)_i})\right\|_{\mathrm{SW}}
 \le\frac{2n\rho}{r}=2(a/A)\sqrt{kh}.
\]

Apply \eqref{eq:sw-linearization} to $u$ and $u_0$.
The two remainder terms are each at most $(a/A)\sqrt{kh}$, so
$F(u-u_0)\le4(a/A)\sqrt{kh}$.

Place a cube of side $2a/A$ around every vector $u$ in this event.
These cubes have disjoint interiors.
Any point $\widetilde u$ in such a cube is within Euclidean distance
$(a/A)\sqrt k$ of its center $u$.
The upper Gram bound gives
\[
 F(\widetilde u-u)\le C_d(a/A)\sqrt{kh},
\]
and hence
\[
 F(\widetilde u-u_0)\le(4+C_d)(a/A)\sqrt{kh}.
\]
The volume bound \eqref{eq:sw-volume} therefore bounds the total volume
of these cubes. Each cube has volume $(2a/A)^k$, and each of its centers
has probability $A^{-k}$. Thus the probability in
\eqref{eq:sw-prior-ball} is the total volume divided by
$A^k(2a/A)^k=(2a)^k$. This gives the first inequality in
\eqref{eq:sw-prior-ball}.
Choose $A$ large enough that $C_d/A\le A^{-1/2}$, then choose $a$ as above.
Both constants depend only on $d$.

\paragraph{Step 6: Contradiction with privacy.}

Assume $0<\eps\le1$ and $0\le\delta\le\eps/100$.
Let $Z=\mathcal M(X^U)$ be the probability measure returned by any
$(\eps,\delta)$-DP mechanism.
Changing $U_i$ moves the $r$ copies of the record in cell $i$ and leaves
all other rows unchanged. Hence, when $u,u'$ agree except in coordinate
$i$, Lemma~\ref{lem:group-privacy} gives the two inequalities there for
the corresponding output laws with
\[
 \alpha=r\eps\le1,\qquad
 \beta=\delta\sum_{j=0}^{r-1}e^{j\eps}
       \le e\delta/\eps.
\]

Fix $U_1,\ldots,U_{i-1}$ and average over the later coordinates.
These later coordinates are independent of $U_i$, so both sides of the
privacy inequality are averaged with the same weights. The resulting
conditional output distributions satisfy the same inequality.
Given the earlier coordinates, $U_i$ is still uniform on $A$ values.
Write $P_u$ for the conditional law of $Z$ when $U_i=u$, and let
$Q=A^{-1}\sum_uP_u$ be their average. Since $P_u\le A Q$, the density
$t_u=dP_u/dQ$ exists and satisfies $0\le t_u\le A$.
Averaging the privacy inequality over the other value of $U_i$ gives
$P_u(B)\le e^\alpha Q(B)+\beta$ for every measurable $B$.
Apply this to $B=\{t_u>e^\alpha\}$ to obtain
\[
 \int(t_u-e^\alpha)_+\,dQ\le\beta.
\]
For $0\le t\le A$, with $0\log0=0$,
\[
 t\log t\le\alpha t+(1+\log A)(t-e^\alpha)_+.
\]
Indeed, this is immediate when $t\le e^\alpha$; above $e^\alpha$,
the function $t\log t-\alpha t$ starts at zero and has derivative
$1+\log t-\alpha\le1+\log A$.
Integrating, averaging over $u$, and then averaging over the earlier
coordinates gives
\[
 \begin{aligned}
 I(U_i;Z\mid U_1,\ldots,U_{i-1})
 &=\E\left[\frac1A\sum_u\int t_u\log t_u\,dQ\right]\le\alpha+\beta(1+\log A).
 \end{aligned}
\]
The calculation holds for every fixed choice of the earlier coordinates,
and directly applies to any measurable output space.
Since $\alpha\le1$, $\beta\le e/100$, and $\log A\ge16$,
\[
 \alpha+\beta(1+\log A)
 \le\left(\frac1{16}+\frac{17e}{1600}\right)\log A
 <\tfrac18\log A.
\]
Adding over $i$ by the chain rule \cite[Section~2.5]{cover2006} gives

\begin{equation}\label{eq:sw-information}
 I(U;Z)\le k[\alpha+\beta(1+\log A)]
           \le\tfrac18k\log A
 \quad\text{if }\delta\le\eps/100.
\end{equation}

Call the output successful if $\SW_1(p_{X^U},Z)\le\rho$, and write
$B=A^{-k/2}$.
Equation~\eqref{eq:sw-prior-ball} says that any fixed output is successful
for at most a fraction $B$ of the possible databases.
If $U$ and $Z$ were drawn independently, the success probability would
therefore be $p_0\le B$. Let $p$ be the success probability when $Z$
is the actual output on $X^U$.
Apply relative-entropy data processing to the indicator of success
\cite[Section~2.7]{cover2006}, with $D(P\|Q)$ denoting relative entropy:

\[
 \begin{aligned}
 I(U;Z)&=D(P_{U,Z}\|P_U\otimes P_Z)\\
 &\ge p\log\frac p{p_0}+(1-p)\log\frac{1-p}{1-p_0}\ge p\log(1/B)-\log2.
 \end{aligned}
\]

The last inequality uses the bound $\log2$ on binary entropy.
Combining this with \eqref{eq:sw-prior-ball} and
\eqref{eq:sw-information} gives

\[
 \Prob\{\SW_1(p_{X^U},Z)\le\rho\}
 \le\frac{I(U;Z)+\log2}{\log(1/B)}
 \le\frac14+\frac{2\log2}{k\log A}<\frac12,
\]
so the error exceeds $\rho$ with probability greater than $1/2$.

Since $1/(2\eps)\le r\le1/\eps$, we have
$n\eps\le n/r\le2n\eps$.
For $n\eps\ge T_d$ with $T_d$ sufficiently large,
$L\ge2$ and $\tfrac12(n/r)^{1/d}\le L\le(n/r)^{1/d}$.
Thus $k\asymp_d n\eps$ and $2^{-d}\le kr/n\le1$.
Using $h=k^{-1/d}$ in the definition of $\rho$ gives

\[
 \rho=(a/A)(kr/n)\sqrt{h/k}
      =(a/A)(kr/n)k^{-\frac{d+1}{2d}}
      \asymp_d(n\eps)^{-\frac{d+1}{2d}}.
\]

The expected error is at least $\rho/2$ for a random database $X^U$.
At least one of these databases therefore has expected error at least
$\rho/2$. This proves the lower bound in
Theorem~\ref{thm:sliced-wasserstein}.

\section{ReLU IPMs}

\label{app:neural-ipm}

We prove the upper bound in Theorem~\ref{thm:neural-ipm} using four private
releases to estimate ReLU averages, and the lower bound by comparison with
Mat\'ern MMD. All constants in this appendix depend only on $d$.

\begin{proof}[Proof of Lemma~\ref{lem:neural-single-neuron}]
For $x\in\Omega$ and $\norm{w}_1+|b|\le1$,
\[
 0\le\sigma(w^\top x+b)\le\norm{w}_1+|b|\le1.
\]
Hence all allowed neurons and networks are bounded and integrable with
respect to $p,q$, since $|h(x)|\le\sum_j|a_j|\le1$.
For
\[
 h(x)=\sum_{j=1}^r a_j\sigma(w_j^\top x+b_j)\in\mathcal F_{\mathrm{NN}},
\]
linearity, the triangle inequality, and $\sum_j|a_j|\le1$ give
\[
 \left|\int h\,d(p-q)\right|
 \le\sum_{j=1}^r|a_j|
     \left|\int\sigma(w_j^\top x+b_j)\,d(p-q)(x)\right|
 \le\sup_{\norm{w}_1+|b|\le1}
     \left|\int\sigma(w^\top x+b)\,d(p-q)(x)\right|.
\]
Taking the supremum over $h$ proves one inequality in
\eqref{eq:neural-ipm}. The reverse follows because each allowed neuron
belongs to $\mathcal F_{\mathrm{NN}}$ with $r=1$ and $a_1=1$.
\end{proof}

Put
\[
 R_d\coloneqq\sqrt{d+1},\qquad u_x\coloneqq(1,x),\qquad
 r_x\coloneqq\norm{u_x}_2\le R_d,\qquad
 z_x\coloneqq u_x/r_x\in\mathbb S^d.
\]
Let $\tau$ be uniform probability measure on $\mathbb S^d$ and write
$a_v(x)=\sigma(v^\top u_x)$; for $v=(b,w)$ this includes the bias as
$v^\top u_x=b+w^\top x$. For a finite signed measure $\mu$, write
$\mu a_v=\int a_v\,d\mu$, the difference of neuron averages when
$\mu=p-q$. Every allowed parameter has Euclidean norm at most one,
while $\|v\|_1\le R_d\|v\|_2$ makes $v/R_d$ allowed for every
$v\in\mathbb S^d$. Positive homogeneity of ReLU therefore gives
\[
 R_d^{-1}\sup_{v\in\mathbb S^d}|\mu a_v|
 \le d_{\mathrm{NN}}(p,q)
 \le\sup_{v\in\mathbb S^d}|\mu a_v|,
 \qquad \mu=p-q.
\]

\subsection{\texorpdfstring{{Estimating neuron averages privately}}{Estimating neuron averages privately}}

\paragraph{\texorpdfstring{{Four private releases.}}{Four private releases.}}

Set

\[
 \Lambda\coloneqq\log(e\Neff)\log(e/\delta),\qquad
 \mathcal L\coloneqq\Lambda^5.
\]

The factor $\mathcal L$ bounds logarithmic terms in all noise and sampling
errors. If $\Neff\le C_d\mathcal L$, return a fixed public dataset, whose
IPM loss is at most one; otherwise choose

\[
 h=(C_d\mathcal L/\Neff)^{1/\max\{d,3\}}.
\]

Here $h$ bounds the distance to a selected direction. Return a fixed
public dataset also when $h>c_d$, which again gives the claimed bound;
it remains to consider $h\le c_d$.

Fix in advance a public grid $G$ that rounds each record within distance
$\eta=\Neff^{-2}$ and a public net $V_0\subset\mathbb S^d$ of mesh $\eta$.
Both sets have $O_d(\eta^{-d})$ points, hence polynomial size in $\Neff$.
Rounding records separately preserves replacement adjacency and changes
each mean of $a_v$ by at most $\eta$.
Write $H_v(x)=\mathbf1_{\{v^\top u_x>0\}}$ for the indicator that
the ReLU with direction $v$ is active at $x$.
We choose the zero subgradient of ReLU at zero, so $u_xH_v(x)$ is a
subgradient of $a_v(x)$ with respect to $v$ for every $v,x$.
This convention allows zero inner products; no perturbation of the public
net is needed.
In Algorithm~\ref{alg:features}, use $f=(a_v)_{v\in V_0}$,
$\|\cdot\|_*=\|\cdot\|_\infty$, and $\eopt=\eta$: fit $q$ on $G$ so that
the largest absolute difference between its neuron averages and their
private estimates is at most $\eta$ above the minimum.
These moment functions are fixed and public; only the queries inside
$\mathcal A$ may depend on earlier private outputs.
Each of the four releases receives budget $(\eps/4,\delta/4)$.

\paragraph{\texorpdfstring{{PREM estimates and sampling errors.}}{PREM estimates and sampling errors.}}

PREM (Private Relative Error Multiplicative Weight Update)
\cite{ghazi2025prem} privately answers a finite collection of bounded nonnegative
statistical queries by constructing a weighted histogram. Its error has
a relative part, proportional to the true query mean, and an additive
floor. This is useful here because activation disagreements and
nonnegative ReLU approximation remainders have small means.
We use PREM twice: first to select representative neuron directions,
then to estimate the approximation remainders. We state its accuracy
guarantee below before specifying the privacy and noise parameters.

\smallskip
\noindent\textit{Estimation guarantee.}
For a finite list of $[0,1]$-valued query functions, use the binary-query
version of PREM \cite[Theorem~3.1]{ghazi2025prem} together with the exact
threshold decomposition below. PREM returns a nonnegative count histogram
on $G$; dividing weighted query counts by the input size estimates query
means.
For each of the two calls, draw a fresh uniform sample of size
$s=\lfloor c n\eps\rfloor$ without replacement, keep its indices hidden,
and run PREM with privacy $(1,\delta_0)$, where
\[
 \delta_0=\min\{\delta/(4s/n),1/4\}
\]
and $c>0$ is a small absolute constant. Subsampling gives privacy
\[
 (\log(1+(s/n)(e-1)),(s/n)\delta_0)
 \le(\eps/4,\delta/4)
\]
\cite[Theorem~9]{balle2018subsampling}.
Use indexed records in the subsampling theorem to allow repeated values;
PREM discards the indices and uses only $G$. Choose $c\le1/(4(e-1))$, so
\[
 \log(1+(s/n)(e-1))\le(s/n)(e-1)\le\eps/4,
\]
and the choice of $\delta_0$ gives $(s/n)\delta_0\le\delta/4$.
In the range considered, $s\asymp n\eps$ and $s\gtrsim\Neff$.
For the sample empirical measure $p_S$ and count histogram $\widehat\nu$,
PREM gives the following bound simultaneously for all queries, except
with probability $\beta=\Neff^{-2}$:

\[
 \left|s^{-1}\int w\,d\widehat\nu-\int w\,dp_S\right|
 \le C_d\left\{\zeta\int w\,dp_S
       +\frac{\mathcal L}{\zeta\Neff}\right\}.
\]

On the finite set $G$, the binary guarantee extends exactly to
$[0,1]$-valued queries as follows.
For a nonzero $w:G\to[0,1]$, list its distinct positive values as
$0=t_0<t_1<\cdots<t_m\le1$, with $m\le|G|$. Then
\[
 w(x)=\sum_{i=1}^m(t_i-t_{i-1})
       \mathbf1_{\{w(x)\ge t_i\}}.
\]
The zero query needs no thresholds. The coefficients are nonnegative
and sum to $t_m\le1$. Since all estimates use the same histogram,
a simultaneous binary-query error bound with relative part $\zeta$
and additive part $a$ implies
\[
 \begin{aligned}
 \left|s^{-1}\int w\,d\widehat\nu-\int w\,dp_S\right|
 &\le\sum_{i=1}^m(t_i-t_{i-1})
 \left|s^{-1}\widehat\nu\{w\ge t_i\}-p_S\{w\ge t_i\}\right|\le\zeta\int w\,dp_S+a.
 \end{aligned}
\]
Each original query therefore requires at most $|G|$ binary queries,
with no approximation error or additional privacy cost.
In the analysis below, $Q$ counts binary queries after this expansion.
For an adaptive call, first fix the previous private outputs; these fix
the original query functions and all their thresholds.

The following parameter choices give the logarithmic factor.

\smallskip
\noindent\textit{Privacy and noise parameters.}
For $Q$ binary queries, use PREM's schedule of $I$ outer iterations
and at most $T$ RangeMonitor calls per iteration, with
\[
 I=O(\log(es)),\qquad T=O(\log(e|G|)/\zeta^2).
\]
Set
\[
 \beta'=\beta/(2IT),\qquad \delta'=\delta_0/(4IT).
\]
Set the RangeMonitor noise parameter to

\[
 a_{\mathrm{RM}}=\frac{c}{\sqrt{IT}[\log(eIT/\delta_0)]^{3/2}}.
\]

Use this noise parameter in Algorithm~2 of \cite{ghazi2025prem}, with margin
\[
 \alpha_0=C a_{\mathrm{RM}}^{-1}\log(eQ^2/\beta')
\]
for a sufficiently large absolute constant $C$.
The lower threshold subtracts $\alpha_0/2$ from the current estimated
query count and divides by $1+\zeta$; the upper adds $\alpha_0/2$ and
divides by $1-\zeta$.
This margin dominates the simultaneous accuracy radius over at most $Q^2$
RangeMonitor queries, with failure probability at most $\beta'$.

Proposition~2.2 of \cite{ghazi2025prem} gives
\[
 \eps_{\mathrm{RM}}=O(a_{\mathrm{RM}}\log(eIT/\delta_0))
\]
per call. With slack $\delta_0/4$, advanced composition
\cite[Theorem~3.20]{dwork2014} bounds all $IT$ calls by
$(1/2,\delta_0/2)$ for sufficiently small $c$: its leading term is
at most
\[
 C a_{\mathrm{RM}}\sqrt{IT}[\log(eIT/\delta_0)]^{3/2},
\]
and its quadratic term is $O(c^2)$.
Algorithm~3 of \cite{ghazi2025prem} also estimates the active sample count
at each outer iteration, adding independent Laplace noise of scale
$a_{\mathrm{mass}}^{-1}$, where

\[
 a_{\mathrm{mass}}=
 \frac{1}{4\sqrt{2I}\log(I/\delta_0)}.
\]

Conditional on previous outputs, this count has replacement sensitivity
one; advanced composition bounds the privacy cost of these $I$ releases
by the remaining half of $(1,\delta_0)$. Their absolute noise is simultaneously at most
$a_{\mathrm{mass}}^{-1}\log(I/\beta')$ with the allocated failure probability.
Choose the termination parameter
\[
 \alpha=C\left\{I\alpha_0+
 \frac{a_{\mathrm{mass}}^{-1}}{\zeta}\log(I/\beta')\right\},
\]
with a sufficiently large absolute constant, stopping when the noisy
active count is at most $\alpha/4$.
The first term covers accumulated margin errors; the second bounds the
mass-estimation error by a sufficiently small constant times $\zeta$
times the true active count whenever the loop continues, as required by
the multiplicative-weights potential argument.
All thresholds use only public parameters and previous private outputs.
With this schedule, the PREM analysis gives relative error $\zeta$,
additive count error $O(\alpha)$, and total failure probability at most $\beta$.
The scheduled cap $T\gtrsim\zeta^{-2}$ and the two noise scales give
\[
 \frac{a_{\mathrm{mass}}^{-1}}{\zeta}
 \le C I a_{\mathrm{RM}}^{-1}.
\]
Consequently,
\[
 \alpha\le CIa_{\mathrm{RM}}^{-1}\log(eQIT/\beta).
\]
Substituting the bounds on $I,T,a_{\mathrm{RM}}$ gives
\[
 \alpha\le\frac{C}{\zeta}
 [\log(es)]^{3/2}
 \sqrt{\log(e|G|)}[\log(eIT/\delta_0)]^{3/2}
 \log(eQIT/\beta).
\]

Polynomial grid and query sizes and the bounds
$\zeta\ge c_d\Neff^{-1/2}$ for both choices below in this range give
$\log(e|G|)$, $\log(eQ)$, and $\log(1/\zeta)$ of order
$O_d(\log(e\Neff))$.
Also,
\[
 \log(e/\delta_0)=O(\log(e/\delta)),\qquad
 \log(es)=O(\log(e\Neff)+\log\log(e/\delta)).
\]
In the range $\Neff>C_d\Lambda^5$,
$\log(e/\delta)\le C_d\Neff^{1/5}$ and hence
$\log\log(e/\delta)=O_d(\log(e\Neff))$.
Each logarithmic factor is thus $O_d(\Lambda)$, giving
\[
 \alpha\le C_d\Lambda^{9/2}/\zeta
 \le C_d\mathcal L/\zeta.
\]
The original query lists and $G$ have polynomial size in $\Neff$,
so their threshold expansion also has polynomial size and leaves
$\log(eQ)=O_d(\log(e\Neff))$ unchanged in order.
Dividing by $s\gtrsim\Neff$ and applying the exact threshold decomposition
gives the stated query-mean guarantee with the same $\zeta$ and additive
bound. The finite expansion can be constructed by sorting query values
on $G$, so it also preserves polynomial arithmetic runtime.
On PREM failure, return the zero histogram; privacy is preserved and this
failure is included in $\beta$.

\smallskip
\noindent\textit{Fresh-sample concentration.}
For each later call, condition on earlier private outputs to fix its
query functions, and draw fresh sample indices independently of those outputs.
The difference between a subsample average and its database mean is at most

\begin{equation}\label{eq:neural-subsample-error}
 C_d\left\{\sqrt{\frac{\mu\log(e\Neff)}s}+
                 \frac{\log(e\Neff)}s\right\}
\end{equation}
simultaneously for all queries whose database mean is at most $\mu$,
except with probability $\Neff^{-2}$.
Indeed, for one query, the Bernstein bound for sampling without
replacement \cite[Proposition~1.4]{bardenet2015sampling} and
$\operatorname{Var}(f)\le\E f$ for $0\le f\le1$ give error
\[
 C\{\sqrt{\mu t/s}+t/s\}
\]
with failure probability at most $2e^{-t}$.
For $Q$ queries, take
\[
 t=\log(2Q\Neff^2)=O_d(\log(e\Neff))
\]
and sum the failure probabilities.

\paragraph{\texorpdfstring{{First release: selecting directions from a private histogram.}}{First release: selecting directions from a private histogram.}}

For $v,z\in V_0$, let $D_{v,z}=|H_v-H_z|$ indicate different activation
signs. Apply PREM with $\zeta_1=c_dh$ to all these binary queries and the
constant-one query, then normalize its histogram to a probability $q_0$
on $G$, using a fixed public probability if its mass is zero.
Write $X_G$ for the rounded database, $b=\widehat\nu/s$, and $M=b(G)$.
On the PREM and sampling event, the constant-one query gives
\[
 |M-1|\le C_d\{h+\mathcal L/(h\Neff)\},
\]
and, for $D=D_{v,z}$, PREM and \eqref{eq:neural-subsample-error} give
\[
 |bD-p_{X_G}D|
 \le C_d\{h+\mathcal L/(h\Neff)+\sqrt{\mathcal L/\Neff}\}.
\]
Since $h^2\ge C_d\mathcal L/\Neff$, both
$\mathcal L/(h\Neff)$ and $\sqrt{\mathcal L/\Neff}$ are $O_d(h)$.
Both errors are therefore $O_d(h)$, with $1/2\le M\le3/2$ for $h\le c_d$.
On this event $q_0=b/M$ and $0\le p_{X_G}D\le1$, giving

\[
 |q_0D-p_{X_G}D|
 \le M^{-1}\bigl(|bD-p_{X_G}D|+|M-1|p_{X_G}D\bigr)
 \le C_dh.
\]

Select directions using the sum of parameter distance and activation
disagreement under $q_0$:
\[
 \rho(v,z)=\norm{v-z}_2+q_0D_{v,z}.
\]
For the size bound, lift $q_0$ to the sphere and mix it with uniform measure:
\[
 \overline q_0=(z_\cdot)_\#q_0,\qquad
 \omega=\tfrac12\overline q_0+\tfrac12\tau.
\]
For the hemisphere $C_v=\{u\in\mathbb S^d:v^\top u>0\}$,
rotational invariance gives
$\tau(C_v\triangle C_z)=\angle(v,z)/\pi$. Since chord length is at
most angle and $v^\top z_g$ has the same sign as $v^\top u_g$,
\[
 \omega(C_v\triangle C_z)
 =\tfrac12q_0D_{v,z}+\frac{\angle(v,z)}{2\pi}
 \ge\frac{\rho(v,z)}{2\pi}.
\]
Hemisphere traces on $m$ sphere points number at most $C_dm^d$, since
the corresponding parameter-space great spheres have $O_d(m^d)$ faces
of all dimensions, including lower-dimensional boundary faces. Each
relative interior has a constant trace.
The sampling form of Haussler's packing lemma
\cite[Section~1]{mustafa2016packing} bounds an $\epsilon$-separated
family in $L^1(\omega)$ by twice its expected number of traces on
$O_d(1/\epsilon)$ sample points. Thus its size is at most
$C_d\epsilon^{-d}$: VC dimension $d+1$ affects the sample-size constant,
whereas the trace exponent $d$ determines the power of $\epsilon$.
To extend the finite uniform version to $\omega$, partition the sphere
by the hemispheres in the finite family, approximate cell masses by rational
weights, and realize them by repeated atoms. This preserves separation
up to a factor two and gives the same bound for $\omega$.

Scan $V_0$ in a fixed order, selecting a direction whenever its
$\rho$-distance from all previously selected directions exceeds $h$.
The resulting maximal $h$-separated set $V\subset V_0$ gives an
$h/(2\pi)$-separated family of hemispheres in $L^1(\omega)$, so
$K\coloneqq|V|\le C_dh^{-d}$.
For each $v\in V_0$, choose a nearby selected direction $z(v)\in V$
with $\rho(v,z(v))\le h$, breaking ties in a fixed order.
The comparison between $q_0D$ and $p_{X_G}D$ then gives

\begin{equation}\label{eq:neural-private-net}
 \norm{v-z(v)}_2\le h,\qquad
 p_{X_G}D_{v,z(v)}\le C_dh.
\end{equation}

The choices of $V,z(v)$ postprocess the private histogram, costing no
further privacy. The size bound and first inequality in
\eqref{eq:neural-private-net} hold for every histogram, including the
fallback; the second requires the PREM and sampling bounds.

\paragraph{\texorpdfstring{{Second and third releases: estimating neuron values and subgradients.}}{Second and third releases: estimating neuron values and subgradients.}}

For the selected directions $z\in V$, estimate neuron values and the
chosen subgradients with respect to the augmented parameter $v=(b,w)$ using
\[
 A_{z,x}=a_z(x),\qquad
 B_{(z,j),x}=u_{x,j}H_z(x),\qquad z\in V,\quad 0\le j\le d.
\]
Columns index $x\in G$; their averages over the records $g_i$ give
the mean outputs for $A$ and averages of the chosen subgradients at
$v=z$ for $B$.
Small column norms in a factorization bound sensitivity, while small row
norms bound Gaussian noise. We use
\[
 \begin{aligned}
 \gamma_2(W)&=\inf_{W=TC}
   \max_i\norm{T_{i,\cdot}}_2\max_j\norm{C_{\cdot,j}}_2,\\
 \operatorname{herdisc}(W)&=\max_I\min_{\chi\in\{-1,1\}^{|I|}}
   \norm{W_I\chi}_\infty,
 \end{aligned}
\]
where $I$ ranges over column subsets and $W_I$ contains those columns.
The norm $\gamma_2$ satisfies the triangle inequality and is unchanged
by transposition. The determinant lower bound is at most a constant
times hereditary discrepancy, so \cite[Theorem~3.1]{matousek2020factorization}
gives, for nonzero $W$,
\[
 \gamma_2(W)\le
 C\log(e\operatorname{rank}W)\operatorname{herdisc}(W).
\]
We will prove
\begin{equation}\label{eq:neural-factorization}
 \gamma_2(A)\le C_d\log(eK)\max\{1,K^{(d-3)/(2d)}\},\qquad
 \gamma_2(B)\le C_d\log(eK)K^{(d-1)/(2d)}.
\end{equation}

For $d\le3$, use the spherical harmonics $Y_{ja}$, multiplicities
$M_{j,d}$, and ReLU multipliers $\lambda_j$ from
Lemma~\ref{lem:relu-harmonic-multipliers}. For $0<\varrho<1$ and
$u,v,z\in\mathbb S^d$, define
\[
 \mathcal P_{\varrho}(v,u)
 =\frac{1-\varrho^2}
 {(1-2\varrho v^\top u+\varrho^2)^{(d+1)/2}},\qquad
 F_{\varrho}(v,z)
 =\int\mathcal P_{\varrho}(v,u)\sigma(u^\top z)\,d\tau(u).
\]
The Poisson kernel is nonnegative, integrates to one against $\tau$,
and multiplies each degree-$j$ harmonic by $\varrho^j$
\cite[Theorem~5.33]{axler2001harmonic}.
For $U$ with density $\mathcal P_{\varrho}(v,\cdot)$, the degree-one
multipliers give $\E U=\varrho v$ and
$\E\|U-v\|_2^2=2(1-\varrho)$. ReLU is one-Lipschitz, so
\[
 |F_{\varrho}(v,z)-\sigma(v^\top z)|
 \le\E\|U-v\|_2\le\sqrt{2(1-\varrho)}.
\]
Set $A_{\varrho}=(r_xF_{\varrho}(z,z_x))_{z\in V,x\in G}$.
The harmonic multipliers give
\[
 A_{\varrho}
 =\sum_{j\ge0}\varrho^j\lambda_j
 \left(r_x\sum_{a=1}^{M_{j,d}}Y_{ja}(z)Y_{ja}(z_x)\right)_{z,x}.
\]
This series converges uniformly: the addition identity and
Cauchy--Schwarz give $\sum_a|Y_{ja}(z)Y_{ja}(z_x)|\le M_{j,d}$,
and $\sum_j\varrho^j|\lambda_j|M_{j,d}<\infty$.
For each degree, the displayed matrix factors with row vectors
$(Y_{ja}(z))_a$ and column vectors $r_x(Y_{ja}(z_x))_a$, whose norms
are $\sqrt{M_{j,d}}$ and at most $R_d\sqrt{M_{j,d}}$.
Thus the triangle inequality bounds the factorization norm of each
partial sum by $R_d\sum_j\varrho^j|\lambda_j|M_{j,d}$.
For any matrix $E$ with $K$ rows, the factorization $E=I_KE$ gives
$\gamma_2(E)\le\sqrt K\max_{z,x}|E_{z,x}|$; hence uniform convergence
passes this bound to $A_{\varrho}$:
\[
 \begin{aligned}
 \gamma_2(A_{\varrho})
 &\le R_d\sum_{j\ge0}\varrho^j|\lambda_j|M_{j,d}\\
 &\le C_d\left(1+\sum_{j\ge2}\varrho^j j^{(d-5)/2}\right)
 \le C_d\begin{cases}
 1,&d=1,2,\\
 \log(e/(1-\varrho)),&d=3.
 \end{cases}
 \end{aligned}
\]
For $d=1,2$, the series is summable even at $\varrho=1$; for $d=3$,
use $\sum_{j\ge1}\varrho^j/j=-\log(1-\varrho)$.
The approximation bound and the same factorization of $A-A_{\varrho}$
give
\[
 \gamma_2(A)\le\gamma_2(A_{\varrho})
       +R_d\sqrt{2K(1-\varrho)}.
\]
Choose $\varrho=1-1/\max\{2,K\}$, so $K(1-\varrho)\le1$.
It follows that $\gamma_2(A)\le C_d$ for $d=1,2$ and
$\gamma_2(A)\le C_d\log(eK)$ for $d=3$, as required.
Poisson smoothing is used only to prove this bound; the algorithm
still estimates the original neuron averages.

For $d\ge4$, set $W_{x,z}=|z_x^\top z|$, with columns indexed by $z\in V$.
For every even subset $P\subset V$,
\cite[pp.~67--68]{matousek1996zonoids} gives signs
$\chi_z\in\{-1,1\}$ such that

\[
 \sup_{u\in\mathbb S^d}
 \left|\sum_{z\in P}\chi_z|u^\top z|\right|
 \le C_d|P|^{1/2-3/(2d)}.
\]

For odd $P$, remove one point and give it either sign, adding at most
one to the sum. Absorbing small subsets into $C_d$ and using
$\operatorname{rank}W\le K$ gives
\[
 \operatorname{herdisc}(W)\le C_dK^{(d-3)/(2d)},\qquad
 \gamma_2(W)\le C_d\log(eK)K^{(d-3)/(2d)}.
\]
Multiplying row $x$ by $r_x\le R_d$ and transposing bounds
$\gamma_2((|z^\top u_x|)_{z,x})$ by $R_d\gamma_2(W)$.
The matrix $(z^\top u_x)_{z,x}$ factors through $\mathbb R^{d+1}$ with
row norms at most one and column norms at most $R_d$.
Thus $\sigma(t)=(|t|+t)/2$ and the triangle inequality prove the bound for $A$.

For $B$, group the directions by the index and sign of their largest
absolute coordinate, breaking ties in a fixed order. In each of these
$2(d+1)$ groups,
\[
 \begin{aligned}
 v&=\varsigma(e_j+y)/\sqrt{1+\|y\|_2^2},
 \qquad \varsigma\in\{-1,1\},\quad y_j=0,\quad |y_i|\le1,\\
 H_v(x)&=\mathbf1_{\{\varsigma(u_{x,j}+\sum_{i\ne j}y_i u_{x,i})>0\}},
 \end{aligned}
\]
so $H_v(x)$ is a halfspace indicator in the $d$ coordinates of $y$.
For any subset of $k$ points in one group,
\cite{matousek1995halfspaces} gives signs whose sum in every
halfspace is at most $C_dk^{(d-1)/(2d)}$ in absolute value.
This also applies to strict halfspaces: on a finite point set, shifting
the boundary slightly represents the same trace by a closed halfspace.
For $d=1$, the halfspaces are rays; alternating signs along the
ordered points gives discrepancy at most one.
Combining the signs chosen separately in the $2(d+1)$ groups changes
only $C_d$. Apply the $\gamma_2$ bound to $(H_z(x))_{x,z}$, whose columns
index the directions, and transpose back.
Multiplying column $x$ by $u_{x,j}\in[-1,1]$ cannot increase
$\gamma_2$. Stack the $d+1$ subgradient-coordinate factors in orthogonal
blocks: column norms increase by at most $\sqrt{d+1}$, while row bounds
stay unchanged.
This proves the second part of \eqref{eq:neural-factorization}.

For each possible $V$, fix finite factorizations $W=T_WC_W$, $W=A,B$,
within a factor two of $\gamma_2(W)$. Normalize their column norms to
at most one; their row norms are then at most twice the respective
bounds in \eqref{eq:neural-factorization}.
Writing $c_W(x)$ for column $x$ of $C_W$, use independent Gaussian
vectors for the two releases:
\[
 \begin{aligned}
 N_{\eps/4,\delta/4}&\coloneqq
 \frac{n(\eps/4)}{\sqrt{2\log(5/\delta)}}\ge c\Neff,\\
 \widehat b_W&=\frac1n\sum_{i=1}^n c_W(g_i)
       +\frac{2}{N_{\eps/4,\delta/4}}\xi_W,
 \qquad \xi_W\sim\mathcal N(0,I).
 \end{aligned}
\]
For every fixed $V$ and every database in $G^n$, replacement sensitivity
is at most $2/n$, so each release is $(\eps/4,\delta/4)$-DP.
Multiplying by $T_W$ estimates neuron averages for $A$ and averages of
the chosen subgradients for $B$, with noise standard deviation
$2\|t\|_2/N_{\eps/4,\delta/4}$ in row $t$.
Let $e_A,e_B$ be the largest errors in estimating
$n^{-1}\sum_i a_z(g_i)$ and $n^{-1}\sum_i u_{g_i,j}H_z(g_i)$,
respectively. A Gaussian union bound over $K$ values and $(d+1)K$
subgradient coordinates, with $\log(eK)=O_d(\log(e\Neff))$, gives
\[
 e_A\le C_d\mathcal L\max\{1,K^{(d-3)/(2d)}\}/\Neff,\qquad
 e_B\le C_d\mathcal L K^{(d-1)/(2d)}/\Neff
\]
with failure probability at most $2\Neff^{-2}$.

\paragraph{\texorpdfstring{{Fourth release and recovery.}}{Fourth release and recovery.}}

For each $v\in V_0$, let $z=z(v)$ be its nearby selected direction.
Define the linearization error at $z$ by

\[
 R_v(x)=a_v(x)-a_z(x)-(v-z)^\top u_xH_z(x).
\]

The subgradient inequalities at $z$ and $v$, respectively, give
\[
 0\le R_v(x)\le(v-z)^\top u_x\,[H_v(x)-H_z(x)]
 =|(v-z)^\top u_x|D_{v,z}(x)\le R_dhD_{v,z}(x).
\]
The equality uses the monotonicity of $t\mapsto\mathbf1_{\{t>0\}}$;
the last bound uses $\|v-z\|_2\le h$ and $\|u_x\|_2\le R_d$.
The chosen subgradients make both inequalities valid also at zero.
Thus $r_v=R_v/(R_dh)$ takes values in $[0,1]$ for every first-stage
output, because the distance bound in \eqref{eq:neural-private-net}
always holds. On the first-stage accuracy event, the sign-disagreement
bound also gives $\int r_v\,dp_{X_G}\le C_dh$.
Apply PREM to the binary threshold queries representing all $r_v$
through the exact decomposition above, using a fresh subsample and
\[
 \zeta_2=c_d\sqrt{\mathcal L/(h\Neff)}.
\]
By \eqref{eq:neural-subsample-error}, the subsample averages of the
original functions $r_v$ are also $O_d(h)$ on the sampling accuracy event.
Applying the exact threshold decomposition to the output histogram
estimates each $r_v$; multiplying by $R_dh$ then gives maximum remainder error

\[
 e_R\le C_dh\left\{\zeta_2h+
 \frac{\mathcal L}{\zeta_2\Neff}
       +\sqrt{\frac{h\mathcal L}{\Neff}}
       +\frac{\mathcal L}{\Neff}\right\}
 \le C_dh^{3/2}\sqrt{\mathcal L/\Neff}.
\]

Here $\zeta_2h$ and $\mathcal L/(\zeta_2\Neff)$ are both constant
multiples of $\sqrt{h\mathcal L/\Neff}$, and
$\mathcal L/\Neff\le C_dh$ bounds the last term by the same quantity.
The constants in the nontrivial range ensure $0<\zeta_2<1/2$.

For every first-stage output, the selected directions, factorizations,
and functions $r_v$ depend only on that output and public objects.
The two Gaussian releases and the fresh-sample PREM release are therefore
private for every history, whether or not its accuracy bounds hold.
By adaptive composition \cite[Theorem~B.1]{dwork2014}, their budgets,
together with the first release, give $(\eps,\delta)$ privacy.

For each $v\in V_0$, sum the estimated average of $a_z$, the inner
product of $v-z$ with the estimated average of the chosen subgradient,
and the estimated average of $R_v$. This estimates $\int a_v\,dp_{X_G}$ by definition of
$R_v$. Clip the estimates to $[0,R_d]$, which contains every true
neuron average, and call the resulting vector $\widehat\theta$.
Clipping cannot increase the error. With
\[
 \theta(p)=(\int a_v\,dp)_{v\in V_0},
\]
the PREM, subsampling, and Gaussian accuracy bounds give
\[
 \|\widehat\theta-\theta(p_{X_G})\|_\infty
 \le e_A+R_dh e_B+e_R,
\]
because $\|v-z(v)\|_1\le R_dh$. Substituting $K\le C_dh^{-d}$,
the subgradient contribution satisfies
\[
 h e_B\le C_d\frac{\mathcal L}{\Neff}h^{-(d-3)/2}
 \le C_d\frac{\mathcal L}{\Neff}\max\{1,h^{-(d-3)/2}\}.
\]
Rounding adds at most $\eta$ to the error relative to $\theta(p_X)$.
All accuracy bounds hold except with probability $C_d\Neff^{-2}$;
on failure, clipping bounds the error by $R_d$.
Thus $\mathcal A(X)=\widehat\theta$ satisfies
\[
 \epriv
 \le C_d\left\{
 \frac{\mathcal L}{\Neff}\max\{1,h^{-(d-3)/2}\}
 +h^{3/2}\sqrt{\mathcal L/\Neff}+\eta+\Neff^{-2}\right\}.
\]

To verify the comparison condition, choose $v_0\in V_0$ within $\eta$
of any $v\in\mathbb S^d$. ReLU is one-Lipschitz, so
\[
 |a_v(x)-a_{v_0}(x)|\le R_d\eta\qquad\text{for every }x.
\]
Hence, for any probabilities $p,q$,

\[
 d_{\mathrm{NN}}(p,q)
 \le\|\theta(p)-\theta(q)\|_\infty+2R_d\eta.
\]

This is \eqref{eq:moment-comparison} with comparison constant $1$
and $\etail=2R_d\eta$. Rounding any probability $p$ to $G$ gives
$\theta(p^G)\in\mathcal C_G$ and
$\|\theta(p)-\theta(p^G)\|_\infty\le\eta$, proving $\egrid\le\eta$.
Use Algorithm~\ref{alg:features} with $\eopt=\eta=\Neff^{-2}$.
Theorem~\ref{thm:moment-synthesis} yields
\[
 \sup_X\E d_{\mathrm{NN}}(p_X,q)\le C_d(\epriv+\eta),
\]
where $C_d$ absorbs the numerical factor from fitting; sampling adds
at most $s_{\mathcal F_{\mathrm{NN}}}(m)$.

Substitute $h=(C_d\mathcal L/\Neff)^{1/\max\{d,3\}}$.
For $d\le3$, both error terms are $O_d(\mathcal L/\Neff)$.
For $d\ge4$, both have order
$(\mathcal L/\Neff)^{(d+3)/(2d)}$.
The terms $\eta$ and $\Neff^{-2}$ are smaller in the range considered.
Consequently,
\[
 \sup_X\E d_{\mathrm{NN}}(p_X,q)
 \le C_d(\mathcal L/\Neff)^{\min\{1,(d+3)/(2d)\}}.
\]
Since $\mathcal L=\Lambda^5$, $\Lambda\ge1$, and the exponent is at most
one, the logarithmic factor is at most $\Lambda^5$ in every dimension.

Finally, draw $Y_1,\ldots,Y_m$ independently from $q$.
Symmetrization and ReLU contraction
\cite[proof of Theorem~8 and Theorem~12(4)]{bartlett2002}
bound the expected sampling error
in terms of the largest absolute value among $d+1$ coordinate sums:

\[
 \sup_q\E d_{\mathrm{NN}}(q,p_Y)
 \le4\E\left\|m^{-1}\sum_{i=1}^m\varepsilon_i(1,Y_i)\right\|_\infty
 \le4\sqrt{2\log(2(d+1))/m}.
\]

The independent Rademacher signs $\varepsilon_i$ are independent of $Y$.
Symmetrization and contraction with $\sigma(0)=0$ each contribute at most
a factor two. Maximizing the linear sum over $\|(b,w)\|_1\le1$ gives
its largest absolute coordinate. Each coordinate of $(1,Y_i)$ is bounded
by one, so its signed average is sub-Gaussian with variance proxy $1/m$;
the maximum over $d+1$ coordinates gives the last inequality uniformly in $q$.
Adding the sampling error proves \eqref{eq:neural-upper}.

\subsection{Minimax lower bound}

\paragraph[Comparison with a Matern kernel.]{\texorpdfstring{{Comparison with a Mat\'ern kernel.}}{Comparison with a Matern kernel.}}

For the ReLU feature kernel in \eqref{eq:relu-feature-kernel}, squared
MMD is the average squared neuron discrepancy. For $\mu=p-q$, Fubini gives
\[
 \begin{aligned}
 \MMD_{K_{\mathrm{RF}}}(p,q)^2
 &=\int_{\mathbb S^d}\left|\int a_v(x)\,d\mu(x)\right|^2\,d\tau(v)\le R_d^2 d_{\mathrm{NN}}(p,q)^2.
 \end{aligned}
\]
For the last inequality, $v/R_d$ satisfies the neuron constraint
$\|v/R_d\|_1\le1$, and ReLU homogeneity gives
$|\mu a_v|=R_d|\mu a_{v/R_d}|\le R_d d_{\mathrm{NN}}(p,q)$.
Lemma~\ref{lem:relu-feature-comparison} now yields
\[
 \MMD_{K_{3/2,1}}(p,q)
 \le C_d\MMD_{K_{\mathrm{RF}}}(p,q)
 \le C_d d_{\mathrm{NN}}(p,q).
\]
Apply this comparison to any private output, take expectations and the
worst case over databases, and use Theorem~\ref{thm:matern-rates}.
Its lower-bound exponent is
\[
 \min\{1,1/2+3/(2d)\}=\min\{1,(d+3)/(2d)\}
\]
under the same privacy conditions. This proves \eqref{eq:neural-lower}.

\end{document}